\documentclass[11pt,twoside]{article}

\usepackage{geometry}

\input{command-template}

\begin{document}
	
	\begin{center}
		
		{\bf{\LARGE{Accelerated and instance-optimal policy evaluation \\ with linear function approximation}}}
		
		\vspace*{.2in}
		
		{\large{
				\begin{tabular}{ccc}
					Tianjiao Li$^\star$ & Guanghui Lan$^\star$ & Ashwin Pananjady$^{\star, \dagger}$
				\end{tabular}
		}}
		\vspace*{.2in}
		
		\begin{tabular}{c}
			Schools of Industrial \& Systems Engineering$^\star$ and Electrical \& Computer Engineering$^\dagger$ \\
			Georgia Institute of Technology
		\end{tabular}
		
		\vspace*{.2in}

		\begin{tabular}{c}
			December 2021;\quad Revised: August 2022
		\end{tabular}
		
		\vspace*{.2in}
		
		\begin{abstract}
			We study the problem of policy evaluation with linear function approximation and present efficient and practical algorithms that come with strong optimality guarantees. We begin by proving lower bounds that establish baselines on both the deterministic error and stochastic error in this problem. In particular, we prove an oracle complexity lower bound on the deterministic error in an instance-dependent norm associated with the stationary distribution of the transition kernel, and use the local asymptotic minimax machinery to prove an instance-dependent lower bound on the stochastic error in the i.i.d. observation model. 
			Existing algorithms fail to match at least one of these lower bounds: To illustrate, we analyze a variance-reduced variant of temporal difference learning, showing in particular that it fails to achieve the oracle complexity lower bound. To remedy this issue, we develop an accelerated, variance-reduced fast temporal difference algorithm (VRFTD) that simultaneously matches both lower bounds and attains a strong notion of instance-optimality. Finally, we extend the VRFTD algorithm to the setting with Markovian observations, and provide instance-dependent convergence results.
			Our theoretical guarantees of optimality are corroborated by numerical experiments.
		\end{abstract}
	\end{center}
	
	
	
	
	
	
	
	\section{Introduction} \label{sec:intro}
	Reinforcement learning (RL) problems are generally formulated in terms of Markov decision processes (MDPs). At each time step, the agent observes the current state and subsequently takes an action, which leads to the realization of some reward as well as a transition to the next state according to the underlying, but unknown, stochastic transition function. The eventual goal of the agent is to learn a policy, i.e., a mapping from states to actions, to optimize the reward accrued over time. The setting is a very general one, with applications ranging from engineering to the natural and social sciences; see, e.g., \cite{kaelbling1996reinforcement, kober2013reinforcement} for surveys of RL applications. 
	
	A fundamental building block in RL is the problem of \emph{policy evaluation}, in which we are interested in estimating the long-term (discounted) \emph{value} of each state under a fixed policy with sample access to the transition and reward functions. 
	The literature considers three observation models for transition and reward samples, namely the generative model, the so-called ``i.i.d." model, and the Markovian noise model\footnote{Two of these observation models are formally discussed in Section~\ref{sec_setting}.}.
	Furthermore, in modern applications with large state spaces, it is common to seek an approximation to the true value function within the span of a small number of basis functions, a setting that is commonly known as \emph{linear function approximation}. In the canonical setting of the problem, one is interested in using random observations to compute an approximate value function within the subspace, with the distance between the true value function and its approximation being measured according to an instance-dependent ``weighted $\ell_2$-norm" that depends on the stationary distribution of the transition kernel. 
	%
	
	It is common to use stochastic approximation (SA) algorithms to solve the policy evaluation problem in the setting described above. Given the iterative nature of these algorithms, their convergence rates can be decomposed into two types of error: a \emph{deterministic error} that measures how fast the algorithm converges to its fixed point even in the absence of noise, and a \emph{stochastic error} that measures the contribution of the noise. \black{Either of these errors could dominate in practice. While the stochastic error is typically larger in noisy problems, the deterministic error of the algorithm can dominate, for example, in settings with multiple processors. In particular, the collection and use of multiple samples/trajectories in parallel can reduce the stochastic error considerably. }
	
	Loosely speaking, the deterministic error is measured in terms of the \emph{oracle complexity} of the algorithm and the stochastic error is measured in terms of the \emph{sample complexity}.
	The eventual goal of algorithm design is to develop practical algorithms that have optimal oracle and sample complexities. Ideally, these optimality guarantees should be \emph{instance-specific}, in that they depend explicitly on the problem at hand and 
	enable us to draw distinctions between the performance profiles of different algorithms. 

	With this context in hand, let us briefly discuss the state of the art relevant to characterizing the complexity of policy evaluation and related problems.
	Classical work by \citet{nemirovsky1991optimality,nemirovsky1992information} established oracle complexity lower bounds for solving linear operator equations in $\ell_2$-norm. However, these results do not extend to the more specific policy evaluation setting under the weighted $\ell_2$-norm. 
	On the other hand, 
	\citet{khamaru2020temporal} recently provided an instance-specific analysis of the sample complexity of policy evaluation under the $\ell_\infty$-norm, focusing on the generative observation model without function approximation. 
	\citet{mou2020optimal} studied the broader problem of solving projected fixed point equations, with a focus on characterizing the error incurred due to projection onto a subspace. 
	By virtue of studying the more general problem, the lower bounds on the statistical error proved in this paper are not specific enough to capture the policy evaluation setting with function approximation. \black{Concurrent work by \citet{mou2021optimal} studied SA methods for solving linear fixed point equations with Markovian samples and established a non-asymptotic, instance-dependent lower bound.} Given this state of affairs, the central question that motivates this paper is the following:
	\vspace{0.2cm}
	\begin{center}
		{\it What are the optimal oracle and sample complexities of policy evaluation in weighted $\ell_2$-norm with linear function approximation, and do existing algorithms achieve these bounds?}
	\end{center} 
	
	
	\subsection{Related work}
	There is a large literature on stochastic approximation for policy evaluation. The most popular stochastic iterative algorithm used for policy evaluation is temporal difference (TD) learning; see \citet{dann2014policy} for a survey. The TD learning algorithm was first introduced by \citet{sutton1988learning},
	and convergence guarantees for TD have been proven in both asymptotic and non-asymptotic settings.
	%
	Asymptotic convergence of TD with linear function approximation
	was established in \citet{tsitsiklis_vanroy_97}, 
	and other classical asymptotic guarantees include those due to Borkar and co-authors \citep{borkar2000ode, borkar2009stochastic}. The vanilla TD algorithm can also be combined with the iterate averaging technique, and the asymptotic convergence of this algorithm was shown by \citet{tadic2004almost}, who extended the convergence results for solving noisy linear systems \citep{polyak1992acceleration}.
	
	While asymptotic convergence results offer a proof-of-concept, the algorithm is often run in large-scale applications with relatively small sample sizes. The first results proving finite time convergence  under i.i.d. setting were proposed by \citet{sutton2009fast} and later extended by  \citet{lakshminarayanan2018linear}. 
	Finite-time analysis of TD learning under Markovian noise was carried out by \citet{russo_18}, where the authors employed nonsmooth analysis to a  variant of TD learning, by requiring projections at each iteration onto a pre-specified ball. A consequence of the nonsmooth approach is that there is no obvious way of benefiting from the variance reduction effect of parallel computing. In recent work, a subset of the current authors \citep{kotsalis2020simple} provided an improved analysis of vanilla TD algorithm that overcomes this hurdle. 
	\black{There are also several other notable finite sample analyses of policy evaluation in various settings,  e.g.,~\cite{srikant2019finite,chen2021lyapunov,durmus2021stability, li2021q, min2021variance}, and statistical lower bounds have also been shown for offline reinforcement learning with linear function approximation, e.g., \cite{wang2020statistical,zanette2021exponential}.}
	
	\black{While some of these analyses are sharp, to our knowledge, vanilla TD learning is not known to attain the optimal oracle complexity and instance-dependent sample complexity. \citet{li2021q} recently proved that TD learning achieves the minimax lower bound on stochastic error up to logarithmic factors, but their analysis is not instance dependent.}
	Recent work by \citet{kotsalis2020simple} presented two new algorithms, the conditional temporal difference (CTD) and the fast temporal difference (FTD) learning, where FTD exhibits an accelerated rate in deterministic error. However, these algorithms fail to capture the correct stochastic error in the policy evaluation problem, and the bounds can be shown to be suboptimal for policy evaluation both in an instance-dependent sense, and in the worst-case over natural problem classes.
	
	
	During the past decade, there has been a flurry of parallel work in stochastic optimization on developing first-order methods with variance-reduction; early examples include IAG \citep{blatt2007convergent}, SAG \citep{schmidt2017minimizing}, SVRG \citep{johnson2013accelerating,xiao2014proximal}, and SAGA \citep{defazio2014saga}. 
	\black{There are several papers that apply variance reduction to reinforcement learning, e.g., \cite{du2017stochastic, papini2018stochastic, wai2018multi}.}
	Recent work in policy evaluation has shown that variance reduction techniques can also be applied to algorithms of the TD-type~\citep{korda2015td, wai2019variance, xu2020reanalysis, khamaru2020temporal}. Among these, the paper~\cite{khamaru2020temporal} is motivated by the desire to draw distinctions between RL algorithms with similar worst-case performance and follows a line of work deriving instance-dependent bounds on the stochastic error in policy evaluation~\citep{pananjady2020instance,li2020breaking}.
	Specifically, the results of \cite{khamaru2020temporal} capture the optimal instance-dependent stochastic error in the $\ell_\infty$-norm. However, the optimal sample complexity is not achieved in \cite{khamaru2020temporal} since the algorithm requires $\O\{1/(1-\gamma)^3\}$ samples in each epoch. 

	\subsection{Contributions and organization}
	Towards answering the question posed at the end of Section~\ref{sec:intro}, we make three distinct contributions:
	
	\begin{itemize}
		\item\textbf{Lower bounds.} 
		We construct a worst-case instance that 
		shows an oracle complexity lower bound of order $\Omega\{(1-\gamma)^{-1} \cdot \log(1/\epsilon)\}$ for any \black{iterative method whose iterates lie within the linear span of the initial point $v_0$ and subsequent temporal differences,} to converge to $\epsilon$-error in weighted $\ell_2$-norm. We also prove a lower bound on sample complexity using the classical local minimax theorem  
		\citep{hajek1972local,le1972limits,le2000asymptotics} to provide an instance-specific baseline for algorithm design. 
		
		
		\item\textbf{Algorithm design in the i.i.d. setting.} We start by applying the variance reduction technique to the classical TD algorithm, showing that the resulting variance-reduced temporal difference (VRTD) algorithm \black{nearly matches the optimal stochastic error, but the analysis suggests a suboptimal deterministic error}. This motivates us to further improve the VRTD algorithm with the stochastic operator extrapolation (SOE) device~\citep{GGT_20a}. We provide a sharp analysis of our new algorithm---termed variance-reduced fast temporal difference (VRFTD)---showing that it achieves a convergence rate \black{nearly matching both the deterministic error lower bound (for well-conditioned feature matrices) and the stochastic error lower bound.}
		
		\item\textbf{Extension to the Markovian setting.} 
		\black{
			We extend the VRFTD algorithm to the Markovian setting by introducing a burn-in period during sample collection. We show that the resulting algorithm also achieves similarly fast convergence, with a dominating stochastic error term that matches the instance-dependent lower bound proved in \cite{mou2021optimal} and a deterministic error that matches that of the i.i.d. setting up to a multiplicative factor of the mixing time. In particular, in the so-called realizable case when the approximation error caused by linear function approximation is $0$ (e.g., in the tabular setting), the leading Markovian stochastic error term is equal to the i.i.d. stochastic error term, indicating that the additional dependence on mixing time only appears in terms whose dependence on the final tolerance $\epsilon$ is weak.}
	\end{itemize}
	\smallskip
	
	\noindent The rest of this paper is organized as follows. In Section~\ref{sec_setting}, we formally present the problem setting. The three aforementioned main contributions are presented in Sections~\ref{sec_lower}---\ref{sec_markov}. In Section~\ref{sec_numerical}, we provide numerical experiments that corroborate our optimality guarantees. The proofs of our main results are postponed to Section~\ref{sec_proof}, and auxiliary results are collected in the appendices.   
	
	
	\subsection{Notation} 
	
	For a positive integer n, we define $[n] := \{1, 2, . . . , n\}$. We let $\mathbf{1}$ denote the all-ones vector in $\bbr^D$. We let $e_j$ denote the $j$-th standard basis vector in $\bbr^D$. Let $\mathbb{I}_S:X\rightarrow \{0,1\}$ denote the indicator function of the subset $S \subseteq X$. 
	Given a vector $x\in \bbr^m$, denote its 
	$i$-th entry by\footnote{In situations in which there is no ambiguity, we also use $x_i$ to denote the $i$-th coordinate of a vector $x$.} $x_{(i)}$. Let $\|x\|_1:=\tsum_{i=1}^m |x_{(i)}|$, $\|x\|_2:=\sqrt{\tsum_{t=1}^m x_{(i)}^2}$ and $\|x\|_\infty:=\max_{i\in [m]}|x_{(i)}|$ denote the $\ell_1$, $\ell_2$ and $\ell_\infty$-norms respectively. Given a matrix $A$, denote its $(i,j)$-th entry by $P_{i,j}$. Let $\|A\|_2$ denote the spectral norm of matrix $A$. We let $\lambda_{\min} (A)$ and $\lambda_{\max} (A)$ denote the smallest and largest eigenvalue of a square matrix $A$, respectively.
	For a symmetric positive definite matrix $A$, define the inner product $\langle x, y \rangle_A := x^\top A y$ and the associated norm $\|x\|_A := \sqrt{x^\top A x}$. We refer to $\|x\|_A$ as the $\ell_A$-norm of $x$.

	\section{Background and problem setting}\label{sec_setting}
	In this section, we formally introduce Markov reward processes (MRPs) and the (discounted) policy evaluation problem. We also define linear function approximation of the value function, and present the concrete observation models that we study.
	
	\subsection{Markov reward process and policy evaluation}
	An MRP is described by a tuple $(\mathcal{S}, \mathsf{P}, R, \gamma)$, where $\mathcal{S}=[D]$ denotes the state space, $\mathsf{P}$  is the transition kernel, $R$ is the reward function and $\gamma \in (0,1)$ is  the discount factor. At each iteration, the system moves from the current state $s\in \mathcal{S}$  to some state $s'\in \mathcal{S}$ with probability $\mathsf{P}(s' |s)$, while the agent realizes the reward $R(s, s')$. We denote by $r(s):= \tsum_{s'\in \mathcal{S}} \mathsf{P}(s'|s) R(s,s')$ the expected instantaneous reward generated at state $s$.  Let $P$ denote the transition probability matrix having $(i,j)$-th entry $P_{i,j} = \mathsf{P}(j|i)$. The reward $R$ can also be written in matrix form, i.e., $R_{i,j} = R(i,j)$. The value function specifies the  infinite-horizon discounted reward as a function of the initial state:
	\begin{align*}
		v^*(s):=\bbe\big[\tsum_{t=0}^\infty \gamma^t R(s_t, s_{t+1})|s_0=s\big].
	\end{align*}
	In the case where the number of states is finite and equal to $D$, both the expected reward function $r$ and the value function $v^*$ are $D$-dimensional vectors of reals. The value function is given by the solution to the Bellman equation
	\begin{align}\label{bellman}
		v^* =  \gamma P v^* + r.
	\end{align}

	
	Throughout this paper, we assume that the Markov chain is aperiodic, ergodic and that there exists a unique stationary  distribution $\pi:=(\pi_1,...,\pi_D)$ with strictly positive entries, satisfying $\pi P = \pi$. Let $\Pi:=\text{diag}(\pi_1,...,\pi_D)$ denote a $D \times D$ diagonal matrix whose non-zero elements are given by the entries of the stationary distribution. 
	
	\subsection{Linear function approximation}
	In modern applications with large state spaces, it is common to seek approximate solutions to the Bellman equation~\eqref{bellman}, and the standard approach is to choose a $d$-dimensional subspace $\mathbb{S}$ for the purposes of approximation. In particular, one chooses $\bbs:=\text{span}\{\psi_1,...,\psi_d\}$ for $d$ linearly independent basis vectors $\psi_1, \ldots, \psi_d$.
	For each state $s \in [D]$ we let $\psi(s) := [\psi_1(s), \psi_2(s),...,\psi_d(s)]^\top$ denote its feature vector. 
	Letting $\Pi_\mathbb{S}$ denote the projection onto the subspace with respect to the $\| \cdot \|_\Pi$-norm, 
	define $\vbar$ as the solution to the \emph{projected fixed point equation}
	\begin{align}\label{proj_fixed_point}
		\vbar = \Pi_\bbs(\gamma P \vbar + r).
	\end{align}
	It is convenient to write this projection in matrix notation. Let $\Psi := [\psi_1 ,  \psi_2  , ..., \psi_d ]^\top$, and for $v^\Diamond$ in $\mathbb{S}$, use $\theta^\Diamond$ to denote its corresponding parameterization in $\bbr^d$, e.g., $\Psi^\top \theta'=v'$.
	With this shorthand, equation~\eqref{proj_fixed_point} can be equivalently written as 
	\begin{align}\label{proj_fix_point_2}
		\Psi \Pi \Psi^\top\bar \theta = \Psi \Pi \gamma P \Psi^\top \bar \theta + \Psi \Pi r.
	\end{align}
	
	It is convenient in the analysis to have access to an orthonormal basis spanning the projected space $\mathbb{S}$. Define the matrix $B\in \R^{d\times d}$ by letting $B_{i,j}:= \langle \psi_i, \psi_j\rangle_\Pi $ for each $i,j\index{[d]}$, and let
	$$\Phi:= [\phi_1 ,  \phi_2  , ..., \phi_d ]^\top = B^{-\frac{1}{2}} \Psi.$$ 
	By construction, the vectors $\phi_1, \ldots, \phi_d$ satisfy $\langle \phi_i, \phi_j \rangle_{\Pi} = \mathbb{I}(i = j)$. Next, define the scalars
	\begin{align}\label{beta_mu}
		\beta := \lambda_{\max}(B),~~\text{and} ~\mu:=\lambda_{\min}(B),
	\end{align}
	so that $\beta/\mu$ is the condition number of the covariance matrix of the features. 
	Finally, let 
	\begin{align*}
		M:=\gamma \Phi \Pi P \Phi^\top
	\end{align*}
	denote the $d$-dimensional matrix that describes the action of $\gamma P$ on the projected space $\mathbb{S}$. 
	
	
	\subsection{Observation models and problem statement}\label{sec_obs_model}
	We start by introducing the i.i.d. observation model, in which we have access to a black box or simulator that generates samples from the transition kernel and reward functions.  In particular, we observe independent tuples $\xi_i=(s_i, s_i', R(s_i,s_i'))$, such that 
	$$
	s_i \sim \omega,~~ s_i'\sim P(\cdot|s_i),
	$$
	where $\omega:=(\omega_1, ..., \omega_D)$ is a distribution with strictly positive entries, and we use the shorthand $\Omega:=\diag ([\omega_1, ..., \omega_D])$. 
	A natural and popular choice is $\omega = \pi$, in which case the i.i.d. model is meant to approximate the stationary Markov chain.
	
	In the Markovian noise model, we assume that all of our observations come from a single trajectory of a Markov chain. Precisely, the sequence of states  $\{s_0, s_1....,\}$ generated by the MRP is a time-homogeneous Markov chain, with $s_0 \sim \pi$. The tuple $\xi_t = (s_t,s_{t+1},R(s_t,s_{t+1}))$ is observed at each time $t$. The highly correlated nature of these observations renders algorithm design and analysis in the Markovian setting more challenging than in the i.i.d. setting. 
	
	Our goal in both cases is to use the observations to generate an estimator $\widehat v_n$ of $v^*$ which satisfies an \emph{oracle inequality} of the form
	\begin{align}\label{oracle_inequality}
		\bbe \|\widehat v_n-v^*\|_{\Pi}^2 \leq \O(1)\|\bar v - v^*\|_{\Pi}^2 + \delta_n \|v_0 - \bar v\|_{\Pi}^2 + \epsilon_{n,\sigma},
	\end{align}
	where $v_0$ is the initial iterate of the algorithm.
	The three terms appearing on the RHS of inequality~\eqref{oracle_inequality} all have concrete interpretations. The first term $\|\bar v - v^*\|_{\Pi}^2$ characterizes the \emph{approximation error} incurred by the linear function approximation. As a point of the background, we recall the following instance-dependent upper bound on the approximation error due to \cite{mou2020optimal}:
	\begin{align}\label{approximation_error}
		\|\vbar - v^*\|_{\Pi}^2 \leq \mathcal{A}(M, \gamma) \inf_{v \in S} \|v-v^*\|_{\Pi}^2,
	\end{align}
	where $\mathcal{A}(M,\gamma)=1+\lambda_{\max}\left((I-M)^{-1}(\gamma^2 I_d - M M^\top)(I-M)^{-\top}\right)$.
	See \cite{mou2020optimal} for a proof, alongside guarantees of information-theoretic optimality.
	
	This work focuses on sharply analyzing the last two terms on the RHS of inequality~\eqref{oracle_inequality}, both of which have concrete operational interpretations. The term $\delta_n \|v_0 - \bar v\|_{\Pi}^2$ is the \emph{deterministic error}, which characterizes the convergence of the iterative algorithm in the purely deterministic setting. Specifically, the term $\delta_n$, which should tend to zero as the number of iterations (or oracle calls) $n$ goes to infinity, quantifies how fast the discrepancy between the initialization $v_0$ and the approximate solution $\bar v$ diminishes by running the iterative algorithm. The third term $\epsilon_{n,\sigma}$ is the \emph{stochastic error}, which is incurred due to the stochastic observation model. Here we use the notation $\sigma$ as a placeholder for the ``noise level'' in the observed samples. One should expect the stochastic error $\epsilon_{n,\sigma}$ to go to zero as $n$ goes to infinity or as $\sigma$ goes to zero. 
	Several previous works mix the deterministic error with stochastic error in their guarantees \citep[see, e.g.,][]{russo_18, mou2020optimal}. However, the key benefit of separating the deterministic error from the stochastic error 
	is that it allows a clean understanding of situations in which
	either the observations have low noise or parallel implementation may be available. In these cases, the deterministic error dominates the overall convergence rate of the algorithm, and so having algorithms that attain the optimal deterministic error is a key desideratum. 

	Having precisely defined the deterministic and stochastic errors, we are now in a position to present our first set of results on lower bounds for both of these terms.

	\section{Lower bounds in weighted $\ell_2$-norm}\label{sec_lower}
	We study the oracle complexity lower bound on deterministic error in Section~\ref{sec_deterministic_lower} and the instance-specific stochastic error lower bound in Section~\ref{sec_stochastic_lower_bound}.
	
	
	
	
	
	\subsection{Oracle complexity lower bound on deterministic error}\label{sec_deterministic_lower}
	It is well-known that a linear rate can be achieved for the deterministic policy evaluation problem, and the convergence rate is highly dependent on the effective horizon $(1-\gamma)^{-1}$ \citep{puterman2014markov}.
	Accordingly, our goal in this section is to prove an oracle complexity lower bound in terms of $(1-\gamma)^{-1}$, which can be done even in the tabular setting in which the subspace is all of $\mathbb{R}^D$.
	%
	%
	The following assumption on the oracle captures algorithms in the temporal difference learning family. 
	\begin{assumption}[Amenable iterative method]\label{M_assump}
		An \emph{amenable} iterative method $\mathcal{M}$ generates a sequence of iterates $v_k$ such that 
		\begin{align}\label{assump_M}
			v_k \in v_0 + \mathrm{span}\{G(v_0), G(v_1),...,G(v_{k-1})\},\quad k \geq 1,
		\end{align}
		where $G(v) = (I-\gamma P)v-r$.
	\end{assumption}
	Noting that $G(v)$ is precisely the temporal difference operator applied at the point $v$, an amenable algorithm is one
	whose iterates are always in the linear span of the initial point $v_0$ and subsequent temporal differences. The linear span assumption is commonly used in proving oracle complexity lower bounds \citep{nesterov2003introductory, ouyang2021lower}, and as such, nearly all the algorithms in the temporal difference family can be shown to be amenable.  The sole exceptions that we are aware of occur in cases where there are projections involved in the algorithm, e.g., \cite{russo_18}. However, in policy evaluation problems with unbounded feasible region $\bbr^D$, projection steps are often unnatural and vanilla TD algorithms are able to attain similar performance \citep[see, e.g.][]{kotsalis2020simple}. 
	%
	%
	The following theorem provides an oracle complexity lower bound for policy evaluation problem under the $\ell_\Pi$-norm for amenable algorithms. 
	\begin{theorem}\label{theorem_lower_bound}
		Fix a constant $\gamma>\tfrac{1}{2} $ and an initialization $v_0$. 
		There exists a transition kernel $P$ and an expected reward vector $r$ such that any iterative method $\mathcal{M}$ satisfying Assumption~\ref{M_assump} produces iterates $\{v_k\}_{k \geq 1}$ satisfying the following. If $(D,k)$ satisfies $\tfrac{1-(2\gamma-1)^{2D-2k}}{1-(2\gamma-1)^{2D}}\geq \tfrac{1}{2}$, then
		\begin{align}\label{lower_bound}
			\|v_k-v^*\|_\Pi^2 \geq \tfrac{1}{2} (2\gamma-1)^{2k}\|v_0-v^*\|_\Pi^2,
		\end{align}
		where $v^*$ is the solution of equation~\eqref{bellman}.
	\end{theorem}
	\noindent See Section~\ref{proof_theorem_lower_bound} for the proof of this theorem.
	
	Noting that $2\gamma -1 = 1 - 2(1 - \gamma)$, Theorem~\ref{theorem_lower_bound} shows an oracle complexity lower bound  $\O\{\tfrac{1}{1-\gamma}\log(\tfrac{\| v_0 - v^*\|_\Pi^2}{\epsilon})\}$ for finding a solution $\widehat v \in \R^m$ such that $\|\widehat v - v^*\|_\Pi^2\leq \epsilon$. 
	It should be noted that the metric (i.e., the $\ell_\Pi$-norm) used in Theorem~\ref{theorem_lower_bound} depends on the problem instance through the stationary distribution of the transition kernel $P$. Such an instance-dependent metric makes the construction of our worst-case instance non-standard and challenging. 
	
	On a related note, it is instructive to recall that classical oracle complexity bounds for solving linear operator equations \citep{nemirovsky1991optimality,nemirovsky1992information} allow the conjugate operator to be queried within the oracle, making the method class wider than the class of amenable algorithms captured in  Assumption~\ref{M_assump}. On the one hand, the conjugate operator is not natural for solving a policy evaluation problem under stochastic settings since the vector $(I-\gamma P)^\top v$ is hard to estimate with transition and reward samples. On the other hand, our construction used in proving Theorem~\ref{theorem_lower_bound} naturally extends to this wider method class, and we provide an even stronger deterministic error lower bound than Theorem~\ref{theorem_lower_bound} in Appendix~\ref{sec_app_B}.
	
	\subsection{Instance-specific lower bound on stochastic error}\label{sec_stochastic_lower_bound}
	We now turn our attention to proving lower bounds on the instance-specific sample complexity under the i.i.d. observation model introduced in Section~\ref{sec_obs_model}. We assume that the feature matrix $\Psi$ is fixed and known, 
	and let $\param=(\omega, P, R)$ denote an individual problem instance parameterized by the initial state distribution $\omega$, transition kernel $P$, and reward function $R$. 
	Note at this juncture that we do not require that $\omega = \pi$; this is akin to the so-called \emph{off-policy} situation in which the sampling (or behavior) policy may differ from the policy that we are interested in evaluating. 
	Our result will apply in this general case; 
	but given that the initial state is drawn from the distribution $\omega$,
	it is convenient to consider solving the projected fixed point equation with respect to the $\|\cdot\|_\Omega$-norm (cf. Eq.~\eqref{proj_fix_point_2}), written as 
	\begin{align*}
		\Psi \Omega\Psi^\top\theta = \Psi \Omega \gamma P \Psi^\top \theta + \Psi \Omega r.
	\end{align*}
	Use the function $\bar \theta(\param):=(\Psi \Omega \Psi^\top- \Psi \Omega \gamma P \Psi^\top)^{-1} \Psi \Omega r$ to denote the target of interest. 
	
	
	In order to state our result, we require some additional notation. Fix an instance $\param=(\omega, P, R)$, and for any $\epsilon > 0$, define an $\epsilon$-neighborhood of problem instances by
	\begin{align*}
		\mathfrak{N}(\param; \epsilon) := \{\param' = (\omega ', P', R '): \|\omega - \omega '\|_2+ \|P- P'\|_F +\|R- R'\|_F\leq \epsilon\}.
	\end{align*}
	Define the matrix $\widetilde B \in \bbr^{d\times d}$ by $\widetilde B_{i,j}:=\langle \psi_i, \psi_j \rangle_\Omega$ for $i, j \in [d]$.  Thus $\widetilde B$ satisfies
	\begin{align}\label{def_B_tilde}
		\widetilde B^{-\frac{1}{2}}\Psi \Omega \Psi\widetilde B^{-\frac{1}{2}} = I_d.
	\end{align}
	Adopting the $\ell_{\widetilde B}$-norm as our loss function, define the following local asymptotic minimax risk \citep{hajek1972local,le1972limits}:
	\begin{align}\label{minimax_risk}
		\mathfrak{M}(\param) := \lim_{c \rightarrow \infty} \lim_{N\rightarrow \infty} \inf_{\widehat \theta_N} \sup_{\param' \in\mathfrak{N}(\param;c/\sqrt{N})} N\cdot\bbe_{\param'}\left[\left\|\widehat \theta_N - \bar\theta(\param')\big)\right\|_{\widetilde B}^2\right].
	\end{align}
	The infimum in Eq.~\eqref{minimax_risk} is taken over all estimators $\widehat \theta_N$ that are measurable functions of $N$ observations drawn according to the i.i.d. observation model. In contrast to the global minimax risk---which takes a supremum of the risk over all the problem instances within a reasonable class---the local minimax risk $\mathfrak{M}(\param)$ looks for the hardest alternative in a small neighborhood of the instance $\param$ with diameter $c/\sqrt{N}$. To capture the hardest local alternative (in an asymptotic sense) it suffices to take the diameter of the neighborhood to be of the order $1/\sqrt{N}$. Invoking Eq.~\eqref{def_B_tilde} yields the equivalent definition
	\begin{align}\label{minimax_risk_2}
		\mathfrak{M}(\param) = \lim_{c \rightarrow \infty} \lim_{N\rightarrow \infty} \inf_{\widehat \theta_N} \sup_{\param' \in\mathfrak{N}(\param;c/\sqrt{N})} N\cdot\bbe_{\param'}\left[\left\|\Psi^\top \widehat \theta_N - \Psi^\top\bar\theta(\param')\big)\right\|_{\Omega}^2\right].
	\end{align}
	The following proposition characterizes the local asymptotic risk $\mathfrak{M}(\param)$ explicitly.

	\begin{proposition}\label{theorem_stochastic_lower_bound}
		Consider the i.i.d. observation model with the initial state drawn from the distribution $\omega$. Let $Z\in \bbr^d$ be a multivariate Gaussian 
		\begin{align*}
			Z \sim \mathcal{N} \big(0,  (I_d-\widetilde M)^{-1}  \widetilde\Sigma (I_d- \widetilde M)^{-T}\big),
		\end{align*}
		where $\widetilde\Sigma := \cov\big[\widetilde B^{-\frac{1}{2}}
		\big(\langle\psi(s) - \gamma \psi(s'),\bar\theta\rangle - R(s, s')\big) \psi(s)\big]$ and $\widetilde M:= \gamma\widetilde B^{-\frac{1}{2}}\Psi \Omega P \Psi\widetilde B^{-\frac{1}{2}}$. Then we have 
		\begin{align}\label{stats_lower_bound}
			\mathfrak{M}(\param) = \bbe[\|Z\|_2^2] = \trace\left\{ (I_d-\widetilde M)^{-1}  \widetilde \Sigma (I_d- \widetilde M)^{-T}\big) \right \}.
		\end{align}
	\end{proposition} 
	\noindent See Section~\ref{proof_stochastic_lower_bound} for the proof of this theorem. 
	
	A few comments are in order. First, it should be noted that this lower bound is distinct from the asymptotic minimax lower bound shown in \cite{khamaru2020temporal}, in which a generative observation model is assumed (where we observe transitions from all $D$ initial states) and there is no function approximation. 
	Consequently, our choice of a problem instance of interest is $\param=(\omega, P, R)$ rather than $(P, R)$ in \cite{khamaru2020temporal}. Second, and on a related note, it is important that $\omega$ be unknown and included in the set of parameters $\param$; if in contrast $\omega$ is known a priori, then the local asymptotic minimax risk differs from the characterization~\eqref{stats_lower_bound}. Finally, we note that \cite{mou2020optimal} provide non-asymptotic, instance-dependent lower bounds on stochastic error for solving projected fixed-point equations using the Bayesian Cram\'{e}r--Rao bound. However, these lower bounds do not directly apply here, 
	since the family of hardest local alternatives constructed in \cite{mou2020optimal} may not be valid instances in the policy evaluation setting.
	
	Let us now specialize Proposition~\ref{theorem_stochastic_lower_bound} by taking $\omega = \pi$, where $\pi$ is the stationary distribution of the transition kernel $P$.  Denote by $\param_\pi:=(\pi, P, R)$ the instance of interest. Let
	\begin{align}\label{def_Sigma}
		\iidS := \cov\big[B^{-\frac{1}{2}}
		\big(\langle\psi(s) - \gamma \psi(s'),\bar\theta\rangle - R(s, s')\big) \psi(s)\big] \quad \text{ for } s\sim \pi \text{ and } s'\sim \mathsf{P}(\cdot|s).
	\end{align}
	Applying Proposition~\ref{theorem_stochastic_lower_bound}, the local asymptotic minimax risk~\eqref{minimax_risk_2} under this setting is then given by 
	\begin{align}\label{stats_lower_bound_1}
		\lim_{c \rightarrow \infty} \lim_{N\rightarrow \infty} \inf_{\widehat \theta_N} \sup_{\param'\in\mathfrak{N}(\param_\pi;c/\sqrt{N})} N\cdot\bbe_{\param'}&\left[\left\|\Psi^\top \widehat \theta_N - \Psi^\top\bar\theta(\param')\big)\right\|_{\Pi}^2\right] \nn\\
		&= \trace\left\{ (I_d-M)^{-1}  \iidS (I_d-  M)^{-T}\big) \right \}.
	\end{align}
	
	Taking stock, we have proved two lower bounds~\eqref{lower_bound} and~\eqref{stats_lower_bound_1} on the deterministic and stochastic errors in $\ell_\Pi$-norm under the i.i.d. observation model $s \sim \pi$ and $s' \sim \mathsf{P}(\cdot|s)$. Given these baselines, it is natural to ask whether there is a practical \black{iterative} algorithm in the TD family that can achieve both lower bounds, which is the main focus of Section~\ref{sec_iid}.

	
	\section{Algorithms for policy evaluation in the i.i.d. setting}\label{sec_iid}
	Taking both lower bounds proved in Section~\ref{sec_lower} as our baseline, we now turn our attention to the question of algorithm design. 
	In this section, we assume the i.i.d. observation model introduced in Section~\ref{sec_obs_model} with $\omega=\pi$.
	In order to state the results clearly, we require some additional notation. For $\theta \in \bbr^d$, we define the deterministic operator for solving equation~\eqref{proj_fix_point_2} as
	\begin{align}\label{deterministic_opt}
		g(\theta) = \Psi \Pi(\Psi^\top \theta- r-\gamma P \Psi^\top\theta);
	\end{align}
	note that $\bar\theta$ is the solution to $g(\theta) = 0$
	The corresponding stochastic operator calculated from sample $\xi_i$ is defined as
	\begin{align}\label{stochastic_opt}
		\widetilde g(\theta, \xi_i) = \left( \langle \psi(s_i), \theta\rangle - R(s_i,s_i') - \gamma \langle\psi(s_i'), \theta \rangle  \right) \psi(s_i);
	\end{align}
	note that $\bbe_{s_i\sim \pi, s_i'\sim \mathsf{P}(\cdot|s_i)}[\widetilde g(\theta, \xi_i)] =g(\theta) $. 
	To characterize the ``variance" of the stochastic operator under the i.i.d. observation model, we make the following assumption:
	\begin{assumption} \label{assump_variance}
		There exists a constant $\varsigma\geq 0$ such that for every $\theta, \theta' \in \bbr^d$, 
		\begin{align}\label{variance_3}
			\bbe\|\widetilde g(\theta,\xi) - \widetilde g(\theta', \xi)-\big(g(\theta)  - g(\theta')\big)\|^2_2 \leq \varsigma^2\|v - v'\|_\Pi^2,
		\end{align}
		where $v=\Psi^\top \theta$ and $v' = \Psi^\top \theta'$.
	\end{assumption}
	In words, instead of bounding the ``variance" of the stochastic operator directly as in  \cite{kotsalis2020simple}, Assumption~\ref{assump_variance} guarantees that the variance of the difference between stochastic operators with different variables $\theta$ under the same data~$\xi$ is upper bounded by the distance between the variables. This assumption is critical for implementing the variance-reduction techniques and capturing the instance-dependent stochastic error at the approximate solution $\bar \theta$. Clearly, the parameter $\varsigma^2$ is bounded provided the features $\psi(s)$ are bounded, and provides a natural measure of ``noise" in the problem.
	Accordingly, we make use of Assumption~\ref{assump_variance} throughout Sections~\ref{sec_iid} and~\ref{sec_markov}. 
	
	We are now ready to present our algorithms. We start with a variance-reduced version of the TD algorithm that captures the instance-specific stochastic error lower bound but fails to achieve the oracle complexity lower bound on deterministic error. To remedy this issue, we develop an accelerated variance-reduced TD algorithm that matches both lower bounds proved in Section~\ref{sec_lower}.
	
	\subsection{A warm-up algorithm: variance-reduced temporal difference learning}
	
	Variance-reduced temporal difference learning (VRTD) solves the policy evaluation problem using epochs. With a slight ambiguity of notation, we let $v_t$ and its corresponding parameterization $\theta_t$ denote the iterates generated within each epoch, and let $v^0$ and its corresponding parameterization $\theta^0$ denote the initialization of the algorithm. At the beginning of each epoch $k$, the algorithm uses $N_k$ samples to compute an averaged stochastic operator $\widehat g$ and evaluates it at a point $\widetilde \theta$, where $\widetilde \theta$ should be understood as the best current approximation of the optimal solution. The vector $\widehat g(\widetilde \theta)$ is used to recenter the updates in each epoch. 
	
	\begin{algorithm}[H]  \caption{Variance-reduced Temporal Difference Algorithm under i.i.d observations}  
		\label{alg:TD_iid}   
		\begin{algorithmic} 
			\STATE{\textbf{Input}: $ \theta^0 = \widehat \theta_0 \in \R^d$, $\eta > 0$,  $\{\zeta_t\}_{t=1}^T\geq0$ and $\{N_k\}_{k=1}^{K} \subset \mathbb{Z}_+$.}
			\FOR{$ k = 1, \ldots, K$}
			\STATE{Set $\theta_1 = \widetilde \theta  = \widehat \theta_{k-1}$. Collect $N_k$ samples $\xi_i^k=(s_i, s_i', R(s_i,s_i'))$ from the i.i.d. model.}
			\STATE{Calculate $\widehat g(\widetilde\theta) = \tfrac{1}{N_k} \tsum_{i=1}^{N_k} \widetilde g(\widetilde \theta,\xi_i^k).$}
			\FOR{$t=1, \ldots, T$}
			\STATE{ Collect a sample $\xi_{t}=(s_t, s_t', R(s_t,s_t'))$ from the i.i.d. observation model and compute
				\beq \label{TD_iid_step}
				\theta_{t+1} = \theta_t - \eta \left(\widetilde g(\theta_t,\xi_t)-\widetilde g(\widetilde \theta,\xi_t) + \widehat g(\widetilde \theta)\right).
				\eeq}
			\ENDFOR
			\STATE{Output of the epoch: \beq \label{eq:VRTD-ave-output}
				\widehat \theta_k = \frac{\sum_{t=1}^{T+1}\zeta_t \theta_t}{\sum_{t=1}^{T+1}\zeta_t }.\eeq}
			\ENDFOR
		\end{algorithmic}
	\end{algorithm} 
	
	Note that this algorithm is distinct from previous instantiations of variance-reduced temporal difference algorithms \citep{xu2020reanalysis, khamaru2020temporal}, since the output of each epoch~\eqref{eq:VRTD-ave-output} is a weighted average of the iterates.
	The following theorem provides a convergence guarantee on the VRTD algorithm.
	\begin{theorem}\label{theorem_VRTD_2}
		Consider the i.i.d. observation model with the initial state drawn from the distribution~$\pi$. Fix the total number of epochs $K$ and a positive integer $N$. Assume that for each epoch $k\in [K]$, the parameters $\eta$, $N_k$ and $T$ satisfy 
		\begin{align*}
			\eta\leq \min\left\{ \tfrac{1-\gamma}{2\beta(1+\gamma)^2},\tfrac{1-\gamma}{32\varsigma^2}\right\},~~T\geq \tfrac{32}{\mu(1-\gamma)\eta},~\text{and}~~N_k\geq\left\{\tfrac{38\varsigma^2}{\mu(1-\gamma)^2},(\tfrac{3}{4})^{K-k}N\right\}.
		\end{align*}
		Set the output of the epoch to be $\widehat v_k:=\frac{\sum_{t=1}^T \eta (1-\gamma) v_t+(1/\beta )v_{T+1}}{T \eta (1-\gamma) +(1/\beta )}$.
		Then for each $\delta>0$, we have 
		\begin{align} \label{eq:VRTD-conc}
			\bbe[\|\widehat v_K - v^*\|_\Pi^2] &\leq(1+\delta)  \mathcal{A}(M, \gamma) \inf_{v \in S}\|v-v^*\|_\Pi^2  \nn\\
			&\quad+ (1+\tfrac{1}{\delta})\left[\tfrac{1}{2^K}\|v^0-\vbar\|_\Pi^2+ \tfrac{15}{N}\trace\left((I_d-M)^{-1}\iidS (I_d-M)^{-\top}\right)\right].
		\end{align}
	\end{theorem}
	\noindent See Section~\ref{proof_theorem_VRTD} for the detailed proof of this theorem.
	
	The first term in the bound~\eqref{eq:VRTD-conc} is the approximation error term alluded to previously; let us extract the deterministic and stochastic errors from the remaining terms.
	The number of epochs required by the VRTD method to find a solution $\widehat v \in \R^D$, such that $\bbe[\|\widehat v - \vbar\|_\Pi^2]\leq \epsilon$ is bounded by $\O\big\{\log(\|v^0 - \vbar\|_\Pi^2/\epsilon)\big\}$. The total number of samples used is $\tsum_{k=1}^K(T+N_k)$, which is of the order
	\begin{align}\label{complexity_1}
		\underbrace{ \tfrac{\beta }{(1-\gamma)^2\mu} \log(\tfrac{\|v^0 - \vbar\|_\Pi^2}{\epsilon})}_{\text{deterministic error}} + \underbrace{\tfrac{\varsigma^2 }{(1-\gamma)^2\mu} \log(\tfrac{\|v^0 - \vbar\|_\Pi^2}{\epsilon})+\tfrac{\trace\big((I_d-M)^{-1}\iidS(I_d-M)^{-\top}\big)}{\epsilon}}_{\text{stochastic error}}. 
	\end{align}
	A few comments on the upper bound provided in Theorem~\ref{theorem_VRTD_2} are in order. 
	
	\paragraph{Comparing the upper and lower bounds} 
	We first focus on the stochastic error in the bound~\eqref{complexity_1}. The VRTD algorithm requires at least $\mathcal{O}\{\tfrac{\varsigma^2}{(1-\gamma)^2\mu}\}$ samples in each epoch, which accounts for the first term. Note that with noisy observations,
	it is necessary to have $\mathcal{O}\{\tfrac{1}{(1-\gamma)^2}\}$ samples in order to obtain an estimate of the value function within $\mathcal{O}(1)$ error, so this higher-order term is natural. The dominating stochastic error term is the last term, and matches the lower bound in equation~\eqref{stats_lower_bound_1}. Therefore the VRTD algorithm is instance-optimal in terms of its stochastic error. 
	
	Next, we turn our attention to the deterministic error, noticing  
	that the dependence on $\tfrac{1}{1-\gamma}$ is quadratic. Comparing with the oracle complexity lower bound proved in Theorem~\ref{theorem_lower_bound}, this quadratic dependence is suboptimal. This shortcoming motivates us to develop an accelerated algorithm in the next subsection.
	
	\paragraph{Comparing with related work} 
	To our knowledge, the only work using variance reduction that captures the correct instance-specific stochastic error is that of \citet{khamaru2020temporal}, which showed that the VRTD algorithm can match the lower bound on stochastic error in $\ell_\infty$-norm. However, their guarantees require $\mathcal{O}\{\tfrac{1}{(1-\gamma)^3} \}$ samples in each epoch to compute the recentered update, and this sample size is suboptimal. In addition, the deterministic error proved in this paper is of the order $ \tfrac{1}{(1-\gamma)^2\epsilon}$ which is also suboptimal.\footnote{In more detail, a family of such deterministic error guarantees is possible to extract from the paper. The dependence on $\epsilon$ can be improved but the dependence on $(1-\gamma)^{-1}$ is at least quadratic.} 
	The work of \citet{mou2020optimal} provided an analysis for the Polyak--Ruppert averaged temporal difference learning algorithm with linear function approximation in the weighted $\ell_2$-norm, showing that the dominant stochastic error term matches the stochastic lower bound proved in Proposition~\ref{theorem_stochastic_lower_bound}. However, the sample complexity is suboptimal due to the presence of higher-order terms \citep[see][for simulations demonstrating this suboptimality]{khamaru2020temporal}, as is the oracle complexity.


	\subsection{Variance-reduced fast temporal temporal difference algorithm}
	Motivated by the suboptimality of VRTD in its oracle complexity, we now present a variance-reduced ``fast" temporal difference (VRFTD) algorithm, which incorporates the idea of operator extrapolation introduced in \cite{GGT_20a}. This serves to accelerate the algorithm, and our analysis of VRFTD shows a convergence rate matching both the deterministic and stochastic error lower bounds. 
	
	The VRFTD algorithm is formally presented in Algorithm~\ref{alg:FTD_2}, and we introduce the basic idea of the algorithm below. First, it utilizes the idea of recentering updates from VRTD with the operator $\widehat g(\widetilde \theta)$ used in each epoch.  Second, in terms of iterate updating in the inner loop, it involves an inner mini-batch that generates the averaged operator $\widetilde g_t$, \black{which allows the algorithm to be run with a much larger stepsize.} Finally, each iteration within an epoch involves an operator extrapolation step~\eqref{FTD_step_2}. This is crucial to achieving the optimal deterministic error (cf. the VRTD update~\eqref{TD_iid_step}).
	\begin{algorithm}  \caption{Variance-reduced Fast Temporal Difference Algorithm under i.i.d. observations}  
		\label{alg:FTD_2}
		\begin{algorithmic} 
			\STATE{\textbf{Input}: $ \theta^0 = \widehat \theta_0 \in \R^d$, $\eta > 0$, $\lambda\geq 0$, $\{\zeta_t\}_{t=1}^T\geq0$ and nonnegative integers $m$, $\{N_k\}_{k=1}^{K}$.}
			\FOR{$ k = 1, \ldots, K$}
			\STATE{Set $\theta_0 = \theta_1 = \widetilde \theta  = \widehat \theta_{k-1}$. Collect $N_k$ sample tuples $\xi_i^k=(s_i, s_i', R(s_i,s_i'))$ from the i.i.d. observation model.}
			\STATE{Calculate $\widehat g(\widetilde\theta) = \tfrac{1}{N_k} \tsum_{i=1}^{N_k} \widetilde g(\widetilde \theta,\xi_i^k).$}
			\FOR{$t=1, \ldots, T$}
			\STATE{ Collect $m$ sample tuples $\xi_j^t=(s_j, s_j', R(s_j,s_j'))$ from the i.i.d. observation model.}
			\STATE{
				Calculate $\widetilde g_{t}(\cdot) = \tfrac{1}{m}\tsum_{j=1}^{m} \widetilde g(\cdot,\xi_j^t)$.}
			\STATE{
				Denote $\widetilde F_t(\theta_t) = \widetilde g_t(\theta_t)-\widetilde g_t(\widetilde \theta) +\widehat g(\widetilde \theta)$. Set $\widetilde F_0(\theta_0)=\widetilde F_1(\theta_1)$. Let}
			\STATE{
				\beq \label{FTD_step_2}
				\theta_{t+1} = \theta_t - \eta \left[\widetilde F_t(\theta_t)+\lambda \left(\widetilde F_t(\theta_t)-\widetilde F_{t-1}(\theta_{t-1})\right)\right].
				\eeq}
			\ENDFOR
			\STATE{Output of the epoch: \beq\widehat \theta_k = \frac{\sum_{t=1}^{T+1}\zeta_t \theta_t}{\sum_{t=1}^{T+1}\zeta_t }.\eeq}
			\ENDFOR
		\end{algorithmic}
	\end{algorithm} 

	The following theorem establishes a convergence rate for the VRFTD  algorithm.
	\begin{theorem}\label{theorem_VRFTD}
		Fix the total number of epochs $K$ and a positive integer $N$. Assume that for each epoch,  the parameters $\eta$, $\lambda$, $m$, $N_k$, $T$ satisfy 
		\beq\label{stepsize_VRFTD_1}
		\eta \leq \tfrac{1}{4\beta(1+\gamma)}, ~\lambda=1, ~T\geq\tfrac{32}{\mu(1-\gamma)\eta}, ~m \geq \max\left\{1, \tfrac{256\eta \varsigma^2}{1-\gamma}\right\}~\text{and}~N_k\geq \max\left\{\tfrac{56\varsigma^2}{\mu(1-\gamma)^2},  (\tfrac{3}{4})^{K-k} N\right\}.
		\eeq
		Set the output of this epoch to be $\widehat v_k:=\frac{\sum_{t=2}^{T+1} v_t}{T}$. Then for $\delta>0$,
		\begin{align*}
			\bbe[\|\widehat v_K - v^*\|_\Pi^2] &\leq(1+\delta)  \mathcal{A}(M, \gamma) \inf_{v \in S}\|v-v^*\|_\Pi^2\\&\quad  + (1+\tfrac{1}{\delta})\left[\tfrac{1}{2^K}\|v^0-\vbar\|_\Pi^2 + \tfrac{15}{N}\trace\left((I_d-M)^{-1}\iidS(I_d-M)^{-\top}\right)\right].
		\end{align*}
	\end{theorem}
	\noindent See Section~\ref{proof_theorem_VRFTD} for the detailed proof of this theorem.  
	
	In view of Theorem~\ref{theorem_VRFTD},  the number of epochs required by the VRFTD method to find a solution $\widehat v \in \R^D$, such that $\bbe[\|\widehat v - \vbar\|_\Pi^2]\leq \epsilon$ is bounded by $\O\{\log(\|v^0 - \vbar\|^2_\Pi/\epsilon)\}$. The total number of samples used is $\tsum_{k=1}^K(mT+N_k)$, which is bounded on the order	
	\begin{align}\label{complexity_VRFTD_iid}
		\underbrace{ \tfrac{\beta }{(1-\gamma)\mu} \log(\tfrac{\|v^0 - \vbar\|_\Pi^2}{\epsilon})}_{\text{deterministic error}} + \underbrace{\tfrac{\varsigma^2 }{(1-\gamma)^2\mu} \log(\tfrac{\|v^0 - \vbar\|_\Pi^2}{\epsilon})+\tfrac{\trace\big((I_d-M)^{-1}\iidS(I_d-M)^{-\top}\big)}{\epsilon}}_{\text{stochastic error}}. 
	\end{align}
	Similar to the VRTD algorithm, the VRFTD algorithm achieves optimal sample complexity in terms of stochastic error. For the deterministic error,
	the dependence on $1/(1-\gamma)$ matches the oracle complexity lower bound proved in Theorem~\ref{theorem_lower_bound}. \black{Note that the term $\beta/\mu$ is the condition number of the feature matrix in $\ell_\Pi$-norm. Therefore, for  ``well-conditioned'' feature matrices, the VRFTD algorithm achieves optimal oracle complexity\footnote{It is an interesting open problem to prove an oracle complexity lower bound for policy evaluation with linear function approximation having linear dependence on $\beta/\mu$.}.}
	In summary, VRFTD is an accelerated and instance-optimal policy evaluation algorithm, 
	and answers the central question posed in this paper.

	\section{Algorithm for policy evaluation in the Markovian setting}\label{sec_markov}
	Finally, we extend the VRFTD algorithm to the Markovian setting, 
	noting in passing that such an extension is also possible for the VRTD algorithm. 
	%
	The challenge of Markovian noise stems from the presence of dependent data that leads to biased samples. 
	To control the bias caused by correlation, we need a standard ergodicity assumption on the underlying Markov chain.
	\begin{assumption}\label{assump_rho_0}
		There exist constants $\skipcon>0$ and $\rho\in(0,1)$ such that
		\begin{align}\label{assump_rho}
			\max_{s\in S} \|\mathbb{P}(s_t=\cdot|s_0=s)-\pi\|_{\infty}\leq \skipcon \cdot \rho^t \quad \text{ for all } \; t\in \mathbb{Z}_+.
		\end{align}
	\end{assumption}
	\noindent In other words, with the following definition of mixing time
	\begin{align*}
		\tmix := \inf\{ t \in \mathbb{Z}_+ \;\mid\; \max_{s\in S} \|\mathbb{P}(s_t=\cdot|s_0=s)-\pi\|_{\infty} \leq 1/4 \},
	\end{align*}
	Assumption~\ref{assump_rho_0} guarantees that the mixing time is bounded as\footnote{Note that while our choice of the constant $1/4$ in the definition is arbitrary,  there is no additional dependence on $\epsilon$ when accounting for the mixing time, unlike in the assumptions made by~\cite[Eq.~(21)]{russo_18}.} $\tmix \leq \frac{\log (4 \skipcon)}{\log(1 / \rho)}$.
	
	\black{In order to overcome the difficulty caused by highly-correlated data, we introduce a burn-in period for sample collection. For instance, to compute the operator $\widehat g$ defined in Algorithm~\ref{alg:FTD_2}, we collect $N_k$ successive samples and only use the last $N_k-n_0$ of them. 
		With this method, we are able to reduce the bias induced by Markovian samples and achieve the desired variance reduction properties. The following two lemmas make this quantitative. }
	\begin{lemma}\label{lemma_operator_bias_1}
		For every $t,\tau\in \mathbb{Z}_{+}$, with probability 1,
		\begin{align}\label{bias_2}
			\|\bbe[\widetilde g(\bar \theta,\xi_{t+\tau})|\mathcal{F}_t]-g(\bar \theta)\|_2 \leq \mixcon\cdot \rho^\tau \|\vbar -v^*\|_\Pi.
		\end{align}
		where $\mixcon := \tfrac{\skipcon}{\sqrt{\min_{i\in[D]} \pi_i}}\|\Psi\|_2 \|I-\gamma P\|_2$ and $\mathcal{F}_t:= [\xi_1,...,\xi_t]$.
	\end{lemma}
	\noindent See Appendix~\ref{proof_operator_bias_1} for a proof of this lemma.
	In words, Lemma~\ref{lemma_operator_bias_1} provides an upper bound on the bias of the stochastic operator at the solution $\bar \theta$ in terms of the approximation error, and the bound decays exponentially with~$\tau$.
	
	\begin{lemma}\label{lemma_operator_bias_2}
		For every $t,\tau\in \mathbb{Z}_{+}$ and $\theta, \theta'\in \bbr^d$, with probability 1,
		\begin{align}\label{bias_1}
			\|\bbe[\widetilde g(\theta,\xi_{t+\tau})|\mathcal{F}_t] -\bbe[\widetilde g(\theta',\xi_{t+\tau})|\mathcal{F}_t] - [g(\theta)-g(\theta')] \|_2 \leq \mixcon\cdot \rho^\tau \|v-v'\|_\Pi.
		\end{align}
	\end{lemma}
	\noindent See Appendix~\ref{proof_operator_bias_2} for a proof of this lemma.	
	In contrast to Lemma~\ref{lemma_operator_bias_1}, Lemma~\ref{lemma_operator_bias_2} provides an upper bound on the bias of the \emph{difference} of stochastic operators, which allows us to get rid of any dependence on the approximation error.
	We are now ready to formally state the VRFTD algorithm in the Markovian noise setting in Algorithm~\ref{alg:VRFTD_markovian}. 
	
	\begin{algorithm}[H]  \caption{Variance-reduced Fast Temporal Difference Algorithm under Markovian noise}  
		\label{alg:VRFTD_markovian}   
		\begin{algorithmic} 
			\STATE{\textbf{Input}: $ \theta^0 = \widehat \theta_0 \in \R^d$, $\eta > 0$, $\lambda\geq 0$, $\{\zeta_t\}_{t=1}^T\geq0$ and nonnegative integers \black{$m$, $m_0$, $n_0$}, $\{N_k\}_{k=1}^{K}$.}
			\FOR{$ k = 1, \ldots, K$}
			\STATE{Set $\theta_1 = \widetilde \theta= \widehat \theta_{s-1}$. \black{Collect $N_k$ successive samples $\xi_i^k:=(s_i, s_{i+1}, R(s_i,s_{i+1}))$ from the single Markov trajectory. Calculate \beq \label{g_hat_3}
					\widehat g(\widetilde \theta) = \tfrac{1}{N_k-n_0} \tsum_{i=n_0+1}^{N_k} \;\; \widetilde g(\widetilde \theta,\xi_i^k).\eeq}}
			\FOR{$t=1, \ldots, T$}
			\STATE{ \black{Collect $m$ successive samples $\widehat \xi_{j}^t:=(s_{j}, s_{j+1}, R(s_{j},s_{j+1}))$ from the Markov trajectory.}} 
			\STATE{ \black{Calculate $\widetilde g_t(\cdot) = \tfrac{1}{m-m_0}\tsum_{j=m_0+1}^{m}\widetilde g(\cdot,\widehat \xi_{j}^t)$.}}
			\STATE{ 
				Let $\widetilde F_t(\theta_t) = \widetilde g_t(\theta_t)-\widetilde g_t(\widetilde \theta) +\widehat g(\widetilde \theta)$ and set $\widetilde F_0(\theta_0)=\widetilde F_1(\theta_1)$. Let
				\beq \label{VRFTD_markov_2}
				\theta_{t+1} = \theta_t - \eta \left[\widetilde F_t(\theta_t)+\lambda \left(\widetilde F_t(\theta_t)-\widetilde F_{t-1}(\theta_{t-1})\right)\right].
				\eeq}
			\ENDFOR
			\STATE{Output of the epoch: \beq\widehat \theta_k = \frac{\sum_{t=1}^{T+1}\zeta_t \theta_t}{\sum_{t=1}^{T+1}\zeta_t }.\eeq}
			\ENDFOR
		\end{algorithmic}
	\end{algorithm} 
	\black{
		Before presenting our main convergence result for the VRFTD algorithm, we first define the matrix $\MarkovS$, which is a covariance matrix analog for the Markovian case (see \cite{mou2021optimal} and references therein). Letting $\{\widetilde s_t\}_{t=-\infty}^{\infty}$ define a sequence of samples obtained from a \emph{stationary} Markov trajectory, define
		$$
		\MarkovS:= \sum_{t=-\infty}^{\infty}B^{-\tfrac{1}{2}} ~\bbe\left[\big(\widetilde g(\bar \theta, \widetilde \xi_t)-g(\bar \theta)\big)\big(\widetilde g(\bar \theta, \widetilde \xi_0)-g(\bar \theta)\big)^\top\right]B^{-\tfrac{1}{2}},
		$$
		where $\widetilde \xi_t:=(\widetilde s_t, \widetilde s_{t+1}, R(\widetilde s_t, \widetilde s_{t+1}))$. This matrix is an infinite sum of matrices where one of the summands (when $t=0$) is the matrix $\iidS$ defined in Eq.~\eqref{def_Sigma}. }
	
	\black{Similarly to the i.i.d. setting, the instance-dependent complexity of Markovian linear stochastic approximation was shown in~\cite{mou2021optimal} to be governed by the trace of the matrix $(I_d-M)^{-1} \MarkovS (I_d-M)^{-T}$.
		To interpret this functional, consider the special case in which the approximation error caused by linear function approximation is $0$, i.e., $v^* = \bar v$. 
		Let $\widetilde{\mathcal{F}_i}$ denote the $\sigma$-field generated by samples $\widetilde \xi_0, ..., \widetilde \xi_i$ and let $\widetilde \Pi_j^i:= \diag\{[\mathbb{P}(\widetilde s_j=1|\widetilde s_i),...,\mathbb{P}(\widetilde s_j=D|\widetilde s_i)]\}$ for $j\geq i$. Then we have
		\begin{align*}
			&\bbe\left\langle (I_d-M)^{-1}B^{-\frac{1}{2}}\big(\widetilde g(\bar \theta, \widetilde \xi_0)-g(\bar \theta)\big), (I_d-M)^{-1}B^{-\frac{1}{2}}\big(\widetilde g(\bar \theta, \widetilde \xi_i)-g(\bar \theta)\big) \right\rangle\nn\\
			& =\bbe\left\langle (I_d-M)^{-1}B^{-\frac{1}{2}}\big(\widetilde g(\bar \theta, \widetilde \xi_0)-g(\bar \theta)\big), (I_d-M)^{-1}B^{-\frac{1}{2}}\big(\bbe[\widetilde g(\bar \theta, \widetilde \xi_i)|\widetilde{\mathcal{F}}_0]-g(\bar \theta)\big) \right\rangle\nn\\
			& = \bbe\left\langle (I_d-M)^{-1}B^{-\frac{1}{2}}\big(\widetilde g(\bar \theta, \widetilde \xi_0)-g(\bar \theta)\big), (I_d-M)^{-1}B^{-\frac{1}{2}}\Psi (\widetilde \Pi_j^i - \Pi) (\Psi^\top \bar \theta - \gamma P\Psi^\top \bar \theta - r) \right\rangle = 0,
		\end{align*}
		where the last equation follows from the fact that $v^* = \bar v = \Psi^\top \bar \theta$ and the Bellman equation \eqref{bellman}. Then 
		\begin{align*}
			\trace\left((I_d-M)^{-1}\MarkovS(I_d-M)^{-\top}\right) = \trace\left((I_d-M)^{-1}\iidS(I_d-M)^{-\top}\right).
		\end{align*}
		Armed with this intuition, we are now ready to establish the main convergence result for the VRFTD algorithm under Markovian noise. Given the calculation above, we discuss the cases $\bar v = v^*$ and $\bar v \neq v^*$ separately for clarity.}

	\begin{theorem}\label{theorem_VRFTD_mark}
		Fix the total number of epochs $K$ and a positive integer $N$. Consider an integer $\tau$ satisfying
		$
		\rho^{\tau} \leq \min\{\frac{2(1-\rho)\varsigma}{3\mixcon}, \frac{2(1-\rho)^2}{5 \mixcon}\}
		$. 
		Suppose the parameters $n_0$ and $m_0$ satisfy
		\begin{align}\label{n0_m0}
			\rho^{n_0} \leq \tfrac{\min_{i \in [D]}\pi_i}{\skipcon},~ \text{and}\quad
			\rho^{m_0} \leq \min \left\{\tfrac{\min_{i \in [D]}\pi_i}{\skipcon}, \tfrac{\sqrt{\mu}\eta\tau\varsigma^2(1-\rho)}{\mixcon}\right\}.
		\end{align}Assume that for each epoch $k\in [K]$,  the parameters $\eta$, $\lambda$, $m$, $N_k$, $T$ satisfy 
		\begin{align}\label{stepsize_VRFTD_mark_1}
			\eta &\leq \tfrac{1}{4\beta(1+\gamma)}, ~~~~\lambda=1, ~~~~T \geq\tfrac{64}{\mu(1-\gamma)\eta},~~~~m-m_0 \geq \max\left\{1, \tfrac{792\eta(\tau+1) \varsigma^2}{1-\gamma}\right\}, \nn\\
			&\rho^{N_k-n_0} \leq \tfrac{\tau(1-\rho)}{5\mixcon (N_k-n_0)},~\text{and}~~~N_k-n_0\geq\left\{\tfrac{206(\tau+1)\varsigma^2}{\mu(1-\gamma)^2}, (\tfrac{3}{4})^{K-k} N\right\}.
		\end{align}
		Set the output of each epoch to be $\widehat v_k:=\frac{\sum_{t=2}^{T+1} v_t}{T}$. Then the following results hold. \\
		\indent (a) If $\bar v = v^*$, we have
		\begin{align}\label{result_VRFTD_mark_0}
			\bbe[\|\widehat v_K - v^*\|_\Pi^2] &\leq\tfrac{1}{2^K}\|v^0-\vbar\|_\Pi^2 + \tfrac{30}{N} \cdot \trace\left((I_d-M)^{-1}\iidS(I_d-M)^{-\top}\right).
		\end{align}
		\indent (b) If $\bar v \neq v^*$, then for any  $\delta>0$ we have
		\begin{align}\label{result_VRFTD_mark_1}
			\bbe[\|\widehat v_K - v^*\|_\Pi^2] &\leq\left(1+\delta+\tfrac{18(\tau+1)(1+1/\delta)}{\mu (1-\gamma)^2N^2} \right) \cdot \mathcal{A}(M, \gamma) \cdot \inf_{v \in S}\|v-v^*\|_\Pi^2 \nn\\
			&\quad + (1+\tfrac{1}{\delta})\left[\tfrac{1}{2^K}\|v^0-\vbar\|_\Pi^2 + \tfrac{30}{N} \cdot \trace\left((I_d-M)^{-1}\MarkovS(I_d-M)^{-\top}\right) + \tfrac{\mathcal{H}}{N^2}\right],
		\end{align}
		where 
		$
		\mathcal{H}:=(90\tau^2 + 18\tau+ 18)\cdot  \trace\left((I_d-M)^{-1}\iidS(I_d-M)^{-\top}\right).
		$
	\end{theorem}
	\noindent See Section~\ref{proof_theorem_VRFTD_mark} for detailed proofs of Theorem~\ref{theorem_VRFTD_mark}. Let us now discuss a few aspects of the theorem. 
	\paragraph{Estimation of mixing time} From the conditions above, e.g., Ineq.~\eqref{n0_m0}, the parameters $\tau, n_0, m_0$ scale linearly in the mixing time $\tmix$ and logarithmically in other problem parameters.
	As such, only some rough estimation of the mixing time is sufficient, which has been the topic of active research.
	Nontrivial confidence intervals for the reversible case can be found in \cite{hsu2019mixing}. There are also guarantees in the more challenging and prevalent case when the underlying Markov chain is non-reversible~\citep{wolfer2019estimating}.
	\black{
		\paragraph{Sample complexities} We first consider the case when $\bar v = v^*$. In view of Ineq.~\eqref{result_VRFTD_mark_0} in Theorem~\ref{theorem_VRFTD_mark},  the total number of samples required by the VRFTD method to find a solution $\widehat v \in \R^D$, such that $\bbe[\|\widehat v - v^*\|_\Pi^2]\leq \epsilon$ is $\tsum_{k=1}^K(mT+N_k)$, which is bounded on the order\footnote{\black{Note that we omit the logarithmic dependence on the problem parameters, e.g., $\mu, \beta, \gamma$. Same for the following complexity.}}	
		\begin{align}\label{complexity_VRFTD_markov_0}
			\underbrace{ \tfrac{\tmix\beta }{(1-\gamma)\mu} \log(\tfrac{\|v^0 - \vbar\|_\Pi^2}{\epsilon})}_{\text{deterministic error}} + \underbrace{\tfrac{\tmix\varsigma^2 }{(1-\gamma)^2\mu} \log(\tfrac{\|v^0 - \vbar\|_\Pi^2}{\epsilon})+\tfrac{\trace\big((I_d-M)^{-1}\iidS(I_d-M)^{-\top}\big)}{\epsilon}}_{\text{stochastic error}}, 
		\end{align}
		where the mixing time only enters along with terms that scale logarithmically in $1/\epsilon$. The phenomenon that the mixing time does not enter multiplicatively with the leading-order stochastic error term was also noticed by \cite{li2021q} for vanilla TD learning, but as mentioned before, this algorithm does not attain the correct instance-dependent stochastic error.}

	\black{When $\bar v \neq v^*$, the Markovian setting has a biased stochastic operator at optimal solution $\bar v$, and a larger approximation error $\|\vbar - v^*\|_\Pi^2$ caused by linear function approximation enlarges the bias of the stochastic operator and consequently enlarges the dependence on the approximation error in Eq.~\eqref{result_VRFTD_mark_1}. Therefore, a natural stopping criterion for the Markovian setting is to find a solution $\widehat v\in \bbr^D$ satisfying $\bbe[\|\widehat v - v^*\|_\Pi^2]\leq c\|\vbar - v^*\|_\Pi^2+\epsilon$ for some absolute constant $c > 0$.}
	From Ineq.~\eqref{result_VRFTD_mark_1}, the total number of required samples $\tsum_{k=1}^K (m T + N_k)$ is bounded on the order
	\black{
		\begin{align}\label{comlexity_VRFTD_markov_2}
			\underbrace{ \tfrac{\tmix\beta }{(1-\gamma)\mu} \log(\tfrac{\|v^0 - \vbar\|_\Pi^2}{\epsilon})}_{\text{deterministic error}} + \underbrace{ \tfrac{\tmix\varsigma^2 \log(\frac{\|v^0 - \vbar\|_\Pi^2}{\epsilon})}{(1-\gamma)^2\mu} +\tfrac{\sqrt{{\mathcal{H}}}}{\sqrt{\epsilon}}+\tfrac{ \trace\big((I_d-M)^{-1}\MarkovS(I_d-M)^{-\top}\big)}{\epsilon}}_{\text{stochastic error}}.
		\end{align}
		Note that, in this bound, the leading stochastic error matches the lower bound proved in \cite{mou2021optimal}, which can depend on the mixing time, but is generally smaller than the product of the i.i.d. stochastic error and the mixing time. 
		These results show a delicate difference between how the mixing time of the Markov chain enters the bound depending on whether the function approximation is exact or not. 
		It should be noted that, the stronger stopping criterion, i.e., finding $\widehat v$ to satisfy $\bbe[\|\widehat v - \vbar\|_\Pi^2]\leq \epsilon$, can also be applied in this setting. We can generate a similar sample complexity by enlarging the constants $\tau$ and $N_k$ by an additive factor of $\log(\|v^*-\bar v\|_\Pi^2)$. However, given that the approximation error is unavoidable and generally unknown, there is marginal benefit to using this stronger stopping criterion.}
	
	\black{\cite{mou2021optimal} established convergence guarantees for TD with averaging in the Markovian noise setting, showing a similar leading order stochastic error term but without accelerating the deterministic error. Besides improving on the deterministic error, Theorem~\ref{theorem_VRFTD_mark} also guarantees that the higher-order terms on stochastic error are smaller than those proved in~\cite{mou2021optimal}.
	}

	\section{Numerical experiments}\label{sec_numerical}
	In this section, we report numerical experiments for both VRTD and VRFTD, comparing them against temporal difference learning (TD), conditional temporal difference learning (CTD), and fast temporal difference learning (FTD)~\citep{GGT_20a,kotsalis2020simple}. To generate a comprehensive performance profile, we conduct experiments under both the i.i.d. and  Markovian noise models.
	
	\subsection{The i.i.d. setting: A simple two-state construction}
	We consider a family of two-state MRPs inspired by the construction of \cite{duan2021optimal}. For a discount factor $\gamma \in (\tfrac{1}{2},1)$, the transition kernel $P$ and reward vector $r$ are given by
	$
	P = \begin{bmatrix} 
		\tfrac{2\gamma-1}{\gamma} & \tfrac{1-\gamma}{\gamma}\\	\tfrac{1-\gamma}{\gamma} & \tfrac{2\gamma-1}{\gamma}
	\end{bmatrix}$  and  
	$		r = \begin{bmatrix} 
		1\\
		-1
	\end{bmatrix}$.
	Clearly the transition kernel is symmetric, thus the stationary distribution is $\pi = [0.5, 0.5]$. For simplicity, we choose the feature matrix $\Psi = \diag([\sqrt{2},\sqrt{2}])$,
	which forms an orthonormal basis under $\ell_\Pi$-norm. Assuming the i.i.d. model in which $s_i \sim \pi$ and $s_i'\sim P(\cdot|s_i)$,
	it can be shown via simple calculation that the stochastic error term is given by
	\begin{align}\label{numerical_lower_bound}
		\trace\big((I_d-M)^{-1}\iidS(I_d-M)^{-\top}\big) =\tfrac{40}{81}\cdot \tfrac{2\gamma-1}{(1-\gamma)^3}.
	\end{align}
	
	\begin{figure}[ht!]
		\centering
		\subfigure[TD algorithm]{\label{fig:TD}\includegraphics[width=5.5cm]{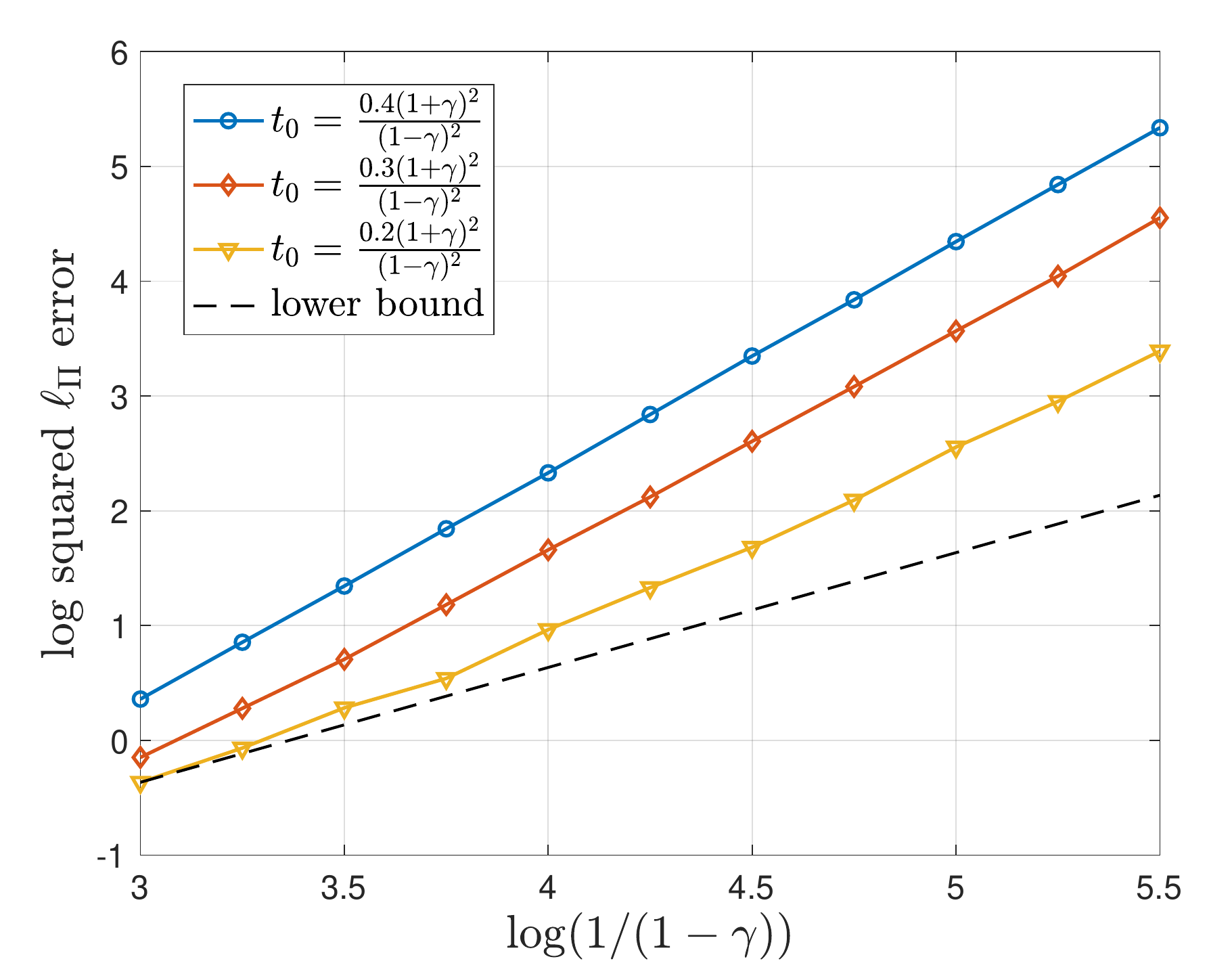}}
		\subfigure[FTD algorithm]{\label{fig:FTD}\includegraphics[width=5.5cm]{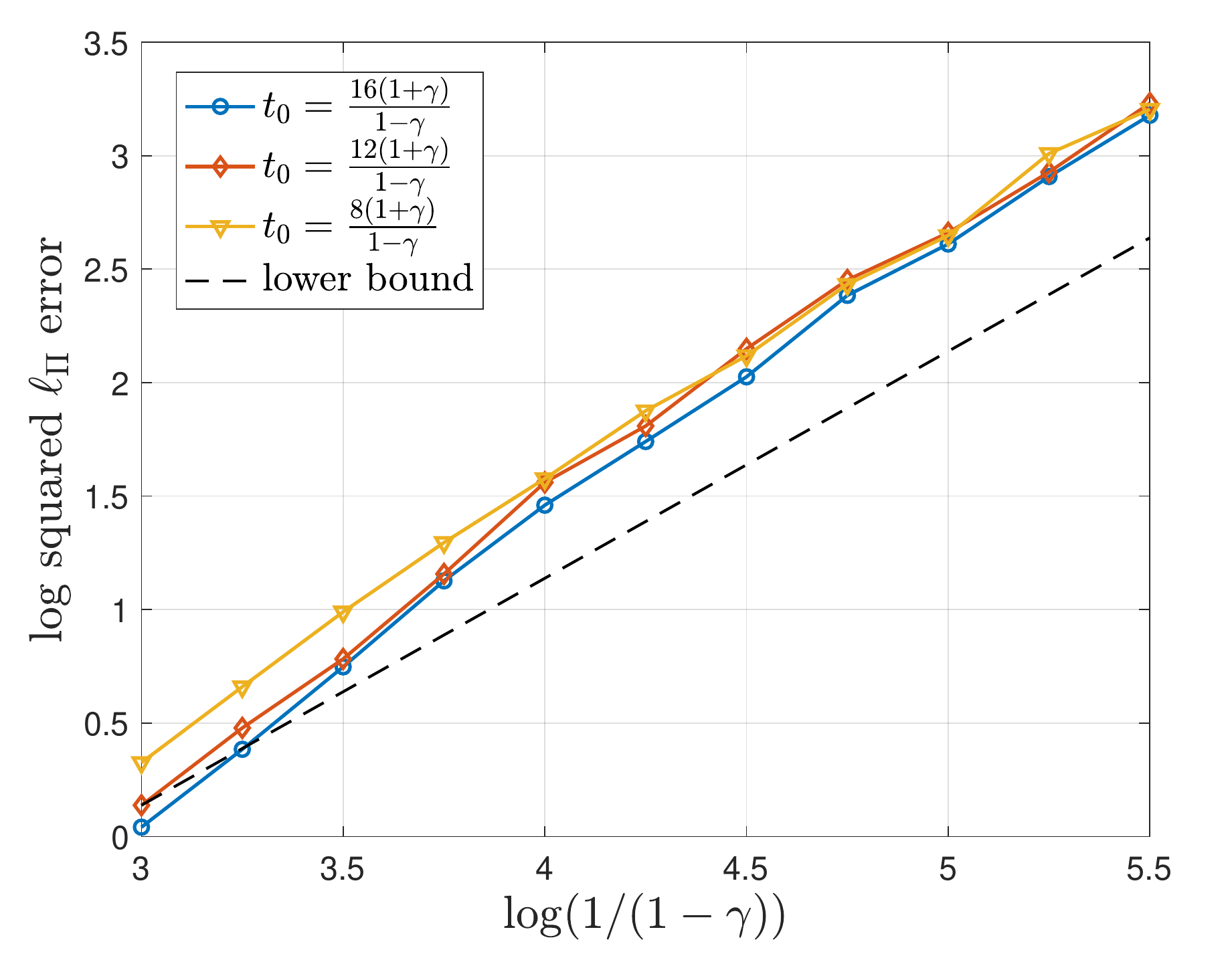}}
		\subfigure[VRTD algorithm]{\label{fig:VRTD}\includegraphics[width=5.5cm]{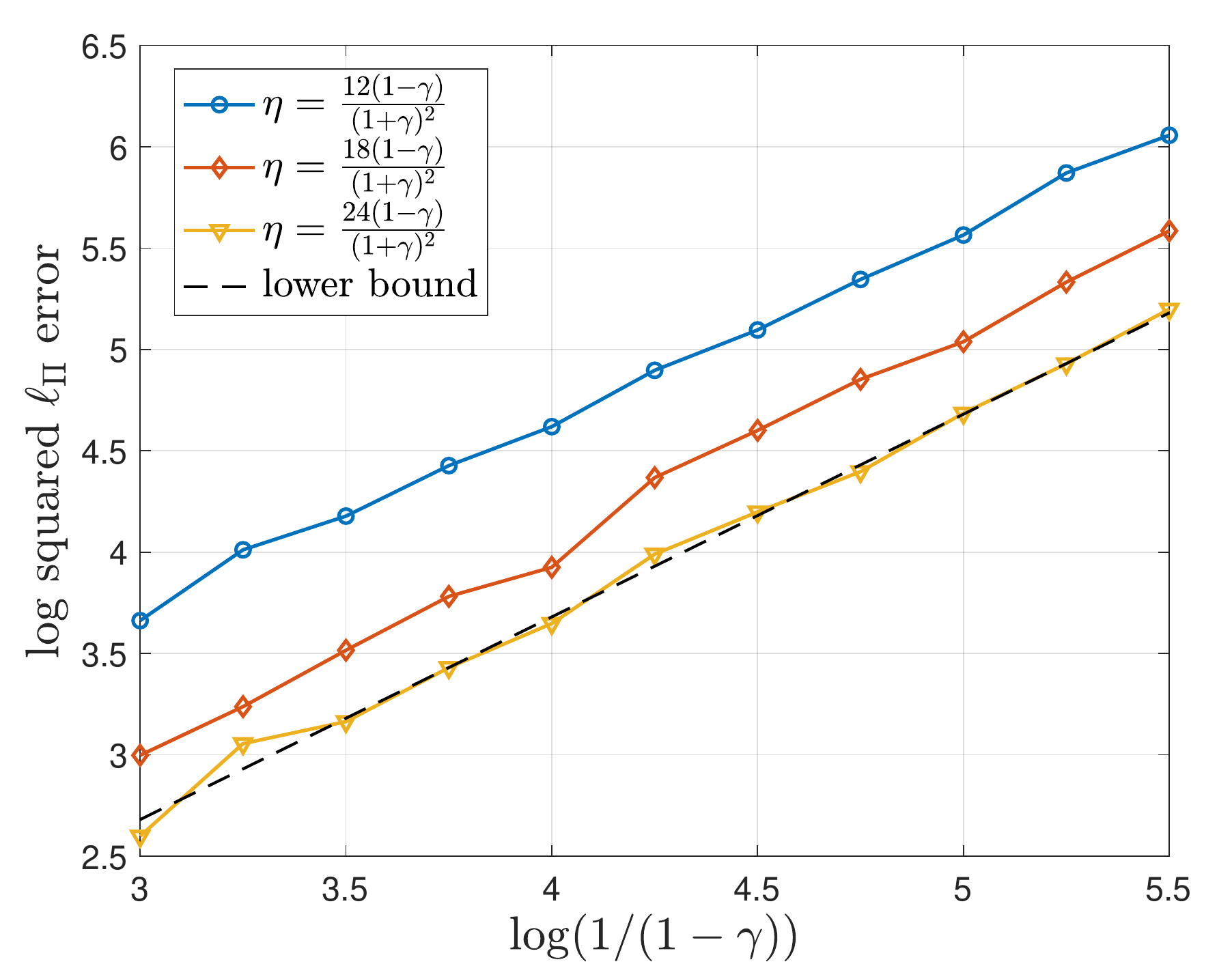}}
		\subfigure[VRFTD algorithm]{\label{fig:VRFTD}\includegraphics[width=6cm]{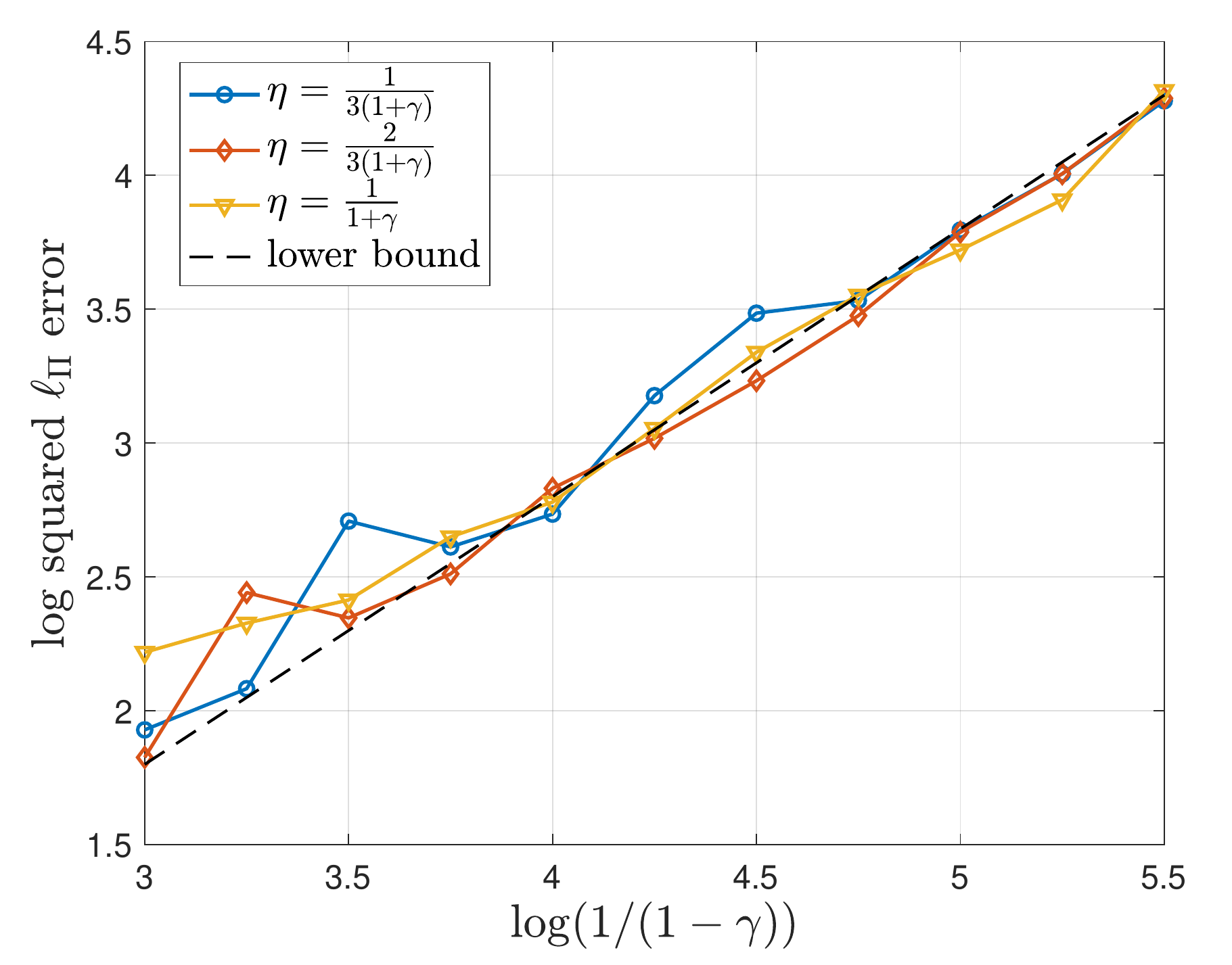}}
		\caption{ Log-log plots of the squared $\ell_\Pi$-norm error versus $1/(1-\gamma)$. Logarithms are to the natural base. The number of  samples used in each single experiment is $N = \lceil 5/(1-\gamma)^2 \rceil$. Each point in the plot is an average of $1000$ independent trials. The slope of the lower bound is 1.}
		\label{fig_iid_1}
	\end{figure}
	
	\paragraph{Instance-optimality} 
	We generate a range of MRPs with different values of discount factor $\gamma$ and run the four aforementioned algorithms on each MRP. In order to test the robustness of our results, we simulate various step-sizes for each algorithm. To be fair in our comparison, we also include a simulation of the best-tuned stepsize for each algorithm. We plot the prediction from the lower bound~\eqref{numerical_lower_bound} as well.
	
	
	
	From subplots (a) and (b) of Figure~\ref{fig_iid_1}, it is clear that the vanilla TD and FTD algorithms with diminishing stepsizes~\citep{kotsalis2020simple} do not achieve the lower bound calculated in equation~\eqref{numerical_lower_bound}. 
	On the other hand, sub-plots (c) and (d) show that the VRTD and VRFTD algorithms achieve the lower bound~\eqref{numerical_lower_bound}, and that these behaviors are robust to the choice of stepsize parameters. However, given their epoch-wise nature, the outputs of variance-reduced algorithms are more volatile than TD and FTD. Another interesting observation is that the accelerated algorithms---FTD and VRFTD---are less sensitive to stepsize parameters. Our next set of experiments explores this further.
	
	\black{
		\paragraph{Ablation analysis of VRFTD} Notice that VRFTD includes two new ingredients when compared with VRTD: mini-batching and operator extrapolation (OE). We now perform an ablation analysis to disentangle the contribution of both ingredients. We generate a range of MRPs with different values of $\gamma$ and run the experimental and control groups on each MRP.}
	\begin{figure}[H]
		\centering
		\subfigure[Ablation study for OE (I)]{\label{fig:OE}\includegraphics[width=5cm]{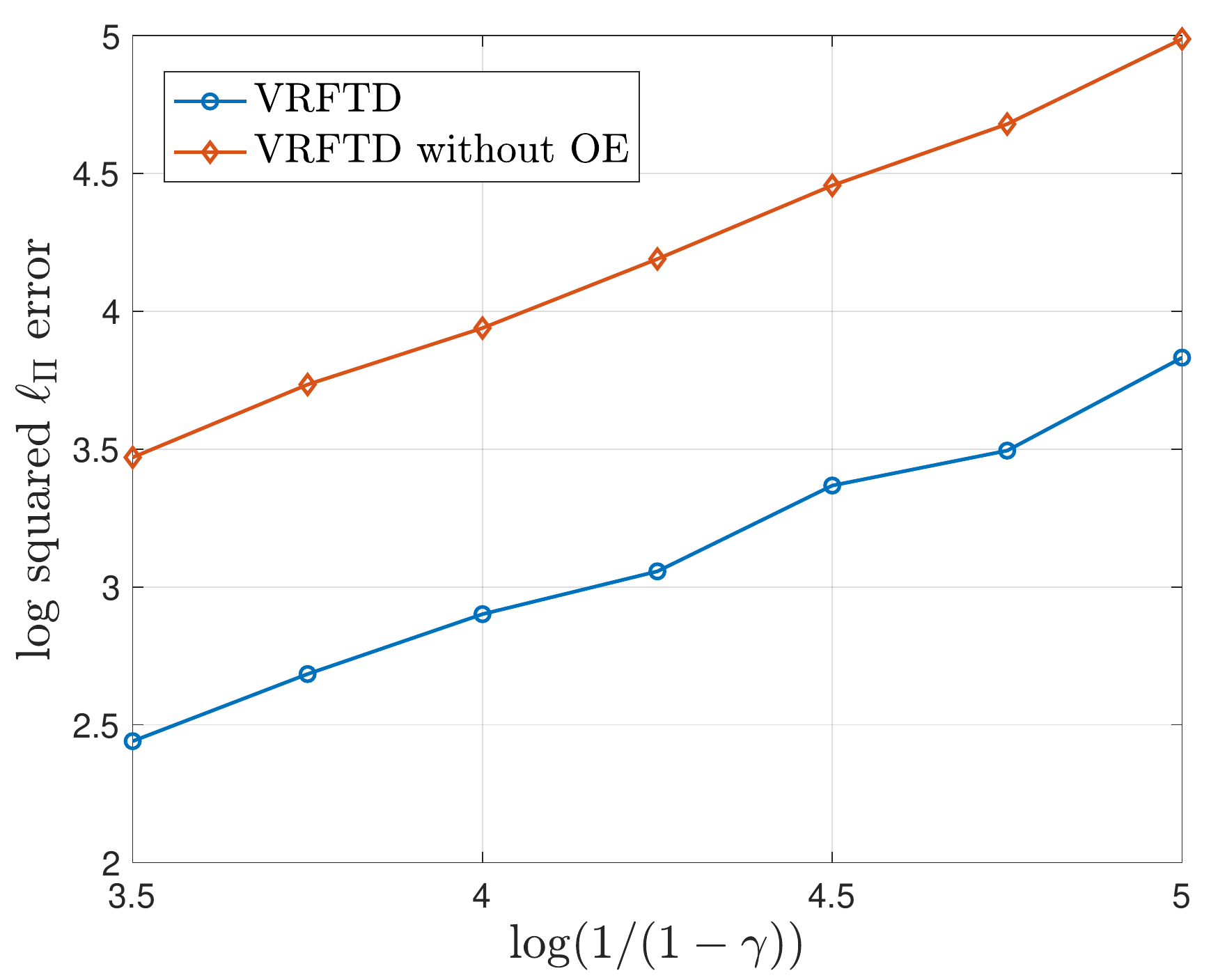}}
		\subfigure[Ablation study for OE (II)]{\label{fig:OE2}\includegraphics[width=5cm]{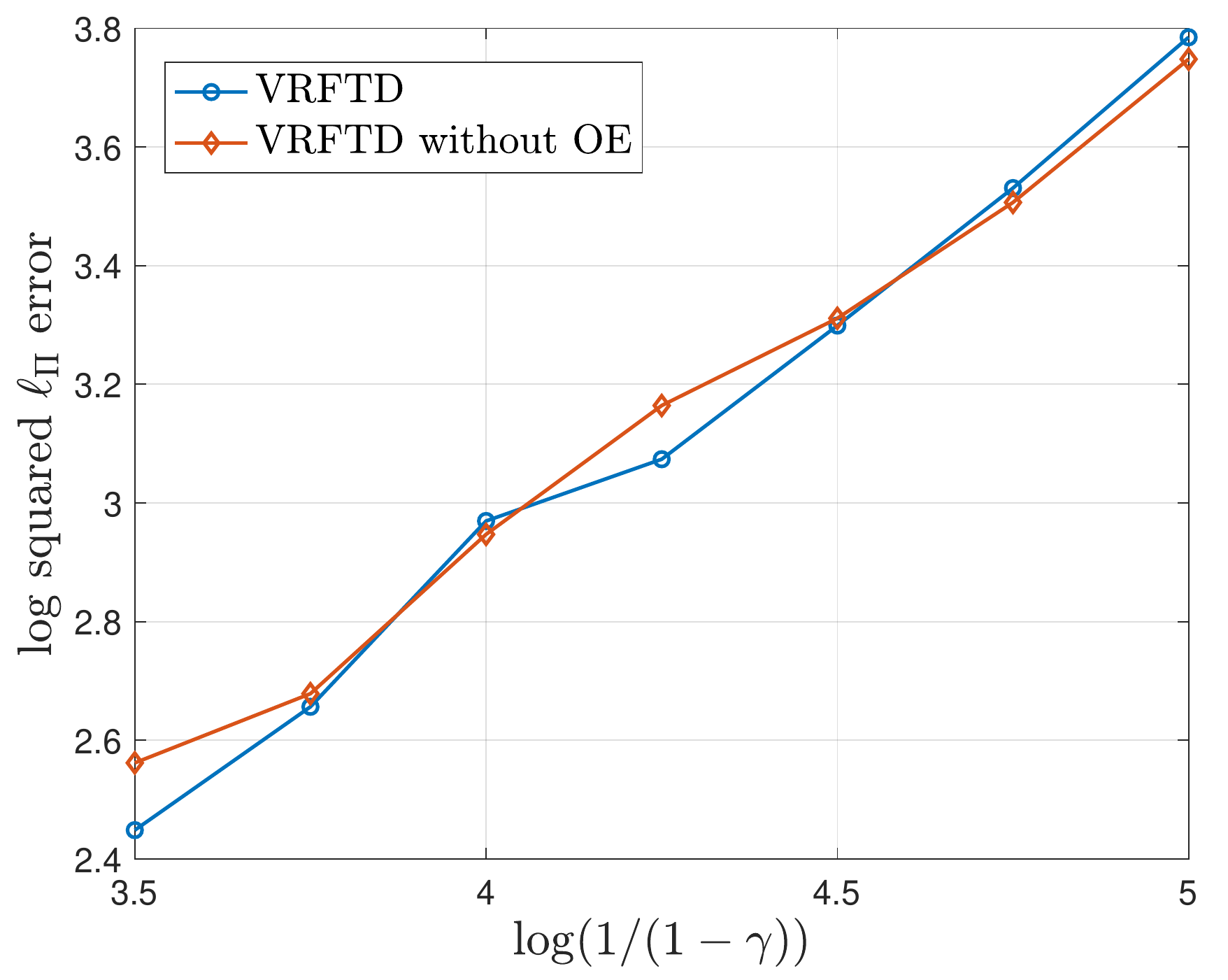}}
		\subfigure[Ablation study for mini-batching]{\label{fig:batch}\includegraphics[width=5cm]{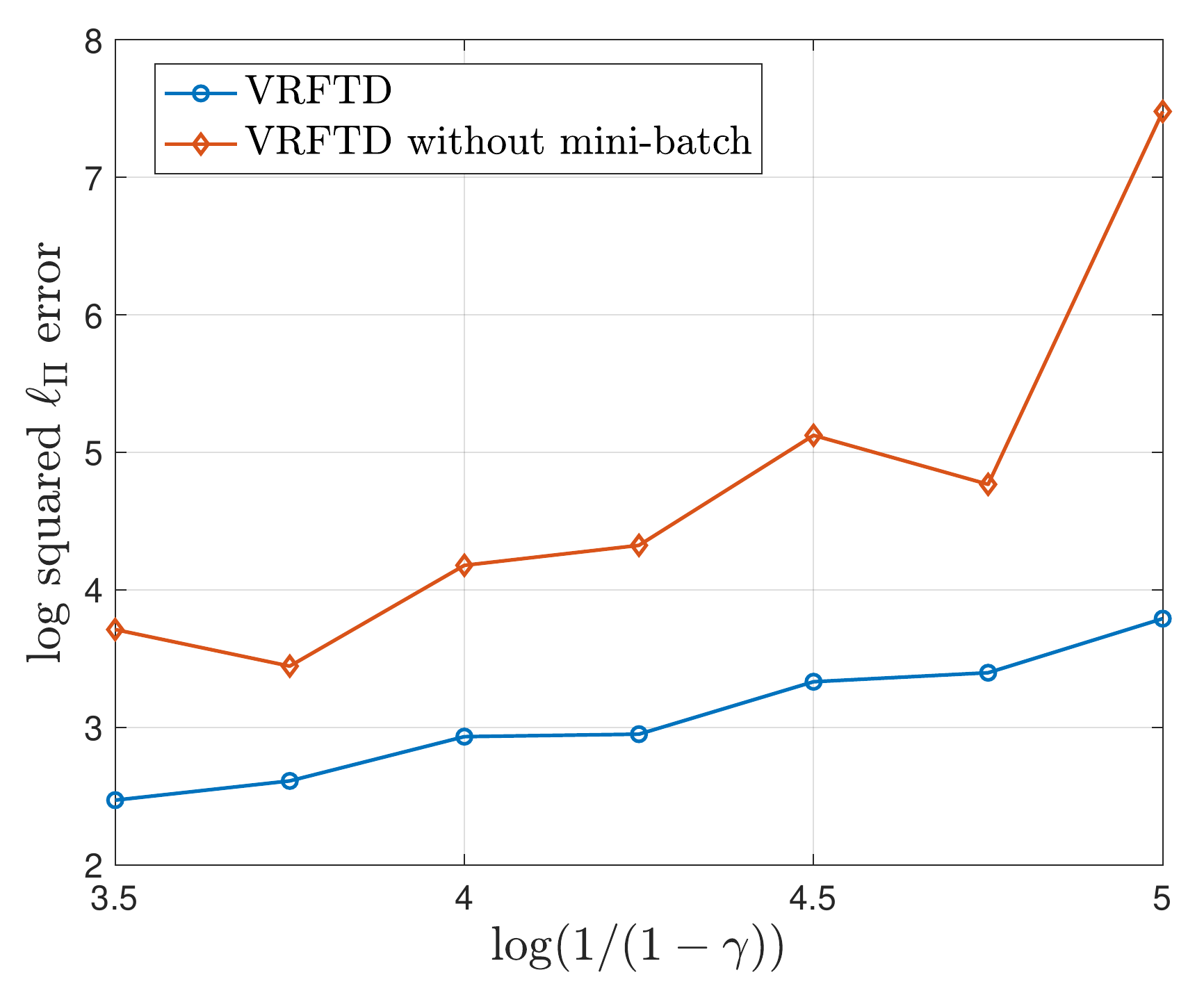}}
		\caption{ Log-log plots of the squared $\ell_\Pi$-norm error versus $1/(1-\gamma)$. Logarithms are to the natural base.
			The number of  samples used in each single experiment is $N = \lceil 5/(1-\gamma)^2 \rceil$. Each point in the plot is an average of $1000$ independent trials. }
		\label{fig_iid_0}
	\end{figure}
	\black{
		In the first experiments, we ran the experimental group with OE steps and the control group without OE steps.
		We first ran both groups with the stepsize policies suggested by the theoretical analysis (see subplot (a) of Figure~\ref{fig_iid_0}). The results indicate that the experimental group significantly outperforms the control group. However, this performance difference can largely be attributed to the conservative stepsize for the control group, as prescribed by the theory. To be more fair to both algorithms, we further fine-tuned the stepsize parameters of both algorithms and obtained subplot (b), where the two algorithms exhibit similar convergence rates. Taking stock, the first set of experiments shows that the analysis and stepsize policy of the VRFTD (with OE steps) serves as a better theoretical guideline for practical applications. }
	
	\black{
		To demonstrate the benefits of mini-batching in the inner loop, we ran a second experiment with two groups with and without mini-batching. Note that we keep the stepsizes and the total number of samples the same for both groups (which means that the control group without mini-batching has larger epoch lengths). From subplot (c), we can see that the performance difference is significant, showing that without mini-batching, the algorithm exhibits instability when run with aggressive stepsize policies.}

	\subsection{The Markovian setting: 2D Grid World}
	Our experiments under the Markovian noise model are conducted on the 2D Grid World environment. 
	This is a classical  problem in reinforcement learning with finite state and action spaces. An agent realizes a positive reward when reaching a predetermined goal and negative ones when going through ``traps". The dimension of the state space is set to be $D=400$, among which we assign a goal state (with reward $r=1$) and 30 traps (with reward $r=-0.2$). The transition kernel is fixed as follows: With probability $0.95$, the agent moves in a direction that points towards the goal and with probability $0.05$ in a random direction. Our goal is to compute the value function $v^*$---for each possible initial state of the agent. \black{We also incorporate linear function approximation in these experiments. Specifically, we generate random features with dimension $d=50$ to estimate the $D = 400$ dimensional value function in this problem.}
	
	
	\begin{figure}[H]
		\centering
		\begin{minipage}[t]{0.49\linewidth}
			\centering
			\includegraphics[width=6cm]{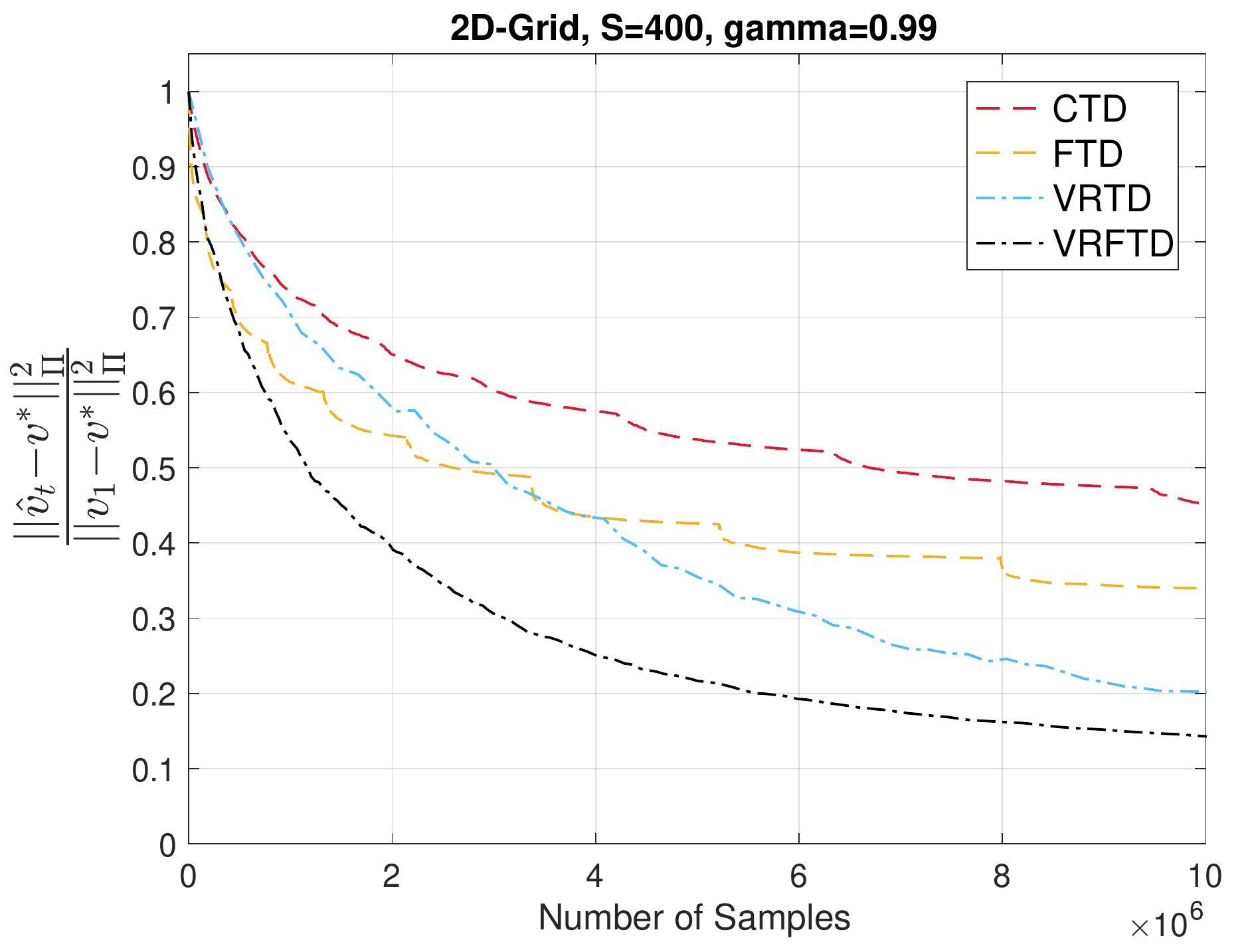}
		\end{minipage}
		\begin{minipage}[t]{0.49\linewidth}
			\centering
			\includegraphics[width=6cm]{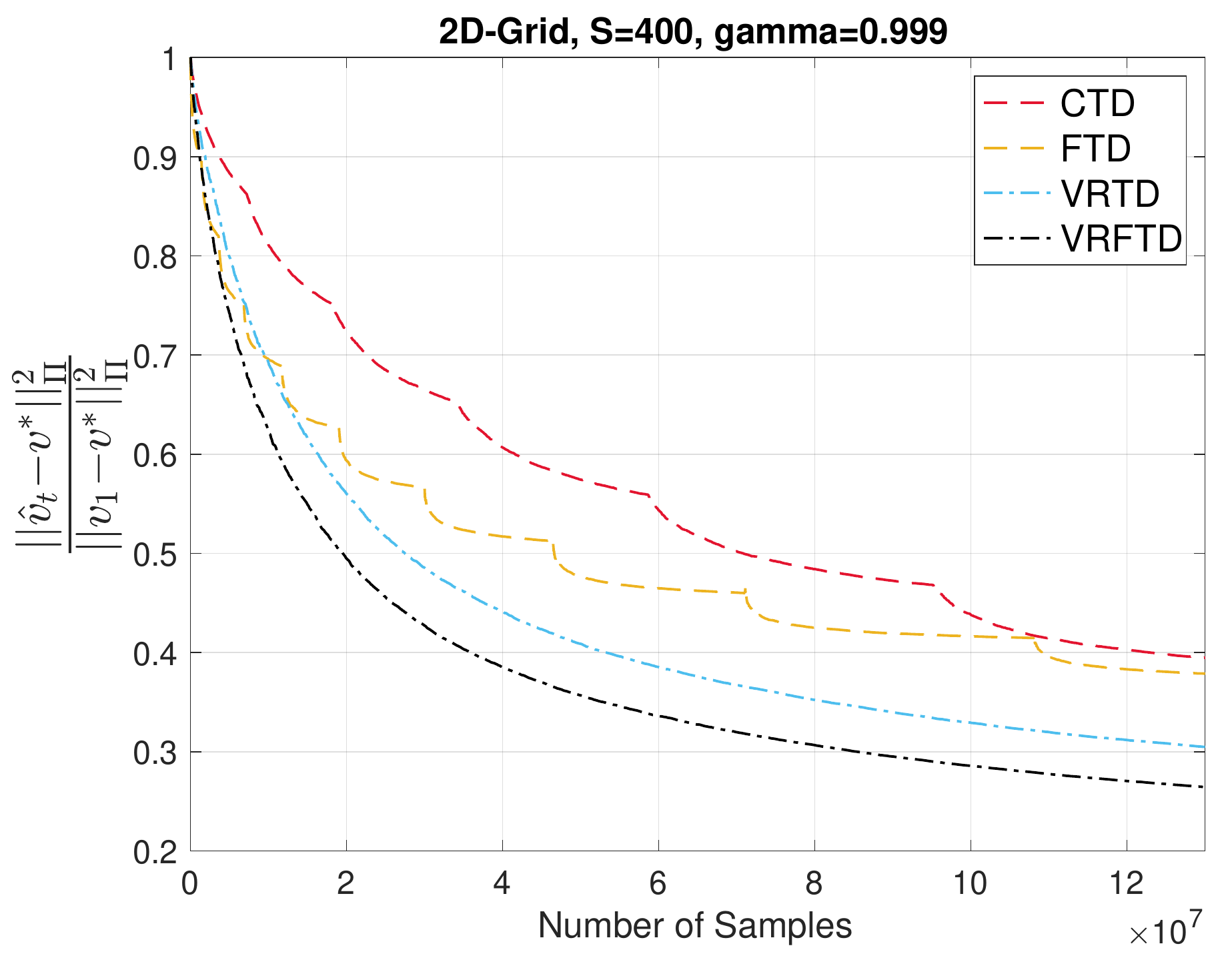}
		\end{minipage}
		\caption{Comparison of the algorithms for the 2D-Grid world example. From left to right $\gamma$ is set to $0.99$, and $0.999$  respectively. In the $y$-axis we report ratios in terms of the Euclidean norm $ \| \cdot \|_\Pi$. }
		\label{fig_2D_Grid_1}
	\end{figure}
	
	We test the performance of four algorithms, with the discount factor $\gamma$ set to $0.99$ and $0.999$. Figure~\ref{fig_2D_Grid_1} plots the normalized error in $\ell_\Pi$-norm against the length of the trajectory. In both experiments, the VRFTD algorithm exhibits the fastest convergence to the true value function, thereby corroborating our theoretical results. \black{Note that in both experiments, the estimation errors do not converge to zero, because there is a nontrivial error incurred by linear function approximation.} Another salient takeaway is the following: Closer to the optimal solution $v^*$, the variance-reduced algorithms (VRTD/VRFTD) achieve faster convergence rate compared to their counterparts that do not employ variance reduction.
	
	\section{Proofs}\label{sec_proof}
	This section contains full proofs of our main results. 
	\subsection{Preliminary supporting lemmas}
	We begin by establishing a few useful lemmas which are frequently used in the proof of the main theorems. All the lemmas are simple to state and prove, and provide the basis for several manipulations in our subsequent proofs.
	\begin{lemma}\label{lemma_pi_theta}
		For any $\theta\in \bbr^d$ and $v:=\Psi^\top \theta$, we have
		\begin{align}\label{eq_pi_theta}
			\|\theta\|_2 = \|\Phi^\top \theta\|_\Pi \quad \text{and} \quad \|v\|_\Pi= \|B^{\frac{1}{2}}\theta\|_2.
		\end{align}
	\end{lemma}
	\begin{proof}
		The desired result follows immediately from $\Phi = B^{-\frac{1}{2}}\Psi$ and $\Phi \Pi \Phi^\top = I_d$.
	\end{proof}
	
	\begin{lemma} \label{lem:transition-pi}
		For each vector $u\in \R^D$, we have
		\begin{align}\label{Pi_norm}
			\|Pu\|_\Pi \leq \|u\|_\Pi.
		\end{align}
	\end{lemma}
	\begin{proof}
		Based on the definition of $\|\cdot\|_\Pi$, we have 
		\begin{align*}
			\|Pu\|_\Pi^2 = \tsum_{i=1}^D \pi_i (Pu)_{(i)}^2 \overset{(i)}\leq \tsum_{i=1}^D \pi_i(\tsum_{j=1}^D P_{i,j}u_{(j)}^2)=\tsum_{j=1}^D(\tsum_{i=1}^D \pi_i P_{i,j})u_{(j)}^2\overset{(ii)}
			= \|u\|_\Pi^2,
		\end{align*}
		where step (i) follows from Jensen's inequality and step (ii) follows from the definition of the stationary distribution.
	\end{proof}
	
	\begin{lemma}\label{lemma_strong_monotone}
		For any $\theta,\theta'\in\bbr^d$, we have
		\begin{align}\label{strong_monotone}
			\langle g(\theta) - g(\theta'),\theta-\theta'\rangle \geq (1-\gamma)\|v-v'\|_\Pi^2.
		\end{align}
	\end{lemma}
	\begin{proof}
		By the definition of $g(\cdot)$, we have 
		\begin{align*}
			\langle g(\theta) - g(\theta'),\theta-\theta'\rangle &= \langle \Psi \Pi (I-\gamma P)\Psi^\top (\theta-\theta'),\theta-\theta'\rangle\\
			&= \langle  (I-\gamma P)\Psi^\top (\theta-\theta'),\Psi^\top(\theta-\theta')\rangle_\Pi\\
			&\overset{(i)}\geq \|v-v'\|_\Pi^2 - \gamma \|P(v-v')\|_\Pi\|v-v'\|_\Pi\\
			&\overset{(ii)}\geq (1-\gamma)\|v-v'\|_\Pi^2 ,
		\end{align*}
		where step (i) follows from the Cauchy--Schwarz inequality, and step (ii) follows from Lemma \ref{lem:transition-pi}. 
	\end{proof}
	
	\begin{lemma}\label{lemma_lipschitz}
		For any $\theta,\theta'\in\bbr^d$, we have
		\begin{align}\label{lipschitz}
			\| g(\theta) - g(\theta')\|_2 \leq(1+\gamma)\sqrt{\beta} \|v-v'\|_\Pi \leq (1+\gamma)\beta \|\theta-\theta'\|_2.
		\end{align}
	\end{lemma}
	\begin{proof}
		Recall the definition of $\beta$ in Eq.~\eqref{beta_mu}. By the definition of $g(\cdot)$, we have
		\begin{align}\label{lipschitz_proof_1}
			\| g(\theta) - g(\theta')\|_2 &=\|( \Psi \Pi \Psi^\top -\gamma\Psi\Pi P \Psi^\top )(\theta-\theta')\|_2\nn\\
			&\overset{(i)}\leq \sqrt{\beta}\|( \Phi \Pi \Psi^\top -\gamma\Phi\Pi P \Psi^\top )(\theta-\theta')\|_2\nn\\
			& \overset{(ii)}= \sqrt{\beta}\|(\Phi^\top \Phi \Pi  -\gamma\Phi^\top \Phi\Pi P  )\Psi^\top(\theta-\theta')\|_\Pi
		\end{align}
		where step (i) follows from $\Psi = B^{\frac{1}{2}}\Phi$ and $\|B^{\frac{1}{2}}\|_2\leq \sqrt{\beta}$, and step (ii) follows from Lemma \ref{lemma_pi_theta}. Invoking the relationship $\Phi \Pi \Phi^\top = I_d$ gives us
		\begin{align}\label{lipschitz_proof_2}
			\|(\Phi^\top \Phi \Pi  -\gamma\Phi^\top \Phi\Pi P  )\Psi^\top(\theta-\theta')\|_\Pi=\|(I-\gamma\Pi_\bbs P )\Psi^\top(\theta-\theta')\|_\Pi\overset{(i)}\leq(1+\gamma) \|v-v'\|_\Pi,
		\end{align}
		where step (i) follows from $\|\Pi_\bbs P \Psi^\top(\theta-\theta')\|_\Pi\leq \|P \Psi^\top(\theta-\theta')\|_\Pi \leq \|\Psi^\top(\theta-\theta')\|_\Pi.$ 
		
		Substituting Ineq.~\eqref{lipschitz_proof_2} into~\eqref{lipschitz_proof_1} yields the first inequality in the chain~\eqref{lipschitz}.
		The second inequality in the chain~\eqref{lipschitz} follows from the fact that 
		\begin{align*}
			\|\Psi(\theta-\theta')\|_\Pi  = \|B^{\frac{1}{2}}(\theta-\theta')\|_2 \leq \sqrt{\beta} \|\theta-\theta'\|_2,
		\end{align*}
		which completes the proof.
	\end{proof}

	\begin{lemma}\label{lemma_inverse}
		For each vector $u\in \R^D$, we have
		\begin{align}\label{variance_1}
			\|(I-\gamma P)^{-1} u \|_\Pi \leq \tfrac{1}{1-\gamma} \|u\|_\Pi.
		\end{align}
	\end{lemma}
	\begin{proof}
		Applying the triangle inequality, we have
		\begin{align*}
			\|u\|_\Pi &= \|(I-\gamma P)(I-\gamma P)^{-1}u\|_\Pi \\
			&\geq \|(I-\gamma P)^{-1}u\|_\Pi - \gamma \|P(I-\gamma P)^{-1}u\|_\Pi\\
			&\overset{(i)}\geq (1-\gamma) \|(I-\gamma P)^{-1}u\|_\Pi .
		\end{align*}
		Here, step (i) follows from Lemma~\ref{lem:transition-pi}. For an alternative proof, note that $(1 - \gamma) (I - \gamma P)^{-1}$ is a Markov transition matrix with stationary distribution $\pi$, and the desired result follows from Lemma~\ref{lem:transition-pi}.
	\end{proof}
	
	\begin{lemma}\label{lemma_trans_0}
		For any vector $u\in \bbr^D$, we have
		\begin{align}\label{lemma_trans}
			(I-\Pi_\bbs \gamma P)^{-1} \Pi_\bbs u = \Phi^\top (I_d-M)^{-1}\Phi \Pi u.
		\end{align}
	\end{lemma}
	\begin{proof}
		Note that $z:=(I-\Pi_\bbs \gamma P)^{-1} \Pi_\bbs u \in \bbs$ since $z =\Pi_\bbs \gamma P z + \Pi_\bbs u $. Also, we have $z = \Pi_\bbs z = \Phi^\top \Phi \Pi z$. Consequently,
		$$
		\Phi \Pi z = \gamma\Phi \Pi P z + \Phi \Pi u = \Phi\gamma \Pi P \Phi^\top  \Phi\Pi z + \Phi \Pi u = M \Phi \Pi z + \Phi \Pi u.
		$$
		Rearranging yields $\Phi \Pi z = (I_d - M)^{-1} \Phi \Pi u$ and putting together the pieces, we have $z = \Phi^\top \Phi \Pi z = \Phi^\top(I_d-M)^{-1}\Phi\Pi u$.
	\end{proof}
	
	\begin{lemma}
		For each vector $\theta\in \R^d$, we have
		\begin{align}\label{variance_2}
			\|(I_d-M)^{-1} \theta \|_2 \leq \tfrac{1}{1-\gamma} \|\theta\|_2.
		\end{align}
	\end{lemma}
	\begin{proof}
		Invoking Lemma~\ref{lemma_trans_0}, we have
		\begin{align*}
			\|(I_d-M)^{-1}\theta\|_2 & = \|\Phi^\top(I_d-M)^{-1}\theta\|_\Pi\\
			& = \|\Phi^\top(I_d-M)^{-1}\Phi \Pi \Phi^\top \theta\|_\Pi\\
			&\overset{(i)}=\|(I-\Pi_\bbs \gamma P)^{-1} \Pi_\bbs \Phi^\top \theta\|_\Pi\\
			&=\|(I-\Pi_\bbs \gamma P)^{-1} \Phi^\top \theta\|_\Pi,
		\end{align*}
		where step (i) follows from Lemma~\ref{lemma_trans_0}. At the same time, we have
		\begin{align*}
			\|\theta\|_2 = \|\Phi^\top \theta\|_\Pi &= \|(I-\Pi_\bbs \gamma P)(I-\Pi_\bbs \gamma P)^{-1}\Phi^\top\theta\|_\Pi\\
			&\geq \|(I-\Pi_\bbs \gamma P)^{-1}\Phi^\top\theta\|_\Pi - \|\Pi_\bbs \gamma P(I-\Pi_\bbs \gamma P)^{-1}\Phi^\top\theta\|_\Pi\\
			&\overset{(i)}\geq \|(I-\Pi_\bbs \gamma P)^{-1}\Phi^\top\theta\|_\Pi - \gamma\|P(I-\Pi_\bbs \gamma P)^{-1}\Phi^\top\theta\|_\Pi\\
			&\overset{(ii)}\geq (1-\gamma)\|(I-\Pi_\bbs \gamma P)^{-1}\Phi^\top\theta\|_\Pi,
		\end{align*}
		where step (i) follows from the non-expansiveness of the projection, and
		step (ii) follows from Lemma~\ref{lem:transition-pi}.  
		Putting together the two pieces yields the result.
	\end{proof}
	
	\subsection{Proof of Theorem~\ref{theorem_lower_bound}}\label{proof_theorem_lower_bound}
	First, it is clear that the methods of this type are invariant to a simultaneous shift of variables. The sequence of iterates for solving $G(v) = 0$ starting from $v_0$ is just a shift of the sequence generated for solving $G(v+v_0) = 0$ starting from the origin. Therefore, without loss of generality, we assume $v_0 = 0$. 
	
	Now let us construct a specific instance $(P, r)$ to show the lower bound. Consider the $D \times D$ matrix 
	\begin{align}\label{P}
		P := \begin{bmatrix} 
			\tfrac{1}{2\gamma} & 0 &0 &...  &0&1-\tfrac{1}{2\gamma}\\ 
			1-\tfrac{1}{2\gamma} &\tfrac{1}{2\gamma} &0&...  &0&0\\
			0 &1-\tfrac{1}{2\gamma} &\tfrac{1}{2\gamma} &...  &0&0\\
			\vdots&\vdots &\vdots &\ddots&\vdots&\vdots\\
			0 &0&0 &...&1-\tfrac{1}{2\gamma}&\tfrac{1}{2\gamma}
		\end{bmatrix},
	\end{align}
	and note that
	\begin{align}\label{I_P}
		I-\gamma P= \begin{bmatrix} 
			\tfrac{1}{2} & 0 &0 &...  &0&\tfrac{1}{2}-\gamma\\ 
			\tfrac{1}{2}-\gamma&\tfrac{1}{2} &0&...  &0&0\\
			0 &\tfrac{1}{2}-\gamma &\tfrac{1}{2} &...  &0&0\\
			\vdots&\vdots &\vdots &\ddots&\vdots&\vdots\\
			0 &0&0 &...&\tfrac{1}{2}-\gamma&\tfrac{1}{2}
		\end{bmatrix}.
	\end{align}
	Also define the $D$-dimensional vector
	\begin{align}\label{b}
		r:=[\gamma - \tfrac{1}{2} + (\tfrac{1}{2}-\gamma)(2\gamma-1)^D, ~0,~0,...,~0,~0]^\top
	\end{align}
	The matrix $I-\gamma P$ is square  and full rank, and it is straightforward to verify that the unique solution of the linear equation $(I-\gamma P) v^* = r$ is 
	\begin{align}\label{v_star}
		(v^*)_{(i)} = (2\gamma - 1)^i \quad \text{for all } i \in [D].
	\end{align}
	It is also easy to see that the Markov chain induced by the transition kernel $P$ is irreducible and aperiodic. From the cyclical behavior of the Markov chain, we have that the stationary distribution is uniform, i.e.,
	\begin{align*}
		\Pi = [\tfrac{1}{D}, ~\tfrac{1}{D},~\tfrac{1}{D},...,~\tfrac{1}{D}]^\top.
	\end{align*}
	Thus, we obtain
	\begin{align}\label{v_0_dis}
		\|v_0-v^*\|_\Pi^2 =\tfrac{1}{D}\tsum_{i=1}^D (2\gamma-1)^{2i} = \tfrac{(2\gamma-1)^2[1-(2\gamma-1)^{2D}]}{D[1-(2\gamma-1)^2]}.
	\end{align}
	Let $\R^{k,D}:= \{v\in \R^D \;|\; v_{(i)}=0 \text{ for all } k+1\leq i\leq D\}$ denote the set of all $D$-dimensional vectors lying in the span of the first $k$ standard basis vectors. 
	Since all entries of $I-\gamma P$ below its subdiagonal are equal to $0$ and all entries of $r$ except for its first are equal to $0$, we conclude that $v_k \in \R^{k, D}$. Therefore,
	\begin{align*}
		\|v_k-v^*\|_\Pi^2 \geq \tfrac{1}{D} \tsum_{i=k+1}^D (2\gamma-1)^{2i} = \tfrac{(2\gamma-1)^{2k+2}[1-(2\gamma-1)^{2D-2k}]}{D[1-(2\gamma-1)^2]}.
	\end{align*}
	If $D\gg k$ is such that $\tfrac{1-(2\gamma-1)^{2D-2k}}{1-(2\gamma-1)^{2D}}\geq \tfrac{1}{2}$, then we conclude that
	\begin{align*}
		\|v_k-v^*\|_\Pi^2 \geq \tfrac{1}{2} (2\gamma-1)^{2k}\|v_0-v^*\|_\Pi^2,
	\end{align*}
	as desired.
	
	\subsection{Proof of Proposition~\ref{theorem_stochastic_lower_bound}}\label{proof_stochastic_lower_bound}
	For the reader's convenience, we begin by stating a version of the H\'{a}jek-Le Cam local asymptotic minimax theorem. 
	
	\begin{theorem}\label{theorem_hajek}
		Let $\{ \mathbb{P}_{\param} \}_{\param \in \Theta}$ denote a family of parametric models, quadratically mean differentiable with Fisher information matrix $J_{\param'}$ . Fix some parameter $\param\in \Theta$, and consider a function $h: \Theta \rightarrow \bbr^d$ that is differentiable at $\param$. Then for any quasi-convex loss $L:\bbr^d \rightarrow \bbr$, we have
		\begin{align}\label{hajek}
			\lim_{c \rightarrow \infty} \lim_{N\rightarrow \infty} \inf_{\widehat h_N} \sup_{\param':\|\param'-\param\|_2\leq c/\sqrt{N}} \bbe_{\param'}\left[L\big(\sqrt{N}(\widehat h_N - h(\param')\big)\right] = \bbe[L(Z)],
		\end{align}
		where the infimum is taken over all estimators $\widehat h_N$ that are measurable functions of $N$ i.i.d. data points drawn from $\mathbb{P}_{\param'}$ , and the expectation is taken over a multivariate Gaussian $$ Z\sim \mathcal{N}\left(0, \nabla h(\param)^\top  J^\dagger_{\param} \; \nabla h(\param) \right).$$
	\end{theorem}
	
	In our model, we have $\param = (\omega, P, R)$ and $h(\param)=(\Psi \Omega \Psi^\top- \Psi \Omega \gamma P \Psi^\top)^{-1} \Psi \Omega r$. We set the loss function $L$ to be $\|\cdot\|_{\widetilde B}^2$. Invoking Theorem~\ref{theorem_hajek} yields
	\begin{align}\label{hajek_2}
		\mathfrak{M}(\param) = \bbe\left[\|Z\|_{\widetilde B}^2\right] \quad \text{where} \quad Z\sim \mathcal{N}\left(0, \nabla h(\param)^\top  J^\dagger_{\param} \; \nabla h(\param) \right).
	\end{align}
	The covariance is explicitly computed in the following lemma.
	\begin{lemma}\label{lemma_hajek}
		We have
		\begin{align}\label{hajek_3}
			\nabla h(\param)^\top  J^\dagger_{\param} \; \nabla h(\param)  = (\Psi \Omega\Psi^\top-\gamma\Psi \Omega P \Psi^\top)^{-1} \Sigma ~(\Psi \Omega \Psi^\top-\gamma\Psi \Omega P \Psi^\top)^{-T} ,
		\end{align}
		where $\Sigma = \cov\big[
		\big(\langle\psi(s) - \gamma \psi(s'),\bar\theta\rangle - R(s, s')\big) \psi(s)\big]$ for $s\sim \omega,~s'\sim \mathbb{P}(\cdot|s)$.
	\end{lemma}
	\noindent The proof of this lemma is straightforward but involves some lengthy calculations, which we defer to Appendix~\ref{proof_lemma_hajek}. 
	
	Recall the definition of $\widetilde M$ and $\widetilde \Sigma$ in the statement of Proposition~\ref{theorem_stochastic_lower_bound}. By substituting equality~\eqref{hajek_3} into~\eqref{hajek_2} and invoking the relation  $\widetilde B^{-\frac{1}{2}} \Psi \Omega \Psi^\top \widetilde B^{-\frac{1}{2}} = I_d$, we obtain
	$$\widetilde B^{\frac{1}{2}} Z \sim \mathcal{N}\left(0, (I_d-\widetilde M)^{-1}  \widetilde \Sigma (I_d- \widetilde M)^{-T}\big)  \right),$$
	which completes the proof of the proposition. 

	\subsection{Proof of Theorem~\ref{theorem_VRTD_2}}\label{proof_theorem_VRTD}
	Let $\underline \theta$ satisfy $g(\underline \theta) - g(\widetilde \theta) + \widehat g(\widetilde \theta) = 0$ and $\underline v = \Psi^\top\underline \theta $. Recalling the definition of $\bar v$ from Eq.~\eqref{proj_fixed_point}, the following lemma provides a bound for $\|\underline v - \vbar\|_\Pi^2$. This bound is also valid for the VRFTD algorithm in the i.i.d. setting.
	
	\begin{lemma}\label{bound_tilde_v}
		Consider a single epoch with index $k \in [K]$. We have 
		\begin{align}\label{VRTD_step_0}
			\bbe[\|\underline v - \vbar \|^2_\Pi]\leq \tfrac{2}{N_k} \trace\left((I_d-M)^{-1}\iidS(I_d-M)^{-\top}\right) + \tfrac{2\varsigma^2 }{N_k(1-\gamma)^2\mu}\bbe\|\widetilde v - \vbar\|_\Pi^2.
		\end{align}
	\end{lemma}
	\noindent See Appendix~\ref{proof_bound_tilde_v} for the proof of this lemma.
	
	Given Lemma~\ref{bound_tilde_v}, we can derive the following proposition which characterizes the progress of the VRTD algorithm in a single epoch.
	\begin{proposition}\label{prop_VRTD}
		Consider a single epoch with index $k \in [K]$. Suppose that the parameters $\eta$, $N_k$ and $T$ satisfy 
		\beq\label{stepsize_1}
		\eta\leq \min\{ \tfrac{(1-\gamma)}{2\beta(1+\gamma)^2},\tfrac{1-\gamma}{32\varsigma^2}\},~~T\geq \tfrac{32}{\mu(1-\gamma)\eta},~~\text{and}~~N_k\geq\tfrac{38\varsigma^2}{\mu(1-\gamma)^2}.
		\eeq
		Set the output of this epoch to be $\widehat v_k:=\frac{\sum_{t=1}^T \eta (1-\gamma) v_t+(1/\beta )v_{T+1}}{T \eta (1-\gamma) +(1/\beta )}$. Then we have
		\begin{align}\label{VRTD_prop}
			\bbe[\|\widehat v_k - \vbar\|_\Pi^2] &\leq \tfrac{1}{2}\bbe[\|\widehat v_{k-1}- \vbar\|_\Pi^2] + \tfrac{5}{N_k}\trace\{(I_d-M)^{-1}\iidS(I_d-M)^{-\top}\}.
		\end{align}
	\end{proposition}
	
	Taking Proposition~\ref{prop_VRTD} as given for the moment, let us complete the proof of the theorem. The main idea is to bound the approximation error term $\|\vbar - v^*\|_\Pi^2$ separately from the term $\|\widehat v_K - \vbar\|_\Pi^2$. 
	To bound $\|\vbar - v^*\|_\Pi^2$, we use the instance-dependent upper bound in Ineq.~\eqref{approximation_error}.  
	We bound the term $\|\widehat v_K - \vbar\|_\Pi^2$ by using Ineq.~\eqref{VRTD_prop} as follows
	\begin{align}\label{step_5}
		\bbe\|\widehat v_K - \vbar\|_\Pi^2 &\leq \tfrac{1}{2^K}\|v^0 - \vbar\|^2_\Pi  + \tsum_{k=1}^K \tfrac{5}{2^{K-k} N_k} \cdot \trace\{(I_d-M)^{-1}\iidS(I_d-M)^{-\top}\}\nn\\
		&\overset{(i)}= \tfrac{1}{2^K}\|v^0 - \vbar\|^2_\Pi  + \tsum_{k=1}^K(\tfrac{2}{3})^{K-k}\tfrac{5}{N} \cdot \trace\{(I_d-M)^{-1}\iidS(I_d-M)^{-\top}\}\nn\\
		&\leq \tfrac{1}{2^K}\|v^0 - \vbar\|^2_\Pi  + \tfrac{15}{N}\cdot \trace\{(I_d-M)^{-1}\iidS(I_d-M)^{-\top}\}.
	\end{align}
	Here, step (i) from the condition that $N_k\geq (\tfrac{3}{4})^{K-k}N$ for all $k\in [K]$.
	To conclude, we use Young's inequality to obtain
	\begin{align*}
		\bbe\|\widehat v_K - v^*\|_\Pi^2 &\leq (1+\delta) \cdot \bbe\|\vbar - v^*\|_\Pi^2 + (1+\tfrac{1}{\delta}) \cdot \bbe\|\widehat v_K - \vbar\|_\Pi^2\\
		& \leq (1+\delta) \cdot \mathcal{A}(M, \gamma) \cdot  \inf_{v \in S}\|v-v^*\|_\Pi^2  + (1+\tfrac{1}{\delta}) \cdot \bbe\|\widehat v_K - \vbar\|_\Pi^2.
	\end{align*}

	\subsubsection{Proof of Proposition~\ref{prop_VRTD}}
	We first provide an upper bound on the term  $\bbe\|\widehat v_k - \underline v\|_{\Pi}^2$. By the updates of the VRTD algorithm, we have
	\begin{align*}
		\theta_{t+1}-\underline \theta=\theta_{t} - \eta \big(\widetilde g(\theta_t,\xi_t)-\widetilde g(\widetilde \theta,\xi_t) + \widehat g(\widetilde \theta)\big)-\underline \theta.
	\end{align*}
	Now write
	\begin{align}\label{VRTD_1}
		\bbe \|\theta_{t+1}-\underline \theta\|_2^2 
		&\overset{(i)}=\bbe\|\theta_t - \underline \theta  - \eta\big(g(\theta_t) - g(\underline \theta)\big) + \eta \big(g(\theta_t) - \widetilde g(\theta_t, \xi_t) - g(\widetilde \theta)+\widetilde g(\widetilde \theta,\xi_t)\big)\|_2^2\nn\\
		&= \bbe\|\theta_t - \underline \theta\|_2^2-2\eta\bbe\langle g(\theta_t) - g(\underline \theta), \theta_t-\underline \theta\rangle \nn\\&\quad+\eta^2\bbe\|g(\theta_t) - g(\underline \theta)\|_2^2 +\eta^2\bbe\|g(\theta_t) - \widetilde g(\theta_t, \xi_t) - g(\widetilde \theta)+\widetilde g(\widetilde \theta,\xi_t)\|_2^2\nn\\
		& \overset{(ii)}\leq \bbe\|\theta_t - \underline \theta\|_2^2 + \big( \eta^2\beta(1+\gamma)^2-2\eta(1-\gamma) \big)\bbe\|v_t-\underline v\|_\Pi^2 + \eta^2\varsigma^2\bbe\|\widetilde v- v_t\|_\Pi^2\nn
	\end{align}
	where step (i) follows from the fact that $  \widehat g(\widetilde \theta) = g(\widetilde \theta)-g(\underline \theta) $, and step (ii) follows from using Lemma~\ref{lemma_strong_monotone} to bound the second term, Lemma~\ref{lemma_lipschitz} to bound the third term and Assumption~\ref{assump_variance} to bound the last term. 
	Using Young's inequality then yields
	\begin{align}
		\bbe \|\theta_{t+1}&-\underline \theta\|_2^2 \nn\\
		&\leq \bbe\|\theta_t - \underline \theta\|_2^2+\big( \eta^2\beta(1+\gamma)^2 + 2\eta^2\varsigma^2-2\eta(1-\gamma)\big)\bbe\|v_t-\underline v\|_\Pi^2 + 2\eta^2\varsigma^2\bbe\|\widetilde v- \underline v\|_\Pi^2\nn\\
		&\overset{(i)}\leq \bbe\|\theta_t - \underline \theta\|_2^2 -\eta(1-\gamma)\bbe\|v_t-\underline v\|_\Pi^2 + 4\eta^2\varsigma^2\bbe\|\vbar - \widetilde v\|_\Pi^2 + 4\eta^2\varsigma^2\bbe\|\underline v - \vbar\|_\Pi^2,
	\end{align}
	where step (i) follows from the assumption $\eta \leq \min\{\tfrac{1-\gamma}{2\beta(1+\gamma)},\tfrac{1-\gamma}{4\varsigma^2}\}$ and a further application of Young's inequality. By taking a telescopic sum of Ineq.~\eqref{VRTD_1} from $t=1$ to $T$, we obtain
	\begin{align*}
		\tsum_{t=1}^T \eta(1-\gamma)\bbe[\|v_t-\underline v\|_\Pi^2] &+ \bbe[\|\theta_{T+1}-\underline \theta\|_2^2]  \\&\leq \bbe[\|\theta_1-\underline \theta\|_2^2]+ 4T\eta^2\varsigma^2\big(\bbe[\|\underline v - \vbar\|_\Pi^2]+ \bbe[\|\widetilde v - \vbar\|_\Pi^2]\big).
	\end{align*}
	Utilizing the inequalities $\|\theta\|_2^2 \geq \tfrac{1}{\beta}\|B^{\frac{1}{2}}\theta\|_2^2 = \tfrac{1}{\beta}\|v\|_\Pi^2$ and  $\|\theta\|_2^2 \leq \tfrac{1}{\mu}\|B^{\frac{1}{2}}\theta\|_2^2 = \tfrac{1}{\mu}\|v\|_\Pi^2$, we then have
	\begin{align*}
		\tsum_{t=1}^T \eta(1-\gamma)\bbe[\|v_t-\underline v\|_\Pi^2] &+ \tfrac{1}{\beta}\bbe[\|v_{T+1}-\underline v\|_\Pi^2]  \\&\leq \tfrac{1}{\mu}\bbe[\|v_1-\underline v\|_\Pi^2]+ 4T\eta^2\varsigma^2(\bbe[\|\underline v - \vbar\|_\Pi^2]+ \bbe[\|\widetilde v - \vbar\|_\Pi^2]).
	\end{align*}
	By definition, we have $\widehat v_k=\frac{\sum_{t=1}^T \eta (1-\gamma) v_t+(1/\beta )v_{T+1}}{T \eta (1-\gamma) +(1/\beta )}$; noting in addition that $v_1=\widetilde v$ then yields
	\begin{align}\label{VRTD_step_2}
		\bbe[\|\widehat v_k - \underline v\|_\Pi^2] &\leq \tfrac{1}{T \eta (1-\gamma) +(1/\beta )}\big[ \tsum_{t=1}^T \eta(1-\gamma)\bbe[\|v_t-\underline v\|_\Pi^2] + \bbe[\|v_{T+1}-\underline v\|_\Pi^2]\big]\nn\\
		&\leq \tfrac{1}{T \eta (1-\gamma) +(1/\beta )} \big[ \tfrac{1}{\mu}\bbe[\|\widetilde v-\underline v\|_\Pi^2]+ 4T\eta^2\varsigma^2(\bbe[\|\underline v - \vbar\|_\Pi^2]+ \bbe[\|\widetilde v - \vbar\|_\Pi^2])\big]\nn\\
		&\leq \tfrac{2/\mu+4T\eta^2\varsigma^2}{T \eta (1-\gamma) }(\bbe[\|\underline v - \vbar\|_\Pi^2]+ \bbe[\|\widetilde v - \vbar\|_\Pi^2]).
	\end{align}
	Combining the bound on $\|\underline v - v^*\|_\Pi^2$ shown in Lemma~\ref{bound_tilde_v} and the bound on $\|\widehat v_k - \underline v\|_\Pi^2$ from Ineq.~\eqref{VRTD_step_2}, we obtain
	\begin{align}\label{VRTD_step_3}
		\bbe[\|\widehat v_k &- \vbar\|_\Pi^2]\nn\\ &\leq 2 \bbe[\|\widehat v_k - \underline v\|_\Pi^2] + 2 \bbe[\|\vbar - \underline v\|_\Pi^2]\nn\\
		& \leq \tfrac{4/\mu+8T\eta^2\varsigma^2}{T \eta (1-\gamma) }\bbe[\|\widetilde v - \vbar\|_\Pi^2]+ \left(\tfrac{4/\mu+8T\eta^2\varsigma^2}{T \eta (1-\gamma) }+2\right)\bbe[\|\underline v - \vbar\|_\Pi^2]\nn\\
		& \leq \tfrac{4/\mu+8T\eta^2\varsigma^2}{T \eta (1-\gamma) }\bbe[\|\widetilde v - \vbar\|_\Pi^2] \nn \\
		& \quad + \left(\tfrac{4/\mu+8T\eta^2\varsigma^2}{T \eta (1-\gamma) }+2\right)\big(\tfrac{2\varsigma^2}{\mu(1-\gamma)^2N_k}\bbe[\|\widetilde v - \vbar \|_\Pi^2] 
		+\tfrac{2}{N_k}\trace\{(I_d-M)^{-1}\iidS(I_d-M)^{-\top}\}\big).
	\end{align}
	The final term takes the desired form; we are left with bounding the first two terms in the display above.
	First, recall by assumption that $\eta \leq \tfrac{1-\gamma}{32\varsigma^2}$, and $T\geq\tfrac{32}{\mu(1-\gamma)\eta}$, so that
	\begin{align}\label{cond_1}
		\tfrac{4/\mu+8T\eta^2\varsigma^2}{T\eta (1-\gamma) }\leq \tfrac{\frac{4}{\mu} + \frac{256\eta\varsigma^2}{\mu(1-\gamma)}}{32/\mu}\leq \tfrac{3}{8}.
	\end{align}
	Second, utilizing the condition $N_k\geq \tfrac{38\varsigma^2}{\mu(1-\gamma)^2}$, we have 
	\begin{align}\label{cond_2}
		\left(\tfrac{4/\mu+8T\eta^2\varsigma^2}{T\eta (1-\gamma) }+2\right)\tfrac{2\varsigma^2}{\mu(1-\gamma)^2N_k} \leq \tfrac{1}{8}.
	\end{align}
	Substituting the bounds~\eqref{cond_1} and~\eqref{cond_2} into inequality \eqnok{VRTD_step_3}, we obtain
	\begin{align*}
		\bbe[\|\widehat v_k - \vbar\|_\Pi^2] &\leq \tfrac{1}{2}\bbe[\|\widetilde v - \vbar\|_\Pi^2] + \tfrac{5}{N_k}\trace\{(I_d-M)^{-1}\iidS(I_d-M)^{-\top}\}\big),
	\end{align*}
	which completes the proof of Proposition~\ref{prop_VRTD}.

	\subsection{Proof of Theorem~\ref{theorem_VRFTD}}\label{proof_theorem_VRFTD}
	The structure of the proof is similar to the analysis of VRTD. We first state a proposition that characterizes the progress of the VRFTD algorithm in a single epoch.
	\begin{proposition}\label{prop_VRFTD}
		Assume that for each epoch $k\in[K]$, the parameter $\eta$, $\lambda$, $m$, $N_k$ and $T$ satisfy 
		\beq\label{stepsize_VRFTD}
		\eta \leq \tfrac{1}{4\beta(1+\gamma)}, ~\lambda=1, ~T\geq\tfrac{32}{\mu(1-\gamma)\eta}, ~m \geq \max\{1, \tfrac{256\eta \varsigma^2}{1-\gamma}\},~\text{and}~N_k\geq \tfrac{56\varsigma^2 }{\mu(1-\gamma)^2}.
		\eeq
		Set the output of this epoch to be $\widehat v_k:=\frac{\sum_{t=2}^{T+1} v_t}{T}$. Then we have
		\begin{align}\label{VRFTD_prop}
			\bbe[\|\widehat v_k - \vbar\|_\Pi^2] &\leq \tfrac{1}{2}\bbe[\|\widehat v_{k-1} - \vbar\|_\Pi^2] + \tfrac{5}{N_k}\trace\{(I_d-M)^{-1}\iidS(I_d-M)^{-\top}\}.
		\end{align}
	\end{proposition}
	
	Taking Proposition~\ref{prop_VRFTD} as given, the proof of Theorem~\ref{theorem_VRFTD} follows exactly as the proof of Theorem~\ref{theorem_VRTD_2} in Section~\ref{proof_theorem_VRTD}. Therefore, to complete the proof, it remains to prove Proposition~\ref{prop_VRFTD}, which we do in the following subsection.
	
	\subsubsection{Proof of Proposition~\ref{prop_VRFTD}}
	The proof structure parallels that of Proposition~\ref{prop_VRTD}: The main idea is to bound the terms $\|\widehat v_k - \underline v\|_\Pi^2$ and $\|\underline v - \vbar\|_\Pi^2$ separately and then use Young's inequality to bound $\|\widehat v_k - \vbar\|_\Pi^2$. As alluded to before, the bound proved on $\|\underline v - \vbar\|_\Pi^2$ in Lemma~\ref{bound_tilde_v} is still valid. It remains to show a bound on the term $\|\widehat v_k - \underline v\|_\Pi^2$.
	
	From step~\eqnok{FTD_step_2} of Algorithm~\ref{alg:FTD_2}, we have
	\begin{align*}
		\theta_t-\underline \theta   = \theta_{t+1}+ \eta\left[\widetilde F_t(\theta_t)+\lambda \left(\widetilde F_t(\theta_t)-\widetilde F_{t-1}(\theta_{t-1})\right)\right] - \underline \theta.
	\end{align*}
	As before, we like to obtain a one-step inequality that will allow us to take a telescopic sum. Toward that end, note that some manipulation yields 
	\begin{align*}
		\|\theta_t-\underline \theta\|_2^2 &=\|\theta_{t+1}-\underline \theta\|_2^2 + 2\eta \langle\widetilde F_t(\theta_t)+\lambda [\widetilde F_t(\theta_t)-\widetilde F_{t-1}(\theta_{t-1})], \theta_{t+1}-\underline \theta \rangle +\|\theta_{t+1}-\theta_t\|_2^2\\
		&=\|\theta_{t+1}-\underline \theta\|_2^2 +\|\theta_{t+1}-\theta_t\|_2^2 +2\eta \langle \widetilde F_{t+1}(\theta_{t+1}), \theta_{t+1} - \underline \theta \rangle \\
		&\quad- 2\eta \langle \widetilde F_{t+1}(\theta_{t+1})-\widetilde F_t(\theta_t), \theta_{t+1}-\underline \theta \rangle+ 2\eta \lambda\langle \widetilde F_{t}(\theta_t)-\widetilde F_{t-1}(\theta_{t-1}), \theta_t-\underline \theta \rangle\\
		&\quad+ 2\eta \lambda\langle \widetilde F_{t}(\theta_t)-\widetilde F_{t-1}(\theta_{t-1}), \theta_{t+1}-\theta_t \rangle.
	\end{align*}
	Defining $\widetilde \Delta_t(\theta) := \widetilde g_t(\theta) - g(\theta)$ for convenience and recalling the definition of $\widetilde F_t$ in Algorithm~\ref{alg:FTD_2}, we can then write
	\begin{align*}
		\|&\theta_t-\underline \theta\|_2^2\\
		&=\|\theta_{t+1}-\underline \theta\|_2^2 +\|\theta_{t+1}-\theta_t\|_2^2 +2\eta \langle \widetilde F_{t+1}(\theta_{t+1}), \theta_{t+1} - \underline \theta \rangle - 2\eta \langle \widetilde F_{t+1}(\theta_{t+1})-\widetilde F_{t}(\theta_t), \theta_{t+1}-\underline \theta \rangle \\
		&\qquad \qquad+ 2\eta \lambda\langle \widetilde F_{t}(\theta_t)-\widetilde F_{t-1}(\theta_{t-1}), \theta_t-\underline \theta \rangle+ 2\eta \lambda\langle g(\theta_t) - g(\theta_{t-1}), \theta_{t+1}-\theta_t \rangle \\
		&\qquad \qquad + 2\eta \lambda \langle  \widetilde \Delta_t(\theta_t) - \widetilde \Delta_t(\widetilde\theta)-\widetilde \Delta_{t-1}(\theta_{t-1}) + \widetilde \Delta_{t-1}(\widetilde\theta)  ,\theta_{t+1}-\theta_t \rangle.
	\end{align*}
	Summing both sides of the equation above from $t=1$ to $T$ and noting that $\lambda=1$, $\theta_1=\theta_0$, and $\widetilde F_0(\theta_0)=\widetilde F_1(\theta_1)$, we have
	\begin{align}\label{VRFTD_step_1}
		\|\theta_1-\underline \theta\|_2^2 &= \|\theta_{T+1}-\underline \theta\|_2^2 -2\eta \langle \widetilde F_{T+1}(\theta_{T+1})-\widetilde F_T(\theta_T), \theta_{T+1 }- \underline \theta\rangle + Q_1\nn\\&\qquad\qquad\qquad + \tsum_{t=1}^T 2\eta \langle \widetilde F_{t+1}(\theta_{t+1}), \theta_{t+1} - \underline \theta \rangle,
	\end{align}
	where we have defined
	\begin{align*}
		Q_1 &:= \tsum_{t=1}^T\|\theta_{t+1}-\theta_t\|_2^2 +  \tsum_{t=1}^T 2\eta \langle g(\theta_t) - g(\theta_{t-1}), \theta_{t+1}-\theta_t \rangle \\
		&\qquad \qquad \qquad +  \tsum_{t=1}^T 2\eta  \langle  \widetilde \Delta_t(\theta_t) - \widetilde \Delta_t(\widetilde\theta)-\widetilde \Delta_{t-1}(\theta_{t-1}) + \widetilde \Delta_{t-1}(\widetilde\theta)  ,\theta_{t+1}-\theta_t \rangle
	\end{align*}
	for convenience.
	Noting that it suffices to lower bound $Q$, write
	\begin{align}\label{VRFTD_step_2}
		Q_1 &\overset{(i)}\geq \tsum_{t=1}^T\|\theta_{t+1}-\theta_t\|_2^2 -  \tsum_{t=1}^T 2\eta (1+\gamma)\beta\|\theta_t-\theta_{t-1}\|_2\| \theta_{t+1}-\theta_t \|_2\nn\\
		&\qquad \qquad -  \tsum_{t=1}^T 2\eta   \|  \widetilde \Delta_t(\theta_t) - \widetilde \Delta_t(\widetilde\theta)-\widetilde \Delta_{t-1}(\theta_{t-1}) + \widetilde \Delta_{t-1}(\widetilde\theta)\|_2 \|\theta_{t+1}-\theta_t\|_2\nn\\
		&=\tfrac{1}{4}\|\theta_{T+1}-\theta_T\|_2^2\nn\\
		&\qquad+\tsum_{t=1}^T\left [ \tfrac{1}{4} \|\theta_{t+1}-\theta_t\|_2^2 +\tfrac{1}{4} \|\theta_{t}-\theta_{t-1}\|_2^2  -2\eta (1+\gamma)\beta\|\theta_t-\theta_{t-1}\|_2\| \theta_{t+1}-\theta_t \| _2 \right]\nn\\
		&\qquad + \tsum_{t=1}^T \left[ \tfrac{1}{2}\|\theta_{t+1}-\theta_t\|_2^2-2\eta   \|  \widetilde \Delta_t(\theta_t) - \widetilde \Delta_t(\widetilde\theta)-\widetilde \Delta_{t-1}(\theta_{t-1}) + \widetilde \Delta_{t-1}(\widetilde\theta)\|_2 \|\theta_{t+1}-\theta_t\| _2\right]\nn\\
		&\overset{(ii)}\geq \tfrac{1}{4}\|\theta_{T+1}-\theta_T\|_2^2 -  \tsum_{t=1}^T 2\eta^2\|  \widetilde \Delta_t(\theta_t) - \widetilde \Delta_t(\widetilde\theta)-\widetilde \Delta_{t-1}(\theta_{t-1}) + \widetilde \Delta_{t-1}(\widetilde\theta)\|_2^2\nn\\
		&\overset{(iii)}\geq \tfrac{1}{4}\|\theta_{T+1}-\theta_T\|_2^2 -  \tsum_{t=1}^T 4\eta^2(\|  \widetilde \Delta_t(\theta_t) - \widetilde \Delta_t(\widetilde\theta) \|_2^2+\|  \widetilde \Delta_{t-1}(\widetilde\theta)-\widetilde \Delta_{t-1}(\theta_{t-1})\|_2^2),
	\end{align}
	where step (i) follows from the Cauchy--Schwarz inequality and Lemma~\ref{lemma_lipschitz}, step (ii) follows from the fact that $4\eta(1+\gamma)\beta\leq 1$ and Young's inequality, and step (iii) follows from Young's inequality.
	Substituting Ineq.~\eqnok{VRFTD_step_2} into Eq.~\eqnok{VRFTD_step_1}, we obtain
	\begin{align*}
		\|\theta_{T+1}-\underline \theta\|_2^2 + \tfrac{1}{4}\|\theta_{T+1}&-\theta_T\|_2^2 + \tsum_{t=1}^T 2\eta \langle \widetilde F_{t+1}(\theta_{t+1}), \theta_{t+1} - \underline \theta \rangle \\
		&\leq \|\theta_1-\underline \theta\|_2^2 +2\eta \langle \widetilde F_{T+1}(\theta_{T+1})-\widetilde F_T(\theta_T), \theta_{T+1 }- \underline \theta\rangle\\
		&\qquad+  \tsum_{t=1}^T 4\eta^2(\|  \widetilde \Delta_t(\theta_t) - \widetilde \Delta_t(\widetilde\theta) \|_2^2+\|  \widetilde \Delta_{t-1}(\widetilde\theta)-\widetilde \Delta_{t-1}(\theta_{t-1})\|_2^2).
	\end{align*}
	Rearranging terms and invoking the definition of the function $\widetilde F(\cdot)$, we have
	\begin{align}
		&\|\theta_{T+1}-\underline \theta\|_2^2 + \tsum_{t=1}^T 2\eta \langle \widetilde F_{t+1}(\theta_{t+1}), \theta_{t+1} - \underline \theta \rangle \nn \\
		&\quad \leq \|\theta_1-\underline\theta\|_2^2+  4\eta^2\big(\tsum_{t=1}^T 2\|  \widetilde \Delta_t(\theta_t) - \widetilde \Delta_t(\widetilde\theta) \|_2^2-\|  \widetilde \Delta_T(\widetilde\theta)-\widetilde \Delta_T(\theta_{T})\|_2^2\big) \underbrace{-\tfrac{1}{4}\|\theta_{T+1}-\theta_T\|_2^2}_{R_1} \nn \\
		&\quad \qquad + \underbrace{2\eta\langle  \widetilde \Delta_{T+1}(\theta_{T+1}) - \widetilde \Delta_{T+1}(\widetilde\theta) + \widetilde \Delta_T(\widetilde\theta)-\widetilde \Delta_T(\theta_{T})  , \theta_{T+1}-\underline \theta\rangle}_{R_2}\nn\\ 
		&\quad\qquad\qquad+ \underbrace{2\eta\langle g(\theta_{T+1})-g(\theta_T),\theta_{T+1}-\underline \theta\rangle}_{R_3}. \label{eq:Tianjiao-long}
	\end{align}
	Using the shorthand
	\[
	Q_2 := 
	2\eta\langle \widetilde\Delta_{T+1}(\theta_{T+1}) - \widetilde \Delta_{T+1}(\widetilde\theta) , \theta_{T+1}-\underline \theta\rangle
	+2\eta\langle \widetilde \Delta_T(\widetilde\theta)-\widetilde \Delta_T(\theta_{T}) , \theta_{T}-\underline \theta\rangle,
	\]
	note that the final three terms on the RHS of Eq.~\eqref{eq:Tianjiao-long} can be written as
	\begin{align*}
		R_1 + R_2 + R_3 &\leq Q_2
		+ \big(2\eta \langle \widetilde \Delta_T(\widetilde\theta)-\widetilde \Delta_T(\theta_{T}), \theta_{T+1}-\theta_T \rangle -\tfrac{1}{8}\|\theta_{T+1}-\theta_T\|_2^2\big) \\
		&\qquad \qquad +\big(2\eta\langle  g(\theta_{T+1}) - g(\theta_T),\theta_{T+1}-\underline \theta\rangle  -\tfrac{1}{8}\|\theta_{T+1}-\theta_T\|_2^2\big) \\
		&\overset{(i)}{\leq} Q_2 + \big(2\eta\|\widetilde \Delta_T(\widetilde\theta)-\widetilde \Delta_T(\theta_{T})\|_2\|\theta_{T+1}-\theta_T\|_2-\tfrac{1}{8}\|\theta_{T+1}-\theta_T\|_2^2\big)\\
		&\qquad\qquad+\big(2\eta(1+\gamma)\beta\| \theta_{T+1}-\theta_T\|_2\|\theta_{T+1}-\underline \theta\|_2 -\tfrac{1}{8}\|\theta_{T+1}-\theta_T\|_2^2\big),
	\end{align*}
	where step (i) follows from the
	Cauchy--Schwarz inequality and Lemma~\ref{lemma_lipschitz}.
	Applying Young's inequality and putting together the pieces, we have
	\begin{align}\label{Markov_F}
		\big(1-8\eta^2\beta^2&(1+\gamma)^2\big)\|\theta_{T+1}-\underline \theta\|_2^2 + \tsum_{t=1}^T 2\eta \langle \widetilde F_{t+1}(\theta_{t+1}), \theta_{t+1}- \underline \theta \rangle  \nn\\
		\leq& \|\theta_1-\underline \theta\|_2^2+  \tsum_{t=1}^T 8\eta^2\|  \widetilde \Delta_t(\theta_t) - \widetilde \Delta_t(\widetilde\theta) \|_2^2+4\eta^2\|  \widetilde \Delta_T(\widetilde\theta)-\widetilde \Delta_T(\theta_{T})\|_2^2+ Q_2.
	\end{align}
	Note that the Ineq.~\eqref{Markov_F} holds pointwise; we now take expectation on both sides, and note that $\bbe[Q_2] = 0$. We now use Assumption~\ref{assump_variance} to upper bound on the RHS of Ineq.~\eqref{Markov_F} and Lemma~\ref{lemma_strong_monotone} to lower bound on the LHS of  Ineq.~\eqref{Markov_F}. Putting these together with the fact that $\|\theta_1-\underline \theta\|_2^2 \leq \tfrac{1}{\mu}\|v_1-\underline v\|_\Pi^2$, we obtain
	\begin{align*}
		\big(1-8\eta^2\beta^2(1+\gamma)^2\big)\bbe[\|\theta_{T+1}&-\underline \theta\|_2^2] + \tsum_{t=2}^{T+1} 2\eta(1-\gamma)\bbe[\|v_t-\underline v\|_\Pi^2]\\
		&\leq \tfrac{1}{\mu}\bbe[\|v_1-\underline v\|_\Pi^2]+ \tfrac{4\eta^2\varsigma^2}{m}\bbe[\|v_T-\widetilde v\|_\Pi^2]+  \tsum_{t=1}^T \tfrac{8\eta^2\varsigma^2}{m}\bbe[\|v_t-\widetilde v\|_\Pi^2].
	\end{align*}
	We now utilize the relation  $\|v_t-\widetilde v\|_\Pi^2 \leq 2\|v_t - \underline v\|_\Pi^2 + 2\|\widetilde v- \underline v\|_\Pi^2$ in conjunction with
	\begin{align*}
		v_1=\widetilde v \quad \text{ and } \quad 8\eta^2(1+\gamma)^2\beta^2\leq 1.
	\end{align*}
	Rearranging terms then yields
	\begin{align*}
		\tsum_{t=2}^{T+1} \big(2\eta(1-\gamma)-\tfrac{24\eta^2\varsigma^2}{m}\big)\bbe[\|v_t-\underline v\|_\Pi^2]  \leq \big(\tfrac{1}{\mu}+\tfrac{(16T+40)\eta^2\varsigma^2}{m}\big)\bbe[\|\widetilde v-\underline v\|_\Pi^2].
	\end{align*}
	Invoking the bounds $m\geq\tfrac{256\eta \varsigma^2}{1-\gamma}$, and $\widehat v_k = \tfrac{\sum_{t=2}^{T+1}v_t}{T}$, we have 
	\begin{align*}
		\bbe[\|\widehat v_k - \underline v\|_\Pi^2] &\leq \tfrac{\sum_{t=2}^{T+1}\eta(1-\gamma)\bbe[\|v_t-\underline v\|_\Pi^2] }{T\eta(1-\gamma)}\leq \tfrac{1/\mu+(16T+40)\eta^2\varsigma^2/m}{T\eta(1-\gamma)} \bbe[\|\widetilde v-\underline v\|_\Pi^2]\\&\leq \tfrac{1/\mu+(16T+40)\eta^2\varsigma^2/m}{T\eta(1-\gamma)} (2\bbe[\|\widetilde v-\vbar\|_\Pi^2]+2\bbe[\|\underline v-\vbar\|_\Pi^2]).
	\end{align*}
	Applying Young's inequality and Lemma~\ref{bound_tilde_v}, we obtain
	\begin{align}\label{VRFTD_step_3}
		&\bbe[\|\widehat v_k - \vbar\|_\Pi^2] \nn\\
		&\quad \leq 2 \bbe[\|\widehat v_k - \underline v\|_\Pi^2] + 2 \bbe[\|\vbar - \underline v\|_\Pi^2]\nn\\
		&\quad \leq \tfrac{4/\mu+(64T+160)\eta^2\varsigma^2/m}{T \eta (1-\gamma) }\bbe[\|\widetilde v - \vbar\|_\Pi^2]+ \left(\tfrac{4/\mu+(64T+160)\eta^2\varsigma^2/m}{T \eta (1-\gamma) }+2\right)\bbe[\|\underline v - \vbar\|_\Pi^2]\nn\\
		&\quad \leq \tfrac{4/\mu+(64T+160)\eta^2\varsigma^2/m}{T \eta (1-\gamma) }\bbe[\|\widetilde v - \vbar\|_\Pi^2] \nn \\
		&\quad \quad + \left(\tfrac{4/\mu+(64T+160)\eta^2\varsigma^2/m}{T \eta (1-\gamma) }+2\right)\bigg(\tfrac{2\varsigma^2}{\mu(1-\gamma)^2N_k}\bbe[\|(\widetilde v - \vbar) \|_\Pi^2]\nn\\
		&\qquad \qquad \qquad \qquad\qquad \qquad\qquad \qquad +\tfrac{2}{N_k}\trace\{(I_d-M)^{-1}\iidS(I_d-M)^{-\top}\}\bigg).
	\end{align}
	As before, we bound the first two terms by assumption. First, note that $\eta \leq \tfrac{1}{4\beta(1+\gamma)}$, $T \geq\tfrac{32}{\mu(1-\gamma)\eta}$ and $m \geq \max\{1, \tfrac{256\eta \varsigma^2}{1-\gamma}\}$ whereby
	\begin{align}\label{cond_3}
		\tfrac{4/\mu+(64T+160)\eta^2\varsigma^2/m}{T \eta (1-\gamma) }\leq \tfrac{4/\mu+T\eta(1-\gamma)/4 + 5\eta(1-\gamma)/8}{T\eta(1-\gamma)}\leq \tfrac{4/\mu +8/\mu + 5/(32\beta)}{32/\mu}\leq \tfrac{13}{32}.
	\end{align}
	Next, owing to the assumption $N_k\geq \tfrac{56\varsigma^2 }{\mu(1-\gamma)^2}$, we have 
	\begin{align}\label{cond_4}
		\left(\tfrac{4/\mu+(64T+160)\eta^2\varsigma^2/m}{T \eta (1-\gamma) }+2\right)\tfrac{2\varsigma^2}{\mu(1-\gamma)^2N_k} \leq \tfrac{3}{32}.
	\end{align}
	Substituting Ineqs.~\eqref{cond_3} and~\eqref{cond_4} into Ineq.~\eqref{VRFTD_step_3}, we obtain
	\begin{align*}
		\bbe[\|\widehat v_k - v^*\|_\Pi^2] &\leq \tfrac{1}{2}\bbe[\|\widetilde v - \vbar\|_\Pi^2] + \tfrac{5}{N_k}\text{trace}\{(I_d-M)^{-1}\iidS(I_d-M)^{-\top}\},
	\end{align*}
	which completes the proof of Proposition~\ref{prop_VRFTD}.

	\subsection{Proof of Theorem~\ref{theorem_VRFTD_mark}}\label{proof_theorem_VRFTD_mark}
	The structure of this proof is similar to proofs of  Theorems~\ref{theorem_VRTD_2} and~\ref{theorem_VRFTD}: We first derive a bound for a single epoch, and then apply it recursively to obtain the eventual convergence result. The following proposition characterizes the progress in a single epoch $k\in [K]$.
	\black{
		\begin{proposition}\label{prop_VRFTD_mark}
			Consider a single epoch with index $k \in [K]$. Consider an integer $\tau$ satisfying
			$
			\rho^{\tau} \leq \min\{\frac{2(1-\rho)\varsigma}{3\mixcon}, \frac{2(1-\rho)^2}{5 \mixcon}\}
			$. 
			Suppose the parameters $N_k$, $n_0$ and $m_0$ satisfy
			\begin{align}\label{n0_m0}
				\rho^{N_k-n_0} \leq \tfrac{\tau(1-\rho)}{5\mixcon (N_k-n_0)},~ \rho^{n_0} \leq \tfrac{\min_{i \in [D]}\pi_i}{\skipcon},~ \text{and}\quad
				\rho^{m_0} \leq \min \left\{\tfrac{\min_{i \in [D]}\pi_i}{\skipcon}, \tfrac{\sqrt{\mu}\eta\tau\varsigma^2(1-\rho)}{\mixcon}\right\}.
			\end{align}
			Suppose that the parameter $\eta$, $\lambda$, $m$, $N_k$ and $T$ satisfy 
			\beq\label{stepsize_VRFTD_mark}
			\eta \leq \tfrac{1}{4\beta(1+\gamma)}, ~~\lambda=1, ~~T \geq\tfrac{64}{\mu(1-\gamma)\eta}, ~~m -m_0 \geq \max\{1, \tfrac{792\eta(\tau+1) \varsigma^2}{1-\gamma}\},~\text{and}~N_k-n_0\geq \tfrac{206(\tau+1)\varsigma^2 }{\mu(1-\gamma)^2}.
			\eeq
			Set the output of this epoch to be $\widehat v_k:=\frac{\sum_{t=2}^{T+1} v_t}{T}$.
			Then we have the following results:\\
			\indent (a) If $v^* = \bar v$, we have
			\begin{align}\label{VRFTD_prop_mark_0}
				\bbe\|\widehat v_k - v^*\|_\Pi^2 \leq \tfrac{1}{2}\bbe\|\widehat v_{k-1}  - \vbar\|_\Pi^2
				+ \tfrac{10\cdot\trace\{(I_d-M)^{-1}\iidS(I_d-M)^{-\top}\}}{N_k-N_0}.
			\end{align}
			\indent (b) If $v^*\neq \bar v$, we have
			\begin{align}\label{VRFTD_prop_mark}
				\bbe\|\widehat v_k - v^*\|_\Pi^2 \leq \tfrac{1}{2}\bbe\|\widehat v_{k-1}  - \vbar\|_\Pi^2
				+ \tfrac{10\cdot\trace\{(I_d-M)^{-1}\MarkovS(I_d-M)^{-\top}\}}{N_k-N_0}+\tfrac{\widetilde{\mathcal{H}}}{(N_k-n_0)^2},
			\end{align}
			where 
			$
			\widetilde{\mathcal{H}}:= \tfrac{2(\tau+1)}{(1-\gamma)^2\mu}\|\bar v - v^*\|_\Pi^2+ (10\tau^2+ 2\tau + 2)\cdot\trace\{(I_d-M)^{-1}\iidS(I_d-M)^{-\top} \}.
			$
	\end{proposition}}
	\black{
		Taking Proposition~\ref{prop_VRFTD_mark} as given for the moment, let us complete the proof of the theorem. 
		%
		First, consider the case when $\bar v\neq v^*$. Recursively using Ineq.~\eqref{VRFTD_prop_mark} yields
		\begin{align}\label{step_5_mark}
			\bbe\|\widehat v_K - \vbar\|_\Pi^2 &\leq \tfrac{1}{2^K}\|v^0 - \vbar\|^2_\Pi  + \tsum_{k=1}^K\left( \tfrac{10\cdot\trace\{(I_d-M)^{-1}\MarkovS(I_d-M)^{-\top}\}}{2^{K-k}(N_k-n_0)}+ \tfrac{\widetilde{\mathcal{H}}}{2^{K-k}(N_k-n_0)^2}\right)\nn\\
			&\overset{(i)}\leq \tfrac{1}{2^K}\|v^0 - \vbar\|^2_\Pi  + \tsum_{k=1}^K(\tfrac{2}{3})^{K-k}\tfrac{10\cdot\trace\{(I_d-M)^{-1}\MarkovS(I_d-M)^{-\top}\}}{N}+\tsum_{k=1}^K(\tfrac{8}{9})^{K-k}\tfrac{\widetilde{\mathcal{H}}}{N^2} \nn\\
			&\leq \tfrac{1}{2^K}\|v^0 - \vbar\|^2_\Pi  + \tfrac{30\cdot\trace\{(I_d-M)^{-1}\MarkovS(I_d-M)^{-\top}\}}{N} + \tfrac{9\widetilde{\mathcal{H}}}{N^2},
		\end{align}
		where step~(i) follows from the condition $N_k - n_0 \geq (\tfrac{3}{4})^{K-k}N$. The proof of the case when $\bar v= v^*$ follows from the same derivation. }

	In order to complete the proof, it remains to prove Proposition~\ref{prop_VRFTD_mark}, which we do in the following subsection.
	
	\subsubsection{Proof of Proposition~\ref{prop_VRFTD_mark}}
	The basic structure of the proof is similar to the proof of Proposition~\ref{prop_VRFTD}. Recall that $\underline \theta$ satisfies $g(\underline \theta) - g(\widetilde \theta) + \widehat g(\widetilde \theta) = 0$ and $\underline v = \Psi^\top\underline \theta $. The main idea of the proof is to bound the terms $\|\widehat v_k - \underline v\|_\Pi^2$ and $\|\underline v - \vbar\|_\Pi^2$ separately and use Young's inequality to bound $\|\widehat v_k - \vbar\|_\Pi^2$. Before we start, we require some supporting lemmas that upper bound the variance and bias of the stochastic operator $\widetilde g$ and averaged stochastic operators $\widehat g$ and $\widetilde g_t$. 
	
	\begin{lemma}\label{stoch_opt_variance_1}
		For every $t\in \mathbb{Z}_{+}$, $\theta, \theta'\in \bbr^d$,  and $\tau \in \mathbb{Z}_{+}$ such that
		\begin{align}\label{cond_tau_1}
			\skipcon \cdot\rho^\tau \leq \min_{i\in[D]} \pi_i,
		\end{align}
		we have 
		\begin{align}\label{variance_markovian}
			\bbe[\|\widetilde g(\theta,\xi_{t+\tau}) - \widetilde g(\theta',\xi_{t+\tau})-g(\theta)+ g(\theta')\|^2_2|\mathcal{F}_t] \leq 2\varsigma^2\|v - v'\|_\Pi^2
		\end{align}
		with probability 1.
	\end{lemma}
	\noindent See Appendix~\ref{proof_stoch_opt_variance_1} for the proof of this lemma.
	\black{
		\begin{lemma}\label{stoch_opt_variance_2}
			For every $\theta, \theta'\in \bbr^d$, $n_0,\tau \in \mathbb{Z}_+$, if $n_0$ satisfies \eqref{cond_tau_1} and $\tau$ satisfies
			\begin{align}\label{cond_tau_0}
				\rho^{\tau} \leq \tfrac{2(1-\rho)\varsigma}{3\mixcon},
			\end{align}
			then we have with probability 1,
			\begin{align}\label{variance_markovian_0_1}
				\bbe\|\widehat g(\theta) - \widehat g(\theta')- g(\theta) + g(\theta')\|^2_2 \leq \tfrac{4(\tau+1)\varsigma^2}{N_k-n_0}\|v-v'\|_\Pi^2.
			\end{align}
	\end{lemma}}
	\noindent See Appendix~\ref{proof_stoch_opt_variance_2} for the proof of this lemma.
	
	Note that $\widetilde g_t(\cdot)$ has the same property~\eqref{variance_markovian_0_1} as $\widehat g(\cdot)$. If $m_0$ satisfies Ineq.~\eqref{cond_tau_1} and $\tau$ satisfies Ineq.~\eqref{cond_tau_0}, then we have that for any $t\in[T]$,
	\begin{align}\label{variance_markovian_0_2}
		\bbe[\|\widetilde g_t(\theta) - \widetilde g_t(\theta')- g(\theta) + g(\theta')\|^2_2] \leq \tfrac{4(\tau+1)\varsigma^2}{m-m_0}\|v-v'\|_\Pi^2.
	\end{align}
	\begin{lemma}\label{stoch_opt_variance_3}
		For every $\theta, \theta'\in \bbr^d$, $m_0 \in [m]$ and $t\in [T]$, with probability 1,
		\begin{align}\label{bias_markovian_0}
			\|\bbe[\widetilde g_t(\theta) |\mathcal{F}_{t-1}]- \bbe[\widetilde g_t(\theta')|\mathcal{F}_{t-1}]- g(\theta) + g(\theta')\|_2 \leq \tfrac{\mixcon \rho^{m_0}}{(1-\rho)(m-m_0)}\|v-v'\|_\Pi. 
		\end{align}
	\end{lemma}
	\noindent See Appendix~\ref{proof_stoch_opt_variance_3} for the proof of this lemma.
	
	With these supporting lemmas in hand, we are now ready to proceed to the meat of the proof. We first focus on upper-bounding $\|\vbar - \underline v\|_\Pi^2$ in the following lemma.
	\black{
		\begin{lemma}\label{bound_tilde_v_Mark}
			Consider a single epoch with index $k \in [K]$. Assume $n_0$ satisfies Ineq. \eqref{cond_tau_1}, $\tau$ satisfies Ineq. \eqref{cond_tau_0} and $\rho^{\tau} \leq  \frac{2(1-\rho)^2}{5 \mixcon}$, and $N_k$ satisfies $\rho^{N_k-n_0} \leq \tfrac{\tau(1-\rho)}{5\mixcon (N_k-n_0)}$.\\
			\noindent(a) If $\bar v = v^*$, then we have 
			\begin{align}\label{VRFTD_step_0_mark_0}
				\bbe[\|\underline v - \vbar \|^2_\Pi]\leq& \tfrac{4\cdot \trace\big\{(I_d-M)^{-1}\iidS(I_d-M)^{-\top}\big\}}{N_k-N_0}  + \tfrac{8(\tau+1)\varsigma^2}{(1-\gamma)^2\mu(N_k-n_0)}\|\widetilde v -\bar v\|_\Pi^2.
			\end{align}
			\noindent(b) If $\bar v \neq v^*$, then we have
			\begin{align}\label{VRFTD_step_0_mark}
				\bbe[\|\underline v - \vbar \|^2_\Pi]\leq& \tfrac{4\cdot \trace\big\{(I_d-M)^{-1}\MarkovS(I_d-M)^{-\top}\big\}}{N_k-N_0}  + \tfrac{4(\tau+1)}{5(1-\gamma)^2\mu(N_k-n_0)^2} \|\bar v - v^*\|_\Pi^2\nn\\
				&~~ + \tfrac{(20\tau^2 + 4\tau + 4)\cdot\trace\big\{(I_d-M)^{-1}\iidS(I_d-M)^{-\top} \big\}}{5(N_k-n_0)^2}+ \tfrac{8(\tau+1)\varsigma^2}{(1-\gamma)^2\mu(N_k-n_0)}\|\widetilde v -\bar v\|_\Pi^2.
			\end{align}
		\end{lemma}
	}
	\noindent See Appendix~\ref{proof_bound_tilde_v_Mark} for the proof of this lemma.
	
	Next, we turn our attention to providing an upper bound on $\|\widehat v_k - \underline v\|_\Pi^2$.
	At this juncture, it is important to note that our notation is chosen to be consistent with the i.i.d. setting. In particular, we keep the definition $\widetilde \Delta_t(\theta):=\widetilde g_t (\theta) - g(\theta)$, and note that $\widetilde g_t (\theta)$ here is defined in Algorithm~\ref{alg:VRFTD_markovian} rather than Algorithm~\ref{alg:FTD_2}. 
	Let $F(\theta):=g(\theta)-g(\widetilde \theta)+\widehat g(\widetilde \theta)$. 
	
	Now recall that in the proof of the corresponding proposition in the i.i.d. setting (Proposition~\ref{prop_VRFTD}), our first few steps up to Ineq.~\eqref{Markov_F} were obtained purely via algebraic manipulation. Therefore, Ineq.~\eqref{Markov_F} is still valid in the Markovian noise setting, and rearranging terms in Ineq.~\eqref{Markov_F} yields the bound
	\begin{align}\label{Markov_F_1}
		&(1-8\eta^2(1+\gamma)^2 \beta^2)\|\theta_{T+1}-\underline \theta\|_2^2 + \tsum_{t=1}^T 2\eta \langle  F(\theta_{t+1}), \theta_{t+1}- \underline \theta \rangle  \nn\\
		&\qquad \qquad \leq \|\theta_1-\underline \theta\|_\Pi^2+4\eta^2\|  \widetilde \Delta_T(\widetilde\theta)-\widetilde \Delta_T(\theta_{T})\|_2^2+ Q_3+  \tsum_{t=1}^T 8\eta^2\|  \widetilde \Delta_t(\theta_t) - \widetilde \Delta_t(\widetilde\theta) \|_2^2,
	\end{align}
	where we have defined 
	\begin{align*}
		Q_3 &:=  
		2\eta\langle \widetilde\Delta_{T+1}(\theta_{T+1}) +\widetilde \Delta_{T+1}(\widetilde\theta) , \theta_{T+1}-\underline \theta\rangle-2\eta\langle \widetilde \Delta_T(\widetilde\theta)-\widetilde \Delta_T(\theta_{T}) , \theta_{T}-\underline \theta\rangle\\
		&\qquad \qquad-\tsum_{t=1}^T 2\eta \langle \widetilde F_{t+1}(\theta_{t+1}) -F(\theta_{t+1}) , \theta_{t+1}- \underline \theta \rangle \\
		&~=  
		2\eta\langle \widetilde\Delta_{T+1}(\theta_{T+1}) - \widetilde \Delta_{T+1}(\widetilde\theta) , \theta_{T+1}-\underline \theta\rangle+2\eta\langle \widetilde \Delta_T(\widetilde\theta)-\widetilde \Delta_T(\theta_{T}) , \theta_{T}-\underline \theta\rangle\\
		&\qquad \qquad-\tsum_{t=1}^T 2\eta \langle\widetilde \Delta_{t+1}(\theta_{t+1}) - \widetilde \Delta_{t+1}(\widetilde\theta), \theta_{t+1}- \underline \theta \rangle .
	\end{align*}
	Recall that $F(\underline \theta)=0$. Invoking Lemma~\ref{lemma_strong_monotone} yields
	\begin{align}\label{Markov_F_2}
		&\big(1-8\eta^2(1+\gamma)^2\beta^2\big)\|\theta_{T+1}-\underline \theta\|_2^2 + \tsum_{t=1}^T 2\eta(1-\gamma)\|v_{t+1}- \underline v\|_\Pi^2  \nn\\
		&\qquad \qquad \qquad \leq \|\theta_1-\underline \theta\|_\Pi^2+4\eta^2\|  \widetilde \Delta_T(\widetilde\theta)-\widetilde \Delta_T(\theta_{T})\|_2^2+ Q_3+  \tsum_{t=1}^T 8\eta^2\|  \widetilde \Delta_t(\theta_t) - \widetilde \Delta_t(\widetilde\theta) \|_2^2.
	\end{align}
	As before, it remains to take an expectation on both sides of Ineq.~\eqref{Markov_F_2}. Note that in the Markovian noise setting, the operator $\widetilde g_t$ is not unbiased, and this necessitates a more careful analysis. We provide an upper bound on $\bbe[Q_3]$ by writing 
	\black{
		\begin{align}\label{bound_Q_3}
			\bbe[Q_3]  &= 2\eta  \bbe \langle \widetilde\Delta_{T+1}(\theta_{T+1}) - \widetilde \Delta_{T+1}(\widetilde\theta) , \theta_{T+1}-\underline \theta\rangle+2\eta  \bbe \langle \widetilde \Delta_T(\widetilde\theta)-\widetilde \Delta_T(\theta_{T}) , \theta_{T}-\underline \theta\rangle \nn\\
			&\qquad \qquad-\tsum_{t=1}^T 2\eta   \bbe \langle\widetilde \Delta_{t+1}(\theta_{t+1}) - \widetilde \Delta_{t+1}(\widetilde\theta), \theta_{t+1}- \underline \theta \rangle \nn\\
			&\overset{(i)} \leq \tfrac{2\eta\mixcon \rho^{m_0}\cdot \bbe[\|v_{T+1} - \widetilde v \|_\Pi\|\theta_{T+1}-\underline \theta\|_2]}{(1-\rho)(m-m_0)}  + \tfrac{2\eta\mixcon \rho^{m_0}\cdot\bbe[\|v_T - \widetilde v\|_\Pi\|\theta_{T}-\underline \theta\|_2]}{(1-\rho)(m-m_0)} \nn \\
			&\qquad \qquad+\tsum_{t=1}^T \tfrac{2\eta\mixcon \rho^{m_0}\cdot \bbe [\|v_{t+1} - \widetilde v\|_\Pi\| \theta_{t+1}- \underline \theta \|_2]}{(1-\rho)(m-m_0)}  \nn\\
			&\overset{(ii)} \leq \tfrac{2\eta\mixcon \rho^{m_0}\cdot\bbe[\|v_{T+1} - \widetilde v \|_\Pi\|v_{T+1}-\underline v\|_\Pi]}{\sqrt{\mu}(1-\rho)(m-m_0)} + \tfrac{2\eta\mixcon \rho^{m_0}}{\sqrt{\mu}(1-\rho)(m-m_0)} \cdot\bbe[\|v_T - \widetilde v\|_\Pi\|v_T-\underline v\|_\Pi]\nn \\
			&\qquad \qquad+\tsum_{t=1}^T \tfrac{2\eta\mixcon \rho^{m_0}\cdot \bbe [\|v_{t+1} - \widetilde v\|_\Pi\| v_{t+1}- \underline v\|_\Pi]}{\sqrt{\mu}(1-\rho)(m-m_0)} ,
	\end{align}}
	where step~(i) follows from the Cauchy--Schwarz inequality and Lemma \ref{stoch_opt_variance_3}, and step (ii) follows from the fact that $\|\theta - \theta'\|_2\leq \tfrac{1}{\sqrt{\mu}}\|v-v'\|_\Pi$. 
	Next, note from Ineq.~\eqref{variance_markovian_0_2} that
	\black{
		\begin{align*}
			4\eta^2\bbe\|  \widetilde \Delta_T(\widetilde\theta)-\widetilde \Delta_T(\theta_{T})\|_2^2&+  \tsum_{t=1}^T 8\eta^2\bbe\|  \widetilde \Delta_t(\theta_t) - \widetilde \Delta_t(\widetilde\theta) \|_2^2\\& \leq \tfrac{16\eta^2(\tau+1)\varsigma^2}{m-m_0}\bbe[\|v_T-\widetilde v\|_\Pi^2]+  \tsum_{t=1}^T \tfrac{32\eta^2(\tau+1)\varsigma^2}{m-m_0}\bbe[\|v_t-\widetilde v\|_\Pi^2].
	\end{align*}}
	Putting together the pieces yields
	\black{
		\begin{align*}
			\big(1-8\eta^2&(1+\gamma)^2\beta^2\big)\bbe[\|\theta_{T+1}-\underline \theta\|_2^2] + \tsum_{t=2}^{T+1} 2\eta(1-\gamma)\bbe[\|v_t-\underline v\|_\Pi^2]\\
			&\leq \tfrac{1}{\mu}\bbe[\|v_1-\underline v\|_\Pi^2]+\tfrac{2\eta\mixcon \rho^{m_0}}{\sqrt{\mu}(1-\rho)(m-m_0)}\bbe[\|v_{T}-\widetilde v\|_\Pi\|v_T-\underline v\|_\Pi ]\\
			&\quad\quad+\tfrac{2\eta\mixcon \rho^{m_0}}{\sqrt{\mu}(1-\rho)(m-m_0)}\bbe[\|v_{T+1}-\widetilde v\|_\Pi\|v_{T+1}-\underline v\|_\Pi]+ \tfrac{16\eta^2(\tau+1)\varsigma^2}{m-m_0}\bbe[\|v_T-\widetilde v\|_\Pi^2]\\
			&\quad\quad+  \tsum_{t=1}^T \tfrac{32\eta^2(\tau+1)\varsigma^2}{m-m_0}\bbe[\|v_t-\widetilde v\|_\Pi^2]+\tsum_{t=2}^{T+1}  \tfrac{2\eta\mixcon \rho^{m_0}}{\sqrt{\mu}(1-\rho)(m-m_0)}\bbe[\|v_{t}-\widetilde v\|_\Pi\|v_t-\underline v\|_\Pi].
	\end{align*}}
	We now use the relations 
	\begin{align*}
		\|v_t-\widetilde v\|_\Pi^2 \leq 2\|v_t - \underline v\|_\Pi^2 + 2\|\widetilde v- \underline v\|_\Pi^2, &\quad  \|v_t-\widetilde v\|_\Pi\|v_t-\underline v\|_\Pi\leq \tfrac{1}{2}(\|v_t-\widetilde v\|_\Pi^2 + \|v_t-\underline v\|_\Pi^2), \text{ and } \\
		&v_1=v_0=\widetilde v.
	\end{align*}
	Rearranging terms above, we then obtain
	\black{
		\begin{align*}
			&\tsum_{t=2}^{T+1} \big(2\eta(1-\gamma)-\tfrac{96\eta^2(\tau+1)\varsigma^2}{m-m_0} - \tfrac{3\eta\mixcon \rho^{m_0}}{\sqrt{\mu}(1-\rho)(m-m_0)}\big)\bbe\|v_t-\underline v\|_\Pi^2  \\
			&\leq \big(\tfrac{1}{\mu}+\tfrac{(64T+32)\eta^2(\tau+1)\varsigma^2}{m-m_0}+\tfrac{2(T+2)\eta\mixcon \rho^{m_0}}{\sqrt{\mu}(1-\rho)(m-m_0)}\big)\bbe\|\widetilde v-\underline v\|_\Pi^2 -\big(1-8\eta^2(1+\gamma)^2\beta^2\big)\bbe\|\theta_{T+1}-\underline \theta\|_2^2.
	\end{align*}}
	Recall that \black{$m-m_0\geq\tfrac{192\eta (\tau+1)\varsigma^2}{1-\gamma}$, $1\geq 8\eta^2(1+\gamma)^2\beta^2$, $\rho^{m_0}\leq \tfrac{\sqrt{\mu}\eta\tau\varsigma^2(1-\rho)}{\mixcon}$ and $\widehat v_k = \tfrac{\sum_{t=2}^{T+1}v_t}{T}$.} Combining these with Jensen's inequality yields 
	\black{
		\begin{align*}
			\bbe[\|\widehat v_k - \underline v\|_\Pi^2] &\leq \tfrac{\sum_{t=2}^{T+1}\eta(1-\gamma)\bbe[\|v_t-\underline v\|_\Pi^2] }{T\eta(1-\gamma)}\leq \tfrac{1/\mu+(66T+36)\eta^2(\tau+1)\varsigma^2/(m-m_0)}{T\eta(1-\gamma)} \bbe[\|\widetilde v-\underline v\|_\Pi^2]\\&\leq \tfrac{1/\mu+(66T+36)\eta^2(\tau+1)\varsigma^2/(m-m_0)}{T\eta(1-\gamma)} (2\bbe[\|\widetilde v-\vbar\|_\Pi^2]+2\bbe[\|\underline v-\vbar\|_\Pi^2]).
	\end{align*}}
	By Young's inequality, the following chain of bounds holds:
	\black{
		\begin{align}\label{VRFTD_step_3_mark}
			\bbe[\|\widehat v_k - \vbar\|_\Pi^2] &\leq 2 \bbe[\|\widehat v_k - \underline v\|_\Pi^2]+ 2 \bbe[\|\vbar - \underline v\|_\Pi^2] \nn\\
			& \leq D_1\cdot\bbe[\|\widetilde v - \vbar\|_\Pi^2]+ \left(D_1+2\right)\bbe[\|\underline v - \vbar\|_\Pi^2].
		\end{align}
		where $D_1:=\tfrac{4/\mu+(264T+144)\eta^2(\tau+1)\varsigma^2/(m-m_0)}{T \eta (1-\gamma) }$.}
	\black{By assumption, we have $\eta \leq \tfrac{1}{4\beta(1+\gamma)}$, $T \geq\tfrac{64}{\mu(1-\gamma)\eta}$, and $m-m_0 \geq \max\{1, \tfrac{792\eta(\tau+1) \varsigma^2}{1-\gamma}\}$. Thus,
		\begin{align}\label{cond_3_mark}
			D_1\leq \tfrac{4/\mu+T\eta(1-\gamma)/3 + 2\eta(1-\gamma)/11}{T\eta(1-\gamma)}\leq \tfrac{4/\mu +64/3\mu + 1/22\beta}{64/\mu}\leq \tfrac{26}{64}.
		\end{align}
		Next, note that $N_k-n_0\geq \tfrac{206(\tau+1)\varsigma^2 }{\mu(1-\gamma)^2}$ to write
		\begin{align}\label{cond_4_mark}
			\left(D_1+2\right)\cdot\tfrac{8(\tau+1)\varsigma^2}{(1-\gamma)^2\mu(N_k-n_0)}
			\leq\tfrac{154}{64}\cdot\tfrac{8}{206} 
			< \tfrac{6}{64}.
		\end{align}
		Substituting Ineqs.~\eqnok{cond_3_mark} and~\eqnok{cond_4_mark} into Ineq.~\eqnok{VRFTD_step_3_mark}, and utilizing Lemma~\ref{bound_tilde_v_Mark} to bound $\bbe[\|\bar v -\underline v\|_\Pi^2]$, we finally obtain the desired results for both cases of $v^*=\bar v$ and $v^*\neq \bar v$ as desired.
	}

	\section{Discussion}
	In this paper, we investigated the problem of policy evaluation with linear function approximation, making three contributions. First, we proved lower bounds on both deterministic error and stochastic error. With these lower bounds in hand, we presented an analysis of a variance-reduced variant of temporal difference algorithm (VRTD) in the i.i.d. observation model and showed that it fails to match the oracle complexity lower bound on the deterministic error. In order to remedy this difficulty, we developed an optimal variance-reduced fast temporal difference algorithm (VRFTD) that nearly matches both lower bounds simultaneously. Finally, we extended the VRFTD algorithm to the Markovian setting and provided instance-dependent convergence results. \black{The leading stochastic error matches the instance-dependent lower bound for Markovian linear stochastic approximation \citep{mou2021optimal}, and the deterministic error matches the i.i.d. setting up to a multiplicative factor proportional to the mixing time of the chain.} Our theoretical guarantees were corroborated with numerical experiments in both the i.i.d. and Markovian settings, showing that the VRFTD algorithm enjoys several advantages over the prior state-of-the-art. 
	
	Our work leaves open severaal salient future directions; let us mention two. 
	\black{First, our oracle complexity lower bound is proved in the tabular setting. On the other hand, our upper bounds on the deterministic error indicate that with linear function approximation, we pay a multiplicative factor depending on the condition number of the feature matrix. It would be interesting to see if an oracle complexity lower bound can be proved under linear function approximation, and whether the linear dependence on the condition number in our bounds is optimal.}
	Second, and more broadly, note that our analysis relies heavily on the linear structure of the problem. However, there are many problems in the reinforcement learning literature that have nonlinear structures, e.g., the policy optimization problem involving the Bellman optimality operator. An interesting direction for future work is to understand problems with nonlinear structure from an instance-specific point of view and develop efficient algorithms to capture the optimal deterministic and stochastic errors. For instance, variance reduction has been applied to the policy optimization problem under the generative model~\citep{sidford2018near,wainwright2019variance} and some instance-dependent bounds are known~\citep[e.g., for variants of $Q$-learning][]{khamaru2021instance}. It is an important open question to develop acceleration schemes for such algorithms in a fashion similar to our paper, while extending the results to the more realistic Markovian setting.
	%
	%
	
	\subsection*{Acknowledgments}
	TL and GL were supported in part by Office of Naval Research grant N00014-20-1-2089.
	TL and AP were supported in part by the National Science Foundation grant CCF-2107455, and are thankful to the Simons Institute for the Theory of Computing for their hospitality when part of this work was performed. 
	
	\bibliographystyle{abbrvnat}
	\bibliography{revised_version}

\newcommand{\noopsort}[1]{} \newcommand{\printfirst}[2]{#1}
  \newcommand{\singleletter}[1]{#1} \newcommand{\switchargs}[2]{#2#1}
\begin{thebibliography}{51}
\providecommand{\natexlab}[1]{#1}
\providecommand{\url}[1]{\texttt{#1}}
\expandafter\ifx\csname urlstyle\endcsname\relax
  \providecommand{\doi}[1]{doi: #1}\else
  \providecommand{\doi}{doi: \begingroup \urlstyle{rm}\Url}\fi

\bibitem[Bhandari et~al.(2018)Bhandari, Russo, and Singal]{russo_18}
J.~Bhandari, D.~Russo, and R.~Singal.
\newblock A finite time analysis of temporal difference learning with linear
  function approximation.
\newblock In \emph{Conference on learning theory}, pages 1691--1692. PMLR,
  2018.

\bibitem[Blatt et~al.(2007)Blatt, Hero, and Gauchman]{blatt2007convergent}
D.~Blatt, A.~O. Hero, and H.~Gauchman.
\newblock A convergent incremental gradient method with a constant step size.
\newblock \emph{SIAM Journal on Optimization}, 18\penalty0 (1):\penalty0
  29--51, 2007.

\bibitem[Borkar(2009)]{borkar2009stochastic}
V.~S. Borkar.
\newblock \emph{Stochastic approximation: A dynamical systems viewpoint},
  volume~48.
\newblock Springer, 2009.

\bibitem[Borkar and Meyn(2000)]{borkar2000ode}
V.~S. Borkar and S.~P. Meyn.
\newblock The {ODE} method for convergence of stochastic approximation and
  reinforcement learning.
\newblock \emph{SIAM Journal on Control and Optimization}, 38\penalty0
  (2):\penalty0 447--469, 2000.

\bibitem[Chen et~al.(2021)Chen, Maguluri, Shakkottai, and
  Shanmugam]{chen2021lyapunov}
Z.~Chen, S.~T. Maguluri, S.~Shakkottai, and K.~Shanmugam.
\newblock A {L}yapunov theory for finite-sample guarantees of asynchronous
  {Q}-learning and {TD}-learning variants.
\newblock \emph{arXiv preprint arXiv:2102.01567}, 2021.

\bibitem[Dann et~al.(2014)Dann, Neumann, and Peters]{dann2014policy}
C.~Dann, G.~Neumann, and J.~Peters.
\newblock Policy evaluation with temporal differences: A survey and comparison.
\newblock \emph{Journal of Machine Learning Research}, 15:\penalty0 809--883,
  2014.

\bibitem[Defazio et~al.(2014)Defazio, Bach, and
  Lacoste-Julien]{defazio2014saga}
A.~Defazio, F.~Bach, and S.~Lacoste-Julien.
\newblock {SAGA}: A fast incremental gradient method with support for
  non-strongly convex composite objectives.
\newblock In \emph{Advances in neural information processing systems}, pages
  1646--1654, 2014.

\bibitem[Du et~al.(2017)Du, Chen, Li, Xiao, and Zhou]{du2017stochastic}
S.~S. Du, J.~Chen, L.~Li, L.~Xiao, and D.~Zhou.
\newblock Stochastic variance reduction methods for policy evaluation.
\newblock In \emph{International Conference on Machine Learning}, pages
  1049--1058. PMLR, 2017.

\bibitem[Duan et~al.(2021)Duan, Wang, and Wainwright]{duan2021optimal}
Y.~Duan, M.~Wang, and M.~J. Wainwright.
\newblock Optimal policy evaluation using kernel-based temporal difference
  methods.
\newblock \emph{arXiv preprint arXiv:2109.12002}, 2021.

\bibitem[Durmus et~al.(2021)Durmus, Moulines, Naumov, Samsonov, and
  Wai]{durmus2021stability}
A.~Durmus, E.~Moulines, A.~Naumov, S.~Samsonov, and H.~T. Wai.
\newblock On the stability of random matrix product with {M}arkovian noise:
  {A}pplication to linear stochastic approximation and {TD} learning.
\newblock \emph{arXiv preprint arXiv:2102.00185}, 2021.

\bibitem[H{\'a}jek(1972)]{hajek1972local}
J.~H{\'a}jek.
\newblock Local asymptotic minimax and admissibility in estimation.
\newblock In \emph{Proceedings of the Sixth Berkeley Symposium on Mathematical
  Statistics and Probability, Volume 1: Theory of Statistics}, pages 175--194.
  University of California Press, 1972.

\bibitem[Hsu et~al.(2019)Hsu, Kontorovich, Levin, Peres, Szepesv{\'a}ri, and
  Wolfer]{hsu2019mixing}
D.~Hsu, A.~Kontorovich, D.~A. Levin, Y.~Peres, C.~Szepesv{\'a}ri, and
  G.~Wolfer.
\newblock Mixing time estimation in reversible {M}arkov chains from a single
  sample path.
\newblock \emph{The Annals of Applied Probability}, 29\penalty0 (4):\penalty0
  2439--2480, 2019.

\bibitem[Johnson and Zhang(2013)]{johnson2013accelerating}
R.~Johnson and T.~Zhang.
\newblock Accelerating stochastic gradient descent using predictive variance
  reduction.
\newblock In \emph{Advances in Neural Information Processing Systems},
  volume~26, pages 315--323, 2013.

\bibitem[Kaelbling et~al.(1996)Kaelbling, Littman, and
  Moore]{kaelbling1996reinforcement}
L.~P. Kaelbling, M.~L. Littman, and A.~W. Moore.
\newblock Reinforcement learning: A survey.
\newblock \emph{Journal of artificial intelligence research}, 4:\penalty0
  237--285, 1996.

\bibitem[Khamaru et~al.(2021{\natexlab{a}})Khamaru, Pananjady, Ruan,
  Wainwright, and Jordan]{khamaru2020temporal}
K.~Khamaru, A.~Pananjady, F.~Ruan, M.~J. Wainwright, and M.~I. Jordan.
\newblock Is temporal difference learning optimal? {A}n instance-dependent
  analysis.
\newblock \emph{SIAM Journal on Mathematics of Data Science}, 3\penalty0
  (4):\penalty0 1013--1040, 2021{\natexlab{a}}.

\bibitem[Khamaru et~al.(2021{\natexlab{b}})Khamaru, Xia, Wainwright, and
  Jordan]{khamaru2021instance}
K.~Khamaru, E.~Xia, M.~J. Wainwright, and M.~I. Jordan.
\newblock Instance-optimality in optimal value estimation: Adaptivity via
  variance-reduced {Q}-learning.
\newblock \emph{arXiv preprint arXiv:2106.14352}, 2021{\natexlab{b}}.

\bibitem[Kober et~al.(2013)Kober, Bagnell, and Peters]{kober2013reinforcement}
J.~Kober, J.~A. Bagnell, and J.~Peters.
\newblock Reinforcement learning in robotics: A survey.
\newblock \emph{The International Journal of Robotics Research}, 32\penalty0
  (11):\penalty0 1238--1274, 2013.

\bibitem[Korda and La(2015)]{korda2015td}
N.~Korda and P.~La.
\newblock On {TD} (0) with function approximation: Concentration bounds and a
  centered variant with exponential convergence.
\newblock In \emph{International conference on machine learning}, pages
  626--634. PMLR, 2015.

\bibitem[Kotsalis et~al.(2020{\natexlab{a}})Kotsalis, Lan, and Li]{GGT_20a}
G.~Kotsalis, G.~Lan, and T.~Li.
\newblock Simple and optimal methods for stochastic variational inequalities,
  {I}: operator extrapolation.
\newblock \emph{arXiv preprint arXiv:2011.02987}, 2020{\natexlab{a}}.

\bibitem[Kotsalis et~al.(2020{\natexlab{b}})Kotsalis, Lan, and
  Li]{kotsalis2020simple}
G.~Kotsalis, G.~Lan, and T.~Li.
\newblock Simple and optimal methods for stochastic variational inequalities,
  {II}: {M}arkovian noise and policy evaluation in reinforcement learning.
\newblock \emph{arXiv preprint arXiv:2011.08434}, 2020{\natexlab{b}}.

\bibitem[Lakshminarayanan and Szepesv{\'a}ri(2018)]{lakshminarayanan2018linear}
C.~Lakshminarayanan and C.~Szepesv{\'a}ri.
\newblock Linear stochastic approximation: {H}ow far does constant step-size
  and iterate averaging go?
\newblock In \emph{International Conference on Artificial Intelligence and
  Statistics}, pages 1347--1355. PMLR, 2018.

\bibitem[Le~Cam(1972)]{le1972limits}
L.~Le~Cam.
\newblock Limits of experiments.
\newblock In \emph{Proceedings of the Sixth Berkeley Symposium on Mathematical
  Statistics and Probability, Volume 1: Theory of Statistics}, pages 245--282.
  University of California Press, 1972.

\bibitem[Le~Cam and Yang(2000)]{le2000asymptotics}
L.~Le~Cam and G.~L. Yang.
\newblock \emph{Asymptotics in statistics: Some basic concepts}.
\newblock Springer Science \& Business Media, 2000.

\bibitem[Li et~al.(2020)Li, Wei, Chi, Gu, and Chen]{li2020breaking}
G.~Li, Y.~Wei, Y.~Chi, Y.~Gu, and Y.~Chen.
\newblock Breaking the sample size barrier in model-based reinforcement
  learning with a generative model.
\newblock In \emph{Advances in Neural Information Processing Systems},
  volume~33, pages 12861--12872. Curran Associates, Inc., 2020.

\bibitem[Li et~al.(2021)Li, Cai, Chen, Gu, Wei, and Chi]{li2021q}
G.~Li, C.~Cai, Y.~Chen, Y.~Gu, Y.~Wei, and Y.~Chi.
\newblock Is q-learning minimax optimal? a tight sample complexity analysis.
\newblock \emph{arXiv preprint arXiv:2102.06548}, 2021.

\bibitem[Min et~al.(2021)Min, Wang, Zhou, and Gu]{min2021variance}
Y.~Min, T.~Wang, D.~Zhou, and Q.~Gu.
\newblock Variance-aware off-policy evaluation with linear function
  approximation.
\newblock \emph{Advances in neural information processing systems},
  34:\penalty0 7598--7610, 2021.

\bibitem[Mou et~al.(2020)Mou, Pananjady, and Wainwright]{mou2020optimal}
W.~Mou, A.~Pananjady, and M.~J. Wainwright.
\newblock Optimal oracle inequalities for solving projected fixed-point
  equations.
\newblock \emph{arXiv preprint arXiv:2012.05299}, 2020.

\bibitem[Mou et~al.(2021)Mou, Pananjady, Wainwright, and
  Bartlett]{mou2021optimal}
W.~Mou, A.~Pananjady, M.~J. Wainwright, and P.~L. Bartlett.
\newblock Optimal and instance-dependent guarantees for {M}arkovian linear
  stochastic approximation.
\newblock \emph{preprint}, 2021.

\bibitem[Nemirovsky(1991)]{nemirovsky1991optimality}
A.~S. Nemirovsky.
\newblock On optimality of {K}rylov's information when solving linear operator
  equations.
\newblock \emph{Journal of Complexity}, 7\penalty0 (2):\penalty0 121--130,
  1991.

\bibitem[Nemirovsky(1992)]{nemirovsky1992information}
A.~S. Nemirovsky.
\newblock Information-based complexity of linear operator equations.
\newblock \emph{Journal of Complexity}, 8\penalty0 (2):\penalty0 153--175,
  1992.

\bibitem[Nesterov(2003)]{nesterov2003introductory}
Y.~Nesterov.
\newblock \emph{Introductory lectures on convex optimization: A basic course},
  volume~87.
\newblock Springer Science \& Business Media, 2003.

\bibitem[Ouyang and Xu(2021)]{ouyang2021lower}
Y.~Ouyang and Y.~Xu.
\newblock Lower complexity bounds of first-order methods for convex-concave
  bilinear saddle-point problems.
\newblock \emph{Mathematical Programming}, 185\penalty0 (1):\penalty0 1--35,
  2021.

\bibitem[Pananjady and Wainwright(2021)]{pananjady2020instance}
A.~Pananjady and M.~J. Wainwright.
\newblock Instance-dependent $\ell_\infty$-bounds for policy evaluation in
  tabular reinforcement learning.
\newblock \emph{IEEE Transactions on Information Theory}, 67\penalty0
  (1):\penalty0 566--585, 2021.

\bibitem[Papini et~al.(2018)Papini, Binaghi, Canonaco, Pirotta, and
  Restelli]{papini2018stochastic}
M.~Papini, D.~Binaghi, G.~Canonaco, M.~Pirotta, and M.~Restelli.
\newblock Stochastic variance-reduced policy gradient.
\newblock In \emph{International conference on machine learning}, pages
  4026--4035. PMLR, 2018.

\bibitem[Polyak and Juditsky(1992)]{polyak1992acceleration}
B.~T. Polyak and A.~B. Juditsky.
\newblock Acceleration of stochastic approximation by averaging.
\newblock \emph{SIAM journal on control and optimization}, 30\penalty0
  (4):\penalty0 838--855, 1992.

\bibitem[Puterman(2014)]{puterman2014markov}
M.~L. Puterman.
\newblock \emph{{M}arkov decision processes: discrete stochastic dynamic
  programming}.
\newblock John Wiley \& Sons, 2014.

\bibitem[Schmidt et~al.(2017)Schmidt, Le~Roux, and Bach]{schmidt2017minimizing}
M.~Schmidt, N.~Le~Roux, and F.~Bach.
\newblock Minimizing finite sums with the stochastic average gradient.
\newblock \emph{Mathematical Programming}, 162\penalty0 (1-2):\penalty0
  83--112, 2017.

\bibitem[Sidford et~al.(2018)Sidford, Wang, Wu, Yang, and Ye]{sidford2018near}
A.~Sidford, M.~Wang, X.~Wu, L.~F. Yang, and Y.~Ye.
\newblock Near-optimal time and sample complexities for solving {M}arkov
  decision processes with a generative model.
\newblock In \emph{Proceedings of the 32nd International Conference on Neural
  Information Processing Systems}, pages 5192--5202, 2018.

\bibitem[Srikant and Ying(2019)]{srikant2019finite}
R.~Srikant and L.~Ying.
\newblock Finite-time error bounds for linear stochastic approximation and {TD}
  learning.
\newblock In \emph{Conference on Learning Theory}, pages 2803--2830. PMLR,
  2019.

\bibitem[Sutton(1988)]{sutton1988learning}
R.~S. Sutton.
\newblock Learning to predict by the methods of temporal differences.
\newblock \emph{Machine learning}, 3\penalty0 (1):\penalty0 9--44, 1988.

\bibitem[Sutton et~al.(2009)Sutton, Maei, Precup, Bhatnagar, Silver,
  Szepesv{\'a}ri, and Wiewiora]{sutton2009fast}
R.~S. Sutton, H.~R. Maei, D.~Precup, S.~Bhatnagar, D.~Silver,
  C.~Szepesv{\'a}ri, and E.~Wiewiora.
\newblock Fast gradient-descent methods for temporal-difference learning with
  linear function approximation.
\newblock In \emph{International Conference on Machine Learning}, pages
  993--1000, 2009.

\bibitem[Tadic(2004)]{tadic2004almost}
V.~B. Tadic.
\newblock On the almost sure rate of convergence of linear stochastic
  approximation algorithms.
\newblock \emph{IEEE Transactions on Information Theory}, 50\penalty0
  (2):\penalty0 401--409, 2004.

\bibitem[Tsitsiklis and Van~Roy(1997)]{tsitsiklis_vanroy_97}
J.~N. Tsitsiklis and B.~Van~Roy.
\newblock An analysis of temporal-difference learning with function
  approximation.
\newblock \emph{IEEE transactions on automatic control}, 42\penalty0
  (5):\penalty0 674--690, 1997.

\bibitem[Wai et~al.(2018)Wai, Yang, Wang, and Hong]{wai2018multi}
H.-T. Wai, Z.~Yang, Z.~Wang, and M.~Hong.
\newblock Multi-agent reinforcement learning via double averaging primal-dual
  optimization.
\newblock \emph{Advances in Neural Information Processing Systems}, 31, 2018.

\bibitem[Wai et~al.(2019)Wai, Hong, Yang, Wang, and Tang]{wai2019variance}
H.~T. Wai, M.~Hong, Z.~Yang, Z.~Wang, and K.~Tang.
\newblock Variance reduced policy evaluation with smooth function
  approximation.
\newblock In \emph{Advances in Neural Information Processing Systems},
  volume~32, pages 5784--5795, 2019.

\bibitem[Wainwright(2019)]{wainwright2019variance}
M.~J. Wainwright.
\newblock Variance-reduced ${Q}$-learning is minimax optimal.
\newblock \emph{arXiv preprint arXiv:1906.04697}, 2019.

\bibitem[Wang et~al.(2020)Wang, Foster, and Kakade]{wang2020statistical}
R.~Wang, D.~P. Foster, and S.~M. Kakade.
\newblock What are the statistical limits of offline rl with linear function
  approximation?
\newblock \emph{arXiv preprint arXiv:2010.11895}, 2020.

\bibitem[Wolfer and Kontorovich(2019)]{wolfer2019estimating}
G.~Wolfer and A.~Kontorovich.
\newblock Estimating the mixing time of ergodic {M}arkov chains.
\newblock In \emph{Conference on Learning Theory}, pages 3120--3159. PMLR,
  2019.

\bibitem[Xiao and Zhang(2014)]{xiao2014proximal}
L.~Xiao and T.~Zhang.
\newblock A proximal stochastic gradient method with progressive variance
  reduction.
\newblock \emph{SIAM Journal on Optimization}, 24\penalty0 (4):\penalty0
  2057--2075, 2014.

\bibitem[Xu et~al.(2020)Xu, Wang, Zhou, and Liang]{xu2020reanalysis}
T.~Xu, Z.~Wang, Y.~Zhou, and Y.~Liang.
\newblock Reanalysis of variance reduced temporal difference learning.
\newblock \emph{arXiv preprint arXiv:2001.01898}, 2020.

\bibitem[Zanette(2021)]{zanette2021exponential}
A.~Zanette.
\newblock Exponential lower bounds for batch reinforcement learning: Batch rl
  can be exponentially harder than online rl.
\newblock In \emph{International Conference on Machine Learning}, pages
  12287--12297. PMLR, 2021.

\end{thebibliography}
	
	\appendix
	
	\section{Extension of Theorem~\ref{theorem_lower_bound}}\label{sec_app_B}
	In this section, we extend Theorem~\ref{theorem_lower_bound} to a broader method class motivated by~\cite{nemirovsky1991optimality,nemirovsky1992information}. This method class, which allows a matrix transpose in addition to the family covered by Assumption~\ref{M_assump}, is defined by the following.
	\begin{assumption}\label{M_tilde_assump}
		The iterative method generates a sequence of iterates $v_k$ such that $v_0 = 0$ and
		\begin{align}\label{assump_M_tilde}
			v_k \in \mathrm{span}\{r, (I-\gamma P)v_1, (I-\gamma P)^\top v_1,..., (I-\gamma P)v_{k-1}, (I-\gamma P)^\top v_{k-1}\},\quad k \geq 1.
		\end{align}
	\end{assumption}
	\noindent It should be noted that the method class covered by Assumption~\ref{M_tilde_assump} is not particularly natural for solving the policy evaluation problem under standard observation models. In particular, in typical stochastic settings, the stochastic operator corresponding to the vector $(I-\gamma P)^\top v$ is hard to obtain with transitions and reward samples. However, the oracle complexity lower bound applies in the deterministic setting and still serves as a useful baseline; in particular, it disallows the possibility that the transpose of the matrix $(I - \gamma P)$ can somehow be leveraged to obtain faster rates.
	\begin{corollary}\label{coro_lower_bound}
		Fix a discount factor $\gamma>\tfrac{1}{2} $. 
		There exists a transition kernel $P$ and an expected reward vector $r$ such that for any iterative method satisfying Assumption~\ref{M_tilde_assump}, the following statement holds. If $(D,k)$ satisfies $\tfrac{1-(2\gamma-1)^{2D-4k+2}}{1-(2\gamma-1)^{2D}}\geq \tfrac{1}{2}$, then
		\begin{align*}
			\|v_k-v^*\|_\Pi^2 \geq \tfrac{1}{2} (2\gamma-1)^{2k}\|v_0-v^*\|_\Pi^2.
		\end{align*}
	\end{corollary}
	\begin{proof}
		In this proof, we utilize the same worst-case instance as in Theorem~\ref{theorem_lower_bound}, which is defined in Eqs.~\eqref{P} and \eqref{b}. Then the square matrix $(I-\gamma P)^\top \in \bbr^{D\times D}$ is
		\begin{align}\label{I_P_t}
			(I-\gamma P)^\top= \begin{bmatrix} 
				\tfrac{1}{2} & \tfrac{1}{2}-\gamma &0 &...  &0&0\\ 
				0&\tfrac{1}{2} &\tfrac{1}{2}-\gamma&...  &0&0\\
				0 &0 &\tfrac{1}{2} &...  &0&0\\
				\vdots&\vdots &\vdots &\ddots&\vdots&\vdots\\
				\tfrac{1}{2}-\gamma &0&0 &...&0&\tfrac{1}{2}
			\end{bmatrix},
		\end{align}
		which is an upper triangular matrix except for one elements on the bottom left. Let $\bbr^{(k,j),D}:=\{v\in \R^D \;|\; v_{(i)}=0 \text{ for all } k+1\leq i\leq D-j\}$ denote the subspace of vectors in $\bbr^D$ in which only the first $k$ and the last $j$ components can be nonzero. Provided $D > 2k$, we claim that the sequence $\{v_k\}$ generated by the iterative method satisfies 
		\begin{align} \label{eq:vk-ind}
			v_k \in \bbr^{(k,k-1),D}.
		\end{align} 
		Taking this claim as given for the moment, let us establish the claimed result.
		Recall the solution vector $v^*$ given by Eq.~\eqref{v_star}. We have
		\begin{align*}
			\|v_k-v^*\|_\Pi^2 \geq \tfrac{1}{D} \tsum_{i=k+1}^{D-k+1} (2\gamma-1)^{2i} = \tfrac{(2\gamma-1)^{2k+2}[1-(2\gamma-1)^{2D-4k+2}]}{D[1-(2\gamma-1)^2]}.
		\end{align*}
		If $\tfrac{1-(2\gamma-1)^{2D-4k+2}}{1-(2\gamma-1)^{2D}}\geq \tfrac{1}{2}$, then together with Eq.~\eqref{v_0_dis} we conclude that
		\begin{align*}
			\|v_k-v^*\|_\Pi^2 \geq \tfrac{1}{2} (2\gamma-1)^{2k}\|v_0-v^*\|_\Pi^2,
		\end{align*}
		as desired.
	\end{proof}
	
	\paragraph{Proof of claim~\eqref{eq:vk-ind}:}
	We prove the claim by induction on $k$. \\
	\underline{Base case:}	Given all entries of $r$ except for its first are equal to $0$, we immediately have $v_1\in \bbr^{(1,0),D}$. Thus, the claim is true for $k = 1$. \\
	\underline{Induction step:} Now assume that the claim holds for all $t=1,2,...,k-1$, i.e., that $v_t \in \bbr^{(t,t-1),D}$. From the special structure of $I-\gamma P$ and $(I-\gamma P)^\top$, we obtain that
	\begin{align*}
		(I-\gamma P)v_t \in \bbr^{(t+1, t-1),D}~\text{and}~(I-\gamma P)^\top v_t \in \bbr^{(t, t),D}.
	\end{align*}
	Assumption~\ref{M_tilde_assump} then guarantees that $v_k \in \bbr^{(k,k-1),D}$, thus completing the induction step.

	
	\section{Proofs of supporting lemmas}\label{sec_app}
	
	In this section, we provide deferred proofs of various lemmas used throughout the main text and Section \ref{sec_proof}.
	
	\subsection{Proof of Lemma~\ref{lemma_operator_bias_1}}\label{proof_operator_bias_1}
	Let $\Pi_\tau(s_{t+\tau}):=\diag([P(s_{t+\tau}=1|s_t)~...~P(s_{t+\tau=D}|s_t)])$. We have
	\begin{align*}
		\bbe[\widetilde g(\theta,\xi_{t+\tau}^\tau)|\mathcal{F}_t] &= \bbe[\left( \langle \psi(s_{t+\tau})- \gamma \psi(s_{t+\tau+1}), \theta\rangle - r(s_{t+\tau},s_{t+\tau+1}) \right) \psi(s_{t+\tau})|s_t] \\&= \Psi \Pi_{\tau}(s_{t+\tau})(\Psi^\top \theta- r-\gamma P \Psi^\top\theta).
	\end{align*} 
	Combining the above equality with the definition of $g(\cdot)$ in~\eqref{deterministic_opt} yields
	\begin{align*}
		\|g(\bar \theta)-\bbe[\widetilde g(\bar \theta,\xi_{t+\tau})|\mathcal{F}_t]\|_2
		&=\|\Psi  \big(\Pi-\Pi_t(s_{t+\tau})\big) (\Psi^\top \bar \theta - r -\gamma P \Psi^\top \bar \theta)\|_2\\
		&= \|\Psi  \big(\Pi-\Pi_t(s_{t+\tau})\big) (I-\gamma P)(\Psi^\top \bar \theta -v^*)\|_2\\
		&\leq \skipcon\rho^\tau \|\Psi\|_2\|I-\gamma P\|_2\|\vbar - v^*\|_2 \\
		&\leq  \mixcon \rho^\tau \|\vbar - v^*\|_\Pi,
	\end{align*}
	as claimed. 
	\qed
	
	\subsection{Proof of Lemma~\ref{lemma_operator_bias_2}}\label{proof_operator_bias_2}
	We write
	\begin{align*}
		\|g(\theta)&-\bbe[\widetilde g(\theta,\xi_{t+\tau})|\mathcal{F}_t]-g(\theta')+\bbe[\widetilde g(\theta',\xi_{t+\tau})|\mathcal{F}_t]\|_2\\
		&=\|\Psi \big(\Pi-\Pi_t(s_{t+\tau})\big)(I-\gamma P)\Psi^\top (\theta-\theta')\|_2\\
		&\leq \skipcon\rho^\tau \|\Psi\|_2   \|I-\gamma P\|_2  \|\Psi^\top(\theta-\theta')\|_2 \\
		&\leq  \mixcon \rho^\tau \|v- v'\|_\Pi,
	\end{align*}
	which completes the proof.
	\qed
	
	\subsection{Proof of Lemma~\ref{lemma_hajek}}\label{proof_lemma_hajek}
	We compute the pseudo-inverse of the Fisher information matrix $J^\dagger_{\param}$ and the gradient $\nabla h (\param)$ separately and combine them to obtain the desired result.
	
	\noindent\textbf{Step 1}: 
	Note that by a sufficiency argument, the problem is equivalent to one in which we observe
	\begin{align*}
		\widehat H = \big(\psi(s)\psi(s)^\top - \gamma \psi(s)\psi(s')^\top, R(s,s')\psi(s) \big),\quad\text{where}~s\sim \omega,~s'\sim \mathsf{P}(\cdot|s).
	\end{align*}
	Let us denote 
	\begin{align*}
		H := (\Psi \Omega\Psi^\top - \gamma\Psi \Omega P \Psi^\top, \Omega\Psi r )
	\end{align*}
	to be the respective means of these observations. 
	For each pair $i, j \in [D]$, let $H^{i,j}:= \big(\psi(i)\psi(i)^\top - \gamma \psi(i)\psi(j)^\top, R(i,j)\psi(i) \big)$. The log-likelihood given our observation is then
	\begin{align*}
		\mathcal{L}(\param|\widehat H) = \sum_{i,j\in [D]} \mathbb{I}_{\{\widehat H = H^{i,j}\}}\cdot (\log \omega_i + \log P_{i,j}).
	\end{align*}
	It should be noted that for $\omega$, we take the first $D-1$ components as the parameters and the last component of $\omega$ is determined by the fact that the entries sum to $1$. A similar statement holds for each row of $P$.  Taking derivatives with respect to $\omega$, $P$ and $R$, we obtain 
	\begin{align*}
		\frac{\partial \mathcal{L}}{\partial \omega_k} &= \sum_{j \in [D]} \mathbb{I}_{\{\widehat H = H^{k,j}\}} \cdot \omega^{-1}_k -  \mathbb{I}_{\{\widehat H = H^{D,j}\}} \cdot \omega^{-1}_D,  \;\; \text{ for all } k \in [D-1].\\
		\frac{\partial \mathcal{L}}{\partial P_{i,j}} &= \mathbb{I}_{\{\hat H = H^{i,j}\}} {P_{i,j}^{-1}} - \mathbb{I}_{\{\hat H = H^{i,D}\}} {P_{i,D}^{-1}},  \;\; \text{ for all } i \in [D], ~j\in [D-1].\\
		\frac{\partial \mathcal{L}}{\partial R_{i,j}} &= 0,  \;\; \text{ for all }i,j\in[D].
	\end{align*}
	Consequently, a direct calculation yields
	\begin{align*}
		\bbe\left[\frac{\partial \mathcal{L}}{\partial \omega_k} \times \frac{\partial \mathcal{L}}{\partial \omega_k} \right] & = {\omega_k^{-1}} + {\omega_D^{-1}},\;\; &\text{ for all } k\in[D-1].\\
		\bbe\left[\frac{\partial \mathcal{L}}{\partial \omega_k} \times \frac{\partial \mathcal{L}}{\partial \omega_l} \right] & = {\omega_D^{-1}},\;\; &\text{ for all } k,l\in [D-1],~k\neq l.\\
		\bbe\left[\frac{\partial \mathcal{L}}{\partial \omega_k} \times \frac{\partial \mathcal{L}}{\partial P_{i,j}} \right] & = 0, \;\; &\text{ for all } i \in [D], ~j,k\in [D-1].\\
		\bbe\left[\frac{\partial \mathcal{L}}{\partial P_{i,j}} \times \frac{\partial \mathcal{L}}{\partial P_{i,j}} \right] & = \frac{\omega_i}{P_{i,j}}+\frac{\omega_i}{P_{i,D}}, \;\; &\text{ for all } i \in [D], ~j\in [D-1].\\
		\bbe\left[\frac{\partial \mathcal{L}}{\partial P_{i,j}} \times \frac{\partial \mathcal{L}}{\partial P_{i,j'}} \right] &= \frac{\omega_i}{P_{i,D}}, \;\; &\text{ for all } i \in [D], ~j,j'\in [D-1], ~j\neq j'.
	\end{align*}
	The other terms in $J_{\param}$ matrix are all $0$. Therefore $J_{\param}$ can be written as a block diagonal matrix of the form 
	\begin{align*}
		J_\param= \begin{bmatrix} 
			J_\omega & 0 &0 &...  &0&0\\ 
			0 & J_{P_1} &0&...  &0&0\\
			0 &0 &J_{P_2} &...  &0&0\\
			\vdots&\vdots &\vdots &\ddots&\vdots&\vdots\\
			0 &0&0 &...&J_{P_D} &0\\
			0 &0&0 &...&0 &J_R
		\end{bmatrix},
	\end{align*}
	where $J_R$ is the all-zero matrix. Hence the matrix $J^\dagger_{\param}$ is also block diagonal, of the form
	\begin{align}\label{J_dagger}
		J_\param^\dagger= \begin{bmatrix} 
			J_{\omega}^{\dagger} & 0 &0 &...  &0&0\\ 
			0 & J_{P_1}^{\dagger} &0&...  &0&0\\
			0 &0 &J_{P_2}^{\dagger} &...  &0&0\\
			\vdots&\vdots &\vdots &\ddots&\vdots&\vdots\\
			0 &0&0 &...&J_{P_D}^{\dagger} &0\\
			0 &0&0 &...&0 & 0
		\end{bmatrix}.
	\end{align}
	By using the Sherman--Morrison--Woodbury formula we can straightforwardly compute $J_{\omega}^\dagger$ and $J_{P_i}^{\dagger}$ as
	\begin{align*}
		J_{\omega}^\dagger &= \diag\left( [\omega_1,...,\omega_{D-1}]\right) - [\omega_1,...,\omega_{D-1}]^\top \cdot [\omega_1,...,\omega_{D-1}] \\
		J_{P_i}^{\dagger} &= \omega^{-1}_i \cdot \left( \diag([P_{i,1},...,P_{i,D-1}]) - [P_{i,1},...,P_{i,D-1}]^\top \cdot [P_{i,1},...,P_{i,D-1}] \right), \quad \text{ for all } i\in [D].
	\end{align*}
	
	\noindent\textbf{Step 2}: Next, we turn our attention to evaluating the gradient $\nabla h(\param)$. We do not need to compute $\nabla_R h(\param)$ since $J^{\dagger}_R = 0$. Letting $U:= \Psi \Omega \Psi^\top - \gamma \Psi \Omega P \Psi^\top$ and $U^{i,j} := \psi(i)\psi(i)^\top - \gamma \psi(i)\psi(j)^\top,~\text{for all}~i,j\in[D]$, we have
	\begin{align*}
		\frac{\partial h(\param)}{U_{j,k}} &= - U^{-1}e_je_k^\top U^{-1} \Psi \Omega r = - U^{-1}e_j e_k^\top h(\param),\\
		\frac{\partial U}{\partial \omega_i} & = \sum_{j \in [D]}\left(P_{i,j}U^{i,j}  -P_{D,j} U^{D,j}\right).
	\end{align*}
	Using the chain rule yields that
	\begin{align}\label{partial_1}
		\frac{\partial h(\param)}{\partial \omega_i} = U^{-1}\psi(i)r(i) - U^{-1}\psi(D)r(D) - U^{-1}\cdot\sum_{j\in [D]}\left(P_{i,j}U^{i,j}  - P_{D,j} U^{D,j} \right)h(\param)
	\end{align}
	for all $i\in [D-1]$.
	As for the gradient with respect to $P$, we have
	\begin{align*}
		\frac{\partial U}{\partial P_{i,j}} = \omega_i\cdot  U^{i,j} - \omega_i \cdot U^{i,D}.
	\end{align*}
	Using the chain rule once again, we obtain
	\begin{align}\label{partial_2}
		\frac{\partial h(\param)}{\partial P_{i,j}} =
		\omega_i\cdot U^{-1} (U^{i,D}-U^{i,j}) h(\param) + \omega_i \cdot U^{-1}\psi(i) \big(R(i,j)-R(i,D)\big).
	\end{align}
	for all $i\in [D]$ and $j \in [D-1]$.
	
	\noindent\textbf{Step 3}: Finally, we evaluate $\nabla h (\param)^\top  J_{\param}^\dagger~\nabla h(\param)$. From the block structure, we can write
	\begin{align}\label{combine_0}
		\nabla h (\param)^\top  J_{\param}^\dagger~\nabla h (\param) = \nabla_\omega h (\param)^\top  J_{\omega}^\dagger~\nabla_\omega h (\param) + \sum_{i\in[D]} \nabla_{P_i} h (\param)^\top  J_{P_i}^\dagger~\nabla_{P_i} h (\param).
	\end{align}
	Let $y(s,s'):= \psi(s)\psi(s)^\top h (\param) - \gamma \psi(s)\psi(s')^\top h (\param) - R(s,s') \psi(s)$.
	Combining the relations \eqref{J_dagger}, \eqref{partial_1} and \eqref{partial_2}, we have
	\begin{subequations}
		\begin{align}\label{combine_1}
			\nabla_\omega h (\param)^\top  J_{\omega}^\dagger&~\nabla_\omega h (\param)\nn\\ = &U^{-1} \bigg\{\sum_{i\in [D]} \omega_i \cdot \big(\bbe[y(i,s')| s'\sim \mathsf{P}(\cdot|i)]\big) \cdot \big( \bbe[y(i,s')| s'\sim \mathsf{P}(\cdot|i)]\big)^\top \nn\\
			&\quad- \bigg( \sum_{i \in [D]} \omega_i \cdot \bbe[y(i,s')| s'\sim \mathsf{P}(\cdot|i)]\bigg)\bigg( \sum_{i \in [D]} \omega_i \cdot \bbe[y(i,s')|s'\sim \mathsf{P}(\cdot|i)]\bigg)^\top\bigg\}\cdot U^{-\top},
		\end{align}
		and 
		\begin{align}\label{combine_2}
			\nabla_{P_i} h (\param)^\top  J_{P_i}^\dagger~\nabla_{P_i} h (\param)= \omega_i \cdot U^{-1} \bigg\{&\sum_{j\in [D]} P_{i,j}\cdot y(i,j) y(i,j)^\top\nn\\
			&-\big(\bbe[y(i,s')| s'\sim \mathsf{P}(\cdot|i)]\big)\big( \bbe[y(i,s')| s'\sim \mathsf{P}(\cdot|i)]\big)^\top\bigg\} \cdot U^{-\top}.
		\end{align}
	\end{subequations}
	Substituting Eqs.~\eqref{combine_1} and~\eqref{combine_2} into Eq.~\eqref{combine_0}, we obtain 
	\begin{align*}
		&\nabla h (\param)^\top  J_{\param}^\dagger~\nabla h (\param) \\
		&\quad = U^{-1}\bigg\{ \sum_{i\in [D]}\sum_{j\in[D]} \omega_i P_{i,j} \cdot y(i,j) y(i,j)^\top \\
		&\qquad \qquad \quad - \bigg( \sum_{i \in [D]} \omega_i \cdot \bbe[y(i,s')| s'\sim \mathsf{P}(\cdot|i)]\bigg)\bigg( \sum_{i \in [D]} \omega_i \cdot\bbe[y(i,s')|s'\sim \mathsf{P}(\cdot|i)]\bigg)^\top  \bigg\} \cdot U^{-\top}\\
		&\quad = U^{-1} \cdot  \cov_{s\sim \omega, s' \sim \mathsf{P}(\cdot|s)}\big\{ y(s,s') \big\} \cdot U^{-\top},
	\end{align*}
	which completes the proof.
	\qed

	\subsection{Proof of Lemma~\ref{bound_tilde_v}}\label{proof_bound_tilde_v}
	By the definition of $g(\cdot)$ , we have
	\begin{align*}
		\Psi \Pi \Psi^\top(\bar \theta - \underline \theta) &= \Psi\Pi r + \gamma \Psi \Pi P \Psi^\top \bar \theta - \big(\Psi \Pi r+ \gamma \Psi \Pi P \Psi^\top \underline \theta+g(\widetilde \theta)-\widehat g(\widetilde \theta)\big)\\
		&= \gamma \Psi \Pi P \Psi^\top(\bar \theta- \underline \theta) + \big(\widehat g(\widetilde \theta) - g(\widetilde\theta)\big).
	\end{align*}
	Invoking the fact that $\Phi = B^{-\frac{1}{2}}\Psi$ and $\Phi \Pi \Phi^\top = I_d$, we obtain
	\begin{align}\label{bound_0}
		B^{\frac{1}{2}}(\bar\theta-\underline \theta) = (\Phi\Pi\Phi^\top-\gamma \Phi\Pi P \Phi^\top)^{-1}B^{-\frac{1}{2}}\big(\widehat g(\widetilde\theta) - g(\widetilde\theta)\big)=(I_d-M)^{-1}B^{-\frac{1}{2}}\big(\widehat g(\widetilde\theta) - g(\widetilde\theta)\big).
	\end{align}
	Therefore, we have that
	\begin{align}\label{bound_0_1}
		\bbe\|\underline v - \vbar\|_\Pi^2 &=\bbe\|\Psi^\top B^{-\frac{1}{2}}(I_d-M)^{-1}B^{-\frac{1}{2}}\big(\widehat g(\widetilde\theta) - g(\widetilde\theta)\big)\|_\Pi^2\nn\\
		&=\bbe\|(I_d-M)^{-1}B^{-\frac{1}{2}}\big(\widehat g(\widetilde\theta) - g(\widetilde\theta)\big)\|_2^2.
	\end{align}
	Applying Young's inequality then yields
	\begin{align*}
		\bbe\|\underline v - \vbar\|_\Pi^2
		& \leq  2\bbe\|(I_d-M)^{-1} B^{-\frac{1}{2}}\big( \widehat g(\bar \theta) - g(\bar \theta) \big)\|_2^2\\
		&\qquad\quad+ 2 \bbe\|(I_d-M)^{-1} B^{-\frac{1}{2}}\big( \widehat g(\widetilde \theta) - \widehat g(\bar \theta)  - g(\widetilde \theta)+ g(\bar \theta) \big)\|_2^2\\
		& \overset{(i)}\leq \tfrac{2}{N_k} \cdot \bbe\|(I_d-M)^{-1}B^{-\frac{1}{2}} \big(  \widetilde g(\bar \theta,\xi_1) - g(\bar \theta) \big)\|_2^2 \\
		&\qquad \quad + \tfrac{2}{N_k}\cdot \bbe\|(I_d-M)^{-1} B^{-\frac{1}{2}}\big( \widetilde g(\widetilde \theta,\xi_1) - \widetilde g(\bar \theta,\xi_1)  - g(\widetilde \theta)+ g(\bar \theta) \big)\|_2^2\\
		&\overset{(ii)}\leq \tfrac{2}{N_k}\cdot \trace\left((I_d-M)^{-1}\iidS(I_d-M)^{-\top}\right)\\
		&\qquad\quad+ \tfrac{2 }{N_k(1-\gamma)^2\mu}\cdot \bbe\|\widetilde g(\widetilde \theta,\xi_1) - \widetilde g(\bar \theta,\xi_1)  - g(\widetilde \theta)+ g(\bar \theta) \|_\Pi^2\\
		&\overset{(ii)}\leq \tfrac{2}{N_k}\cdot \trace\left((I_d-M)^{-1}\iidS(I_d-M)^{-\top}\right) + \tfrac{2\varsigma^2 }{N_k(1-\gamma)^2\mu}\cdot \bbe\|\widetilde v - \vbar\|_\Pi^2,
	\end{align*}
	where step (i) follows from the i.i.d. property of the samples, step (ii) follows from Lemma \ref{lemma_inverse}, and step (iii) follows from Assumption~\ref{assump_variance}.\qed 
	
	\subsection{Proof of Lemma~\ref{stoch_opt_variance_1}}\label{proof_stoch_opt_variance_1}
	
	We index samples $\xi_t = (s_t,s_{t+1},R(s_t,s_{t+1}))$ successively as in Section~\ref{sec_obs_model}. We have 
	\begin{align*}
		&\bbe[\|\widetilde g(\theta,\xi_{t+\tau}) - \widetilde g(\theta',\xi_{t+\tau})-g(\theta)+ g(\theta')\|^2_2|\mathcal{F}_t]\\
		& =\sum_{i\in [D]} \mathbb{P}(s_{t+\tau}=i|s_t) \cdot \bbe[\|\langle \psi(i)- \gamma \psi(s_{t+\tau+1}), \theta-\theta' \rangle  \psi(i)-g(\theta)+ g(\theta')\|^2_2|s_{t+\tau}=i]\\
		&\leq\sum_{i\in [D]} \pi_i\cdot \bbe\left[\|\langle \psi(i)- \gamma \psi(s_{t+\tau+1}), \theta-\theta' \rangle  \psi(i)-g(\theta)+ g(\theta')\|^2_2|s_{t+\tau}=i\right]\\
		&\quad+\sum_{i\in [D]} \|\mathbb{P}(s_{t+\tau}=\cdot|s_t) - \pi\|_\infty \bbe\left[\|\langle \psi(i)- \gamma \psi(s_{t+\tau+1}), \theta-\theta' \rangle  \psi(i)-g(\theta)+ g(\theta')\|^2_2|s_{t+\tau}=i\right]	\\
		& \overset{(i)}\leq\sum_{i\in [D]} 2\pi_i \bbe\left[\|\langle \psi(i)- \gamma \psi(s_{t+\tau+1}), \theta-\theta' \rangle  \psi(i)-g(\theta)+ g(\theta')\|^2_2|s_{t+\tau}=i\right]\\
		& \overset{(ii)}\leq2\varsigma^2\|v-v'\|_\Pi^2	,
	\end{align*}
	where step (i) follows from Assumption~\ref{assump_rho_0} and condition~\eqref{cond_tau_1}, and step (ii) from Assumption~\ref{assump_variance}. 
	\qed
	
	\black{\subsection{Proof of Lemma~\ref{stoch_opt_variance_2}}\label{proof_stoch_opt_variance_2}
		By the definition of $\widehat g$, we write
		\begin{align}\label{error_decomp}
			\bbe\|\widehat g(\theta) &- \widehat g(\theta')- g(\theta) + g(\theta')\|^2_2\nn\\ &= \tfrac{1}{(N_k-n_0)^2}\bbe\|\tsum_{i=n_0+1}^{N_k}\big(\Delta^k_i(\theta) - \Delta^k_i(\theta')\big)\|_2^2\nn\\
			& = \tfrac{1}{(N_k-n_0)^2}\sum_{i=n_0+1}^{N_k}\bbe\|\Delta^k_i(\theta) - \Delta^k_i(\theta')\|_2^2 \nn\\
			&\qquad+ \tfrac{1}{(N_k-n_0)^2} \sum_{n_0+1\leq i <j\leq N_k}2 \bbe \langle \Delta^k_i(\theta) - \Delta^k_i(\theta'), \Delta^k_j(\theta) - \Delta^k_j(\theta') \rangle\nn\\
			&= \underbrace{\tfrac{1}{(N_k-n_0)^2}\sum_{i=n_0+1}^{N_k}\bbe\|\Delta^k_i(\theta) - \Delta^k_i(\theta')\|_2^2}_{R_1}\nn\\
			&\qquad+ \underbrace{\tfrac{1}{(N_k-n_0)^2} \cdot\sum_{\substack{n_0+1\leq i <j\leq N_k\\j-i\leq \tau}}2 \bbe \langle \Delta^k_i(\theta) - \Delta^k_i(\theta'), \Delta^k_j(\theta) - \Delta^k_j(\theta') \rangle}_{R_2}\nn\\
			&\qquad + \underbrace{\tfrac{1}{(N_k-n_0)^2} \cdot\sum_{\substack{n_0+1\leq i <j\leq N_k\\j-i> \tau}}2 \bbe \langle \Delta^k_i(\theta) - \Delta^k_i(\theta'), \Delta^k_j(\theta) - \Delta^k_j(\theta') \rangle}_{R_3}.
		\end{align}
		where $\Delta^k_i(\theta) = \widetilde g(\theta, \xi_i^k) - g(\theta)$. We bound $R_1$, $R_2$ and $R_3$ separately. First, by Lemma SM1.14 and the fact that $n_0$ satisfies condition \eqref{cond_tau_1}, we have that
		\begin{align*}
			R_1 \leq \tfrac{2\varsigma^2}{N_k-n_0}\|v-v'\|_\Pi^2.
		\end{align*}
		For $R_2$, by utilizing Holder's inequality, we obtain
		\begin{align*}
			R_2 &\leq \tfrac{1}{(N_k-n_0)^2} \cdot\sum_{\substack{n_0+1\leq i <j\leq N_k\\j-i\leq \tau}}2 \sqrt{\bbe  \|\Delta^k_i(\theta) - \Delta^k_i(\theta')\|_2^2} \sqrt{\|\Delta^k_j(\theta) - \Delta^k_j(\theta') \|_2^2}\\
			&\overset{(i)}\leq \tfrac{4\varsigma^2\tau}{N_k-n_0}\|v-v'\|_\Pi^2,
		\end{align*}
		where step (i) follows from Lemma SM1.14. To bound $R_3$, we first bound the following term
		\begin{align*}
			\bbe \langle \Delta^k_i(\theta) - \Delta^k_i(\theta'), \Delta^k_j(\theta) - \Delta^k_j(\theta') \rangle& = \bbe \langle \Delta^k_i(\theta) - \Delta^k_i(\theta'), \bbe[\Delta^k_j(\theta) - \Delta^k_j(\theta')|\calF] \rangle\\
			&\overset{(i)}\leq \mixcon \rho^{j-i}\bbe[\|\Delta^k_i(\theta) - \Delta^k_i(\theta')\|_2\|v-v'\|_\Pi]\\
			&\overset{(ii)} \leq \mixcon \rho^{j-i}\bbe[\tfrac{\varsigma}{2}\|v-v'\|_\Pi^2+ \tfrac{1}{2\varsigma}\|\Delta^k_i(\theta) - \Delta^k_i(\theta')\|_2^2]\\
			&\overset{(iii)}\leq \tfrac{3\varsigma}{2}\mixcon \rho^{j-i}\|v-v'\|_\Pi^2.
		\end{align*}
		where step (i) follows from Lemma~\ref{lemma_operator_bias_2}, step (ii) follows from Young's inequality, and step (iii) follows from Lemma SM1.14. Then we can bound $R_3$ as
		\begin{align*}
			R_3 \leq \tfrac{N_k-n_0-\tau}{(N_k-n_0)^2} \cdot \tfrac{3\varsigma\mixcon \rho^\tau}{1-\rho}\|v-v'\|_\Pi^2\overset{(i)}\leq \tfrac{2\varsigma^2}{N_k-n_0}\|v-v'\|_\Pi^2,
		\end{align*}
		where step (i) follows from that $\tau$ satisfies \eqref{cond_tau_0}. Substituting the upper bounds on $R_1, R_2$ and $R_3$ into Ineq.~\eqref{error_decomp} yields the desired result. 
	}
	\subsection{Proof of Lemma~\ref{stoch_opt_variance_3}}\label{proof_stoch_opt_variance_3}
	\black{We have that
		\begin{align*}
			&\|\bbe[\widetilde g_t(\theta) |\mathcal{F}_{t-1} ]- \bbe[\widetilde g_t(\theta') |\mathcal{F}_{t-1}]- g(\theta) + g(\theta')\|_2 \\
			&\qquad\qquad\overset{(i)}\leq\tfrac{1}{m-m_0}\tsum_{j=m_0+1}^m \|\bbe[\widetilde g(\theta,\widehat \xi_j^t) |\mathcal{F}_{t-1}]- \bbe[\widetilde g(\theta',\widehat\xi_j^t)) |\mathcal{F}_{t-1}]- g(\theta) + g(\theta')\|_2\\
			&\qquad \qquad \overset{(ii)}\leq \tfrac{1}{m-m_0}\tsum_{j=m_0+1}^m\mixcon \rho^{j}\|v-v'\|_\Pi\leq \tfrac{\mixcon \rho^{m_0}}{(1-\rho)(m-m_0)}\|v-v'\|_\Pi,
		\end{align*}
		where step (i) follows from triangle inequality and step (ii) follows from Lemma~\ref{lemma_operator_bias_2}. \qed}
	
	\subsection{Proof of Lemma~\ref{bound_tilde_v_Mark}}\label{proof_bound_tilde_v_Mark}
	\textcolor{black}{
		First, by the derivation of our original paper, we have that 
		\begin{align}\label{overall}
			\bbe\|\underline v - \vbar\|_\Pi^2 
			&=\bbe\|(I_d-M)^{-1}B^{-\frac{1}{2}}\big(\widehat g(\widetilde\theta) - g(\widetilde\theta)\big)\|_2^2\nn\\
			& \leq  2 \underbrace{\bbe\|(I_d-M)^{-1} B^{-\frac{1}{2}}\big( \widehat g(\bar \theta) - g(\bar \theta) \big)\|_2^2}_{R_1} \nn\\
			&\quad+ 2 \underbrace{\bbe\|(I_d-M)^{-1} B^{-\frac{1}{2}}\big( \widehat g(\widetilde \theta) - \widehat g(\bar \theta)  - g(\widetilde \theta)+ g(\bar \theta) \big)\|_2^2}_{R_2}.
		\end{align}
		We upper bound the two terms on the RHS of the above inequality separately. We first define a sequence of samples obtained from a stationary Markov trajectory, namely $\widetilde \xi_t = (\widetilde s_t, \widetilde s_{t+1}, R(\widetilde s_t,\widetilde s_{t+1}))$, where $\widetilde s_t \sim \pi$. We write
		\begin{align}\label{bound_R1_0}
			R_1 &= \tsum_{j\in [D]} \mathbb{P}(s_{n_0+1}=j )\cdot \bbe_{s_{n_0+1}=j }\|(I_d-M)^{-1} B^{-\frac{1}{2}}\big(\widehat g(\bar\theta) - g(\bar\theta)\big)\|_2^2\nn\\
			& \leq  \tsum_{j\in [D]} \pi_j\cdot \bbe_{s_{n_0+1}=j }\|(I_d-M)^{-1} B^{-\frac{1}{2}}\big(\widehat g(\bar\theta) - g(\bar\theta)\big)\|_2^2\nn\\
			&\qquad + \|\mathbb{P}(s_{n_0+1}=\cdot ) - \pi\|_\infty \cdot \bbe_{s_{n_0+1}=j }\|(I_d-M)^{-1} B^{-\frac{1}{2}}\big(\widehat g(\bar\theta) - g(\bar\theta)\big)\|_2^2\nn\\
			&\overset{(i)} \leq 2 \bbe\|(I_d-M)^{-1}B^{-\frac{1}{2}} \tfrac{1}{N_k-n_0}\tsum_{i=n_0+1}^{N_k}\big( \widetilde g(\bar \theta, \widetilde \xi_i) - g(\bar \theta)\big)\|_2^2\nn\\
			& = \tfrac{2}{(N_k-n_0)^2}\left[(N_k-n_0)\mu_0 + \tsum_{i=1}^{N_k-n_0-1}2(N_k-n_0-i) \mu_i \right],
		\end{align}
		where step (i) follows from condition~\eqref{cond_tau_1} and the fact that $\widetilde \xi_i$ is a stationary chain, and 
		$$\mu_i:= \bbe\left\langle (I_d-M)^{-1}B^{-\frac{1}{2}}\big(\widetilde g(\bar \theta, \widetilde \xi_0)-g(\bar \theta)\big), (I_d-M)^{-1}B^{-\frac{1}{2}}\big(\widetilde g(\bar \theta, \widetilde \xi_i)-g(\bar \theta)\big) \right\rangle.$$ 
		Recall the definition of the Markovian noise stochastic error term
		$$
		\MarkovS:= \sum_{t=-\infty}^{\infty} B^{-\frac{1}{2}} \bbe\left[\big(\widetilde g(\bar \theta, \widetilde \xi_t)-g(\bar \theta)\big)\big(\widetilde g(\bar \theta, \widetilde \xi_0)-g(\bar \theta)\big)^\top\right] B^{-\frac{1}{2}}.
		$$
		Then we have 
		$
		\trace\big\{(I_d-M)^{-1}\MarkovS(I_d-M)^{-\top}\big\} = \mu_0 + 2 \tsum_{i=1}^\infty \mu_i.
		$
		Now let us discuss the two possible cases: \\
		\textbf{Case 1 ($v^* = \bar v$)}: 
		Let $\widetilde{\mathcal{F}_i}$ denote the $\sigma$-field generated by samples $\widetilde \xi_0, ..., \widetilde \xi_i$ and let $\widetilde \Pi_j^i:= \diag\{[\mathbb{P}(\widetilde s_j=1|\widetilde s_i),...,\mathbb{P}(\widetilde s_j=D|\widetilde s_i)]\}$ for $j\geq i$. Then we have
		\begin{align*}
			\mu_i &= \bbe\left\langle (I_d-M)^{-1}B^{-\frac{1}{2}}\big(\widetilde g(\bar \theta, \widetilde \xi_0)-g(\bar \theta)\big), (I_d-M)^{-1}B^{-\frac{1}{2}}\big(\bbe[\widetilde g(\bar \theta, \widetilde \xi_i)|\widetilde{\mathcal{F}}_0]-g(\bar \theta)\big) \right\rangle\nn\\
			& = \bbe\left\langle (I_d-M)^{-1}B^{-\frac{1}{2}}\big(\widetilde g(\bar \theta, \widetilde \xi_0)-g(\bar \theta)\big), (I_d-M)^{-1}B^{-\frac{1}{2}}\Psi (\widetilde \Pi_j^i - \Pi) (\Psi^\top \bar \theta - \gamma P\Psi^\top \bar \theta - r) \right\rangle \nn\\
			& \overset{(i)}= 0,
		\end{align*}
		where step (i) follows from that $v^* = \bar v = \Psi^\top \bar \theta$ and the Bellman equation \eqref{bellman}. As a result, we can conclude that 
		\begin{align}\label{bound_R1_0}
			R_1 \leq \tfrac{2\mu_0}{N_k-n_0} = \tfrac{2\cdot\trace\big\{(I_d-M)^{-1}\MarkovS(I_d-M)^{-\top}\big\}}{N_k-n_0} = \tfrac{2\cdot\trace\big\{(I_d-M)^{-1}\iidS(I_d-M)^{-\top}\big\}}{N_k-n_0}.
	\end{align}}
	\black{
		\textbf{Case 2 ($\bar v \neq v^*$)}: By Holder's inequality, we have that $|\mu_i| \leq |\mu_0|= \trace\big\{(I_d-M)^{-1}\iidS(I_d-M)^{-\top} \big\}$. On the other hand, for $i\geq 1$, we can obtain another bound for $\mu_i$ as
		\begin{align}\label{bound_mu_1}
			|\mu_i| &= |\bbe\langle (I_d-M)^{-1}B^{-\frac{1}{2}}\big(\widetilde g(\bar \theta, \widetilde \xi_0)-g(\bar \theta)\big), (I_d-M)^{-1}B^{-\frac{1}{2}}\bbe[g(\bar \theta, \widetilde \xi_i)-g(\bar \theta)|\calF_0] \rangle|\nn\\
			&\leq \mixcon \rho^{i}\tfrac{1}{(1-\gamma)\sqrt{\mu}}\bbe[\|(I_d-M)^{-1}B^{-\frac{1}{2}}\big(\widetilde g(\bar \theta, \widetilde \xi_0)-g(\bar \theta)\big)\|_2\|\bar v - v^*\|_\Pi]\nn\\
			&\leq \mixcon \rho^{i} W,
		\end{align}
		where $W := \tfrac{1}{2}\trace\big\{(I_d-M)^{-1}\iidS(I_d-M)^{-\top} \big\} + \tfrac{1}{2 (1-\gamma)^2\mu}\|\bar v - v^*\|_\Pi^2$.
		From Ineq.~\eqref{bound_R1_0}, we have
		\begin{align}\label{bound_R1_2}
			R_1 
			&\leq \tfrac{2\big[(N_k-n_0)\trace\big\{(I_d-M)^{-1}\MarkovS(I_d-M)^{-\top}\big\} - \tsum_{i=1}^{N_k-n_0-1}i\mu_i - 2(N_k-n_0)\tsum_{i=N_k-n_0}^\infty\mu_i  \big]}{(N_k-n_0)^2} \nn\\
			&\leq \tfrac{2\big[(N_k-n_0)\trace\big\{(I_d-M)^{-1}\MarkovS(I_d-M)^{-\top}\big\} + \overbrace{|\tsum_{i=1}^{N_k-n_0-1}i\mu_i|}^{I_1} + 2(N_k-n_0)\overbrace{|\tsum_{i=N_k-n_0}^\infty\mu_i|}^{I_2}  \big]}{(N_k-n_0)^2} .
		\end{align}
		It remains to bound $I_1$ and $I_2$. We write
		\begin{align*}
			I_1 & \leq \tsum_{i=1}^\tau \tau |\mu_i| + \tsum_{i=\tau+1}^\infty i |\mu_i|\\
			&\leq \tau^2 \trace\big\{(I_d-M)^{-1}\iidS(I_d-M)^{-\top} \big\} + \mixcon W\tsum_{i=\tau+1}^\infty i\cdot \rho^i\\
			&\leq \tau^2 \trace\big\{(I_d-M)^{-1}\iidS(I_d-M)^{-\top} \big\} + \tfrac{\mixcon  [(\tau+1)\rho^{\tau+1} + \rho^{\tau+2}/(1-\rho)]}{1-\rho}\cdot W.
		\end{align*}
		For $I_2$, we have 
		\begin{align*}
			I_2 \leq \tsum_{i=N_k-n_0}^\infty|\mu_i| \leq \tfrac{\mixcon\rho^{N_k-n_0}}{1-\rho}W.
		\end{align*}
		Substituting the upper bounds on $I_1$ and $I_2$ and invoking that $\rho^{\tau} \leq  \frac{2(1-\rho)^2}{5 \mixcon}$ and $\rho^{N_k-n_0} \leq \tfrac{\tau(1-\rho)}{5\mixcon (N_k-n_0)}$, we obtain
		\begin{align*}
			R_1 &\leq \tfrac{2\cdot \trace\big\{(I_d-M)^{-1}\MarkovS(I_d-M)^{-\top}\big\}}{N_k-n_0}  + \tfrac{2\tau^2\trace\big\{(I_d-M)^{-1}\iidS(I_d-M)^{-\top} \big\}}{(N_k-n_0)^2} + \tfrac{4(\tau+1)}{5(N_k-n_0)^2} W.
	\end{align*}}
	
	\black{
		On the other hand, by Lemma~\ref{bound_tilde_v_Mark}, we can bound $R_2$ as 
		\begin{align*}
			R_2 \leq \tfrac{4(\tau+1)\varsigma^2}{(1-\gamma)^2\mu(N_k-n_0)}\|\widetilde v -\bar v\|_\Pi^2.
		\end{align*}
		Substituting the bounds on $R_1$ and $R_2$ into Ineq.~\eqref{overall} yields the desired results for both cases.}

\end{document}